%% file: Arxiv.tex
\documentclass[11pt]{article}
\usepackage[top=2.54cm,left=2.54cm,right=2.54cm,bottom=2.54cm]{geometry}

\input{math_commands.tex}

\usepackage[authoryear,round]{natbib}
\usepackage[hidelinks]{hyperref}     
\usepackage{url}      
\usepackage{amsmath,amssymb,amsthm,mathtools}
\usepackage[utf8]{inputenc}
\usepackage[T1]{fontenc}   
\usepackage{booktabs}    
\usepackage{nicefrac}   
\usepackage{microtype}   
\usepackage{xcolor}   
\usepackage{float}
\usepackage{graphicx}
\usepackage{subcaption}
\usepackage{bbm}
\newtheorem{theorem}{Theorem}[section]

\newtheorem{lemma}[theorem]{Lemma}

\newtheorem{definition}[theorem]{Definition}

\newcommand{\non}{{\textnormal{sub}}}
\newcommand{\opt}{{\textnormal{opt}}}
\def\dmin{\Delta_{\textnormal{min}}}
\usepackage[capitalize,noabbrev]{cleveref}
\newcommand{\sah}{\mathcal{S} \times \mathcal{A} \times [H]}
\usepackage{algorithm}
\usepackage{algorithmic}

\title{Variance-Aware Fine-Grained Gap-Dependent Bounds for Online Reinforcement Learning}

\author{
Haochen Zhang$^{1}$, Lingzhou Xue$^{1}$, and Zhong Zheng$^{2}$\\[0.5em]
$^{1}$Department of Statistics, The Pennsylvania State University\\
$^{2}$Department of Applied and Computational Mathematics and Statistics,\\ University of Notre Dame
}
\date{}

\begin{document}

\maketitle
\begin{abstract}
We study model-free online reinforcement learning (RL) for episodic tabular Markov decision processes, focusing on both gap-dependent regret and policy switching cost. While fine-grained gap-dependent analysis has been established for model-free RL algorithms using Hoeffding-type exploration bonuses, such results for model-free algorithms with variance-based exploration bonuses remain unknown, despite their superior worst-case and coarse-grained gap-dependent guarantees. In this paper, we resolve this open problem by establishing the first fine-grained gap-dependent regret upper bound for UCB-Bernstein+, a refined UCB-Bernstein algorithm, in variance-aware model-free online RL. Moreover, by integrating a stage-wise policy update design into our fine-grained framework and using refined variance-based bonuses, we achieve the best-known gap-dependent local switching cost to date. In addition, our analysis yields improved worst-case guarantees for both regret and local switching cost over the original UCB-Bernstein algorithm. Numerical experiments further demonstrate that UCB-Bernstein+ achieves favorable empirical performance in both regret and local switching cost.
\end{abstract}

\section{Introduction}
Reinforcement learning (RL)~\citep{1998Reinforcement} provides a framework for sequential decision-making, where an agent learns to maximize cumulative rewards through interaction with the environment. RL algorithms are commonly divided into model-based and model-free methods. Model-free methods directly learn value functions and policies, and are widely used due to their simple implementation and low memory requirements~\citep{jin2018q}. In contrast, model-based methods learn an empirical model of the transition dynamics and rewards, often incurring higher memory costs.

In this paper, we focus on model-free RL for episodic tabular Markov decision processes (MDPs) with inhomogeneous transition kernels. Specifically, we consider an episodic tabular MDP with $S$ states, $A$ actions, and $H$ steps per episode. For this setting, the minimax regret lower bound over $K$ episodes is $\Omega(\sqrt{H^2SAT})$, where $T = KH$ is the total number of steps \citep{jin2018q}. 

Many model-free RL algorithms achieve worst-case $\sqrt{T}$-type regret bounds
\citep{jin2018q, zhang2020almost, li2023breaking, zhang2025regret},
among which \citet{zhang2020almost,li2023breaking,zhang2025regret} attain the nearly minimax-optimal rate $\tilde{O}(\sqrt{H^2SAT})$. However, in practice, RL algorithms often perform better than their worst-case guarantees when a positive suboptimality gap exists, meaning that the optimal actions at each state-step pair are separated from suboptimal ones by a nontrivial margin. This observation has motivated the study of gap-dependent analysis.

In the model-free setting, \citet{yang2021q} established the first \emph{coarse-grained} gap-dependent regret for UCB-Hoeffding \citep{jin2018q}, which uses Hoeffding-type exploration bonuses, showing a regret of $\tilde{O}\big(H^6 SA/\dmin\big),$ where $\dmin$ denotes the minimum positive suboptimality gap across all state-action-step triples $(s,a,h)$. Here, ``coarse-grained'' means that the bound depends only on the global minimum gap $\dmin$ and does not account for individual suboptimality gaps $\Delta_h(s,a)$.

Beyond Hoeffding-type bonuses, variance-based exploration bonuses have been shown to improve regret performance in both worst-case studies and coarse-grained gap-dependent studies. In particular, several algorithms that leverage variance-based bonuses, including UCB-Bernstein \citep{jin2018q}, UCB-Advantage \citep{zhang2020almost}, Q-EarlySettled-Advantage \citep{li2023breaking}, and Q-EarlySettled-LowCost \citep{zhang2025regret}, achieve better worst-case bounds or coarse-grained gap-dependent bounds \citep{zheng2024gap,zhang2025regret} than UCB-Hoeffding \citep{jin2018q, yang2021q}, especially with the improved dependence on the horizon length $H$.

Meanwhile, \citet{xu2021fine} propose the AMB algorithm, which still employs Hoeffding-type exploration bonuses but provides the first \emph{fine-grained} gap-dependent regret bound. Here, ``fine-grained'' indicates that the bound incorporates the individual suboptimality gaps $\Delta_h(s,a)$ of each state-action-step triple, rather than relying solely on the global minimum gap $\dmin$. Subsequently, \citet{zhang2025q} reanalyze the UCB-Hoeffding algorithm and establish the first fine-grained gap-dependent regret bound for a UCB-based method. The fine-grained gap-dependent bounds for both algorithms take the form:
\[
\tilde{O}\Bigg( \Bigg(\sum_{h = 1}^H \sum_{\Delta_h(s,a) > 0} \frac{1}{\Delta_h(s,a)} 
+ \frac{|Z|}{\Delta_{\min}} + SA \Bigg) \text{poly}(H) \Bigg),
\]
where, under non-degenerate MDPs, $0 \leq |Z| < HSA$ denotes the number of optimal state-action-step triples whose states have multiple optimal actions in \citet{xu2021fine} and the number of optimal triples in \citet{zhang2025q}. While this bound significantly reduces reliance on the large factor $1/\Delta_{\min}$ compared to prior coarse-grained gap-dependent analyses, it suffers from the high polynomial dependence on the horizon due to the Hoeffding-type bonuses, similar to \citet{yang2021q}.

These observations naturally lead to the following question:
\begin{center}
\emph{Is it possible to establish improved fine-grained gap-dependent regret bounds for
$Q$-learning algorithms that employ variance-based exploration bonuses?}
\end{center}
Answering this question is challenging. Existing techniques for deriving fine-grained regret bounds \citep{xu2021fine,zhang2025q} relate the regret to a weighted sum of exploration bonuses, which rely on the simple structure of Hoeffding-type bonuses: the bonus for each state-action-step triple depends only on its visit count. However, this property no longer holds for variance-based bonuses, which additionally depend on triple-specific variance estimates. As a result, no fine-grained regret bounds exist for model-free algorithms with variance-based bonuses, and it remains open whether UCB-Bernstein \citep{jin2018q} admits even a coarse-grained gap-dependent regret bound.

Besides regret, policy switching cost is also an important evaluation criterion in online RL, 
particularly in applications where frequent policy changes are costly or restricted, 
such as hardware placement \citep{mirhoseini2017device}, and material discovery \citep{nguyen2019incomplete}. Two types of switching costs are typically considered in the literature: local and global switching cost. Local switching cost counts the total number of policy changes across all state-step pairs between episodes, whereas global switching cost counts the total number of episodes in which the policy changes, regardless of the number of affected state–step pairs in a single switch.

In model-free RL, under worst-case MDPs, \citet{bai2019provably} modify both UCB-Hoeffding and UCB-Bernstein of \citet{jin2018q} to achieve a local switching cost of $O(H^3 S A \log T)$, while UCB-Advantage \citep{zhang2020almost} further reduces this cost to $O(H^2 S A \log T)$. Both approaches improve upon the $O(T)$ local switching cost of standard $Q$-learning algorithms (e.g., \cite{jin2018q, li2023breaking}). For gap-dependent analyses, \citet{zheng2024gap} eliminate the dependence on the action space dimension $A$ from the $\log T$ term when each state-step pair has a unique optimal action. All of these works also provide the same bound for the global switching cost. More recently, \citet{zhang2025gapdependent, zhang2025regret} further improve these results by removing the dependence on the state dimension $S$ in the $\log T$ term, resulting in the $H^2 \log T$ dependence for the global switching cost.

In this paper, we resolve the open question above by introducing UCB-Bernstein+, a refined version of the UCB-Bernstein algorithm \citep{jin2018q}. We establish the \textbf{first fine-grained gap-dependent
regret} upper bound for model-free RL using {variance-based} exploration bonuses, 
and the \textbf{best-known gap-dependent local switching cost}. Our main contributions are summarized as follows:

\textbf{Variance-Aware Fine-Grained Gap-Dependent Regret (\Cref{thm_gapregret}).} 
We develop a novel generalization of the analytical framework in \citet{zhang2025q}, which is the first in the model-free RL literature to overcome the non-trivial technical challenges of integrating variance-based exploration bonuses into fine-grained regret analysis. By successfully incorporating these bonuses, our framework can yield a sharper dependence on the horizon length $H$ compared to the Hoeffding-style analysis in \citet{zhang2025q}.

\textbf{State-of-the-Art Gap-Dependent Local Switching Cost (\Cref{thm_gapcost}).} 
Integrating our fine-grained gap-dependent framework with variance-based bonuses, we achieve more precise control over the total number of visits to suboptimal state-action-step triples (Lemma \ref{clipqb}). This leads to the best-known gap-dependent local switching cost to date, featuring improved dependence on both the horizon length $H$ and the state space dimension $S$.

\textbf{Improved Worst-Case Guarantees.} 
Our UCB-Bernstein+ introduces two refinements over the original UCB-Bernstein \citep{jin2018q}, targeting policy stability and exploration efficiency. First, we incorporate the stage design of \citet{zhang2020almost} to enable lazy policy updates. This allows us not only to establish the best-known gap-dependent switching cost, but also to match the state-of-the-art worst-case local switching cost of $O(H^2SA\log T)$ (\Cref{thm_cost}) for model-free RL \citep{zhang2020almost}. Second, we employ refined variance-based bonuses that remove polynomial dependence on the state and action space dimensions, $S$ and $A$. Crucially, although lazy updates may increase regret, our refined bonuses offset this overhead, ultimately improving the worst-case regret over the original UCB-Bernstein by removing its superlinear dependence on $S$ and $A$ (see \Cref{thm_regret}).

\textbf{Technical Novelty.} 
Our analysis introduces three primary technical innovations:

(a) To establish the variance-aware fine-grained regret by controlling the variance-based bonuses, we introduce \textbf{pseudo value functions}, a key conceptual technical tool that does not introduce hyperparameters in the design of UCB-Bernstein+. This construction helps us approximate the non-uniform bonus that involves variance estimators with the conditional variances of the optimal value function, which reflect the variance structure of the MDP. It also plays a key role in the worst-case regret analysis. See \Cref{pseudo} for the details. 

(b) To analyze the gap-dependent local switching cost, our fine-grained framework yields a variance-aware bound on the total number of visits to suboptimal state-action-step triples (Lemma \ref{clipqb}). Since $Q$-values are updated only at stage ends, this visit bound also controls the number of policy switches. Relaxing its variance-dependent terms yields the guarantee stated in \Cref{thm_gapcost}.

(c) In a significant refinement of the original UCB-Bernstein \citep{jin2018q}, we redesign the variance-based bonus through a refined concentration analysis to remove its explicit polynomial dependence on the state and action space dimensions, apart from logarithmic factors (see Lemma~\ref{Q}). This refinement is a key ingredient in achieving the improved worst-case regret in \Cref{thm_regret}.

\section{Preliminaries} 
For any $C\in \mathbb{N}$, we use $[C]$ to denote the set $\{1,2,\ldots,C\}$. 
We use $\mathbb{I}[x]$ to denote the indicator function, which equals $1$ if the event $x$ holds and $0$ otherwise. Throughout the paper, all $\arg\max$ operations use a fixed deterministic tie-breaking rule.

\textbf{Tabular Episodic Markov Decision Process (MDP).}
A tabular episodic MDP is denoted as $\mathcal{M}:=(\mathcal{S}, \mathcal{A}, H, \mathbb{P}, r)$, where $\mathcal{S}$ is the set of states with $|\mathcal{S}|=S, \mathcal{A}$ is the set of actions with $|\mathcal{A}|=A$, $H$ is the number of steps in each episode, $\mathbb{P}:=\{\mathbb{P}_h\}_{h=1}^H$ is the transition kernel so that $\mathbb{P}_h(\cdot \mid s, a)$ characterizes the distribution over the next state given the state-action pair $(s,a)$ at step $h$, and $r:=\{r_h\}_{h=1}^H$ is the collection of reward functions. We assume that $r_h(s,a)\in [0,1]$ is a {deterministic} function of $(s,a)$. 

In each episode, an initial state $s_1$ is selected arbitrarily by an adversary. Then, at each step $h \in[H]$, an agent observes a state $s_h \in \mathcal{S}$, picks an action $a_h \in \mathcal{A}$, receives the reward $r_h = r_h(s_h,a_h)$ and then transitions to the next state $s_{h+1}$. The episode terminates after step $H$ upon transitioning to the terminal state $s_{H+1}$.

For convenience, we denote $\mathbb { P }_{s,a,h}f = \mathbb{E}_{s'\sim \mathbb{P}_h(\cdot|s,a)}f(s')$, $\mathbbm { 1 }_s f = f(s)$ and  $\mathbb{V}_{s,a,h}(f) = \mathbb{P}_{s,a,h}f^2-(\mathbb{P}_{s,a,h}f)^2$ for any function $f: \mathcal{S} \rightarrow \mathbb{R}$ and state-action-step triple $(s,a,h)$.

\textbf{Policies and Value Functions.}
	A policy $\pi$ is a collection of $H$ functions $\{\pi_h: \mathcal{S} \rightarrow \Delta^\mathcal{A}\}_{h \in[H]}$, where $\Delta^\mathcal{A}$ is the set of probability distributions over $\mathcal{A}$. A policy is deterministic if for any $s\in\mathcal{S}$,  $\pi_h(s)$ concentrates all the probability mass on an action $a\in\mathcal{A}$. In this case, we denote $\pi_h(s) = a$.

Denote state value functions by $V_h^\pi(s):=\sum_{h^{\prime}=h}^H \mathbb{E}_{(s_{h^{\prime}},a_{h^{\prime}})\sim(\mathbb{P}, \pi)}\left[r_{h^{\prime}}(s_{h^{\prime}},a_{h^{\prime}}) \left. \right\vert s_h = s\right]$
and action value functions by $Q_h^\pi(s, a):=r_h(s,a)+\sum_{h^{\prime}=h+1}^H\mathbb{E}_{(s_{{h^\prime}},a_{h^\prime})\sim\left(\mathbb{P},\pi\right)}\left[ r_{h^{\prime}}(s_{h^{\prime}},a_{h^{\prime}}) \left. \right\vert s_h=s, a_h=a\right]$. 

Since the action space and the horizon are finite, there exists an optimal policy $\pi^{\star}$ that achieves the optimal value function $V_h^{\star}(s)=\sup _\pi V_h^\pi(s)=V_h^{\pi^\star}(s)$ for all $(s,h) \in \mathcal{S} \times [H]$  \citep{azar2017minimax}. Then for any $(s,a,h) \in \mathcal{S}\times\mathcal{A}\times [H]$, the Bellman equation can be written as ($V_{H+1}^{\pi}\equiv V_{H+1}^{\star}\equiv 0$):
	\begin{equation}\label{eq_Bellman}
		\left\{\begin{array} { l } 
			{ V _ { h } ^ { \pi } ( s ) = \mathbb{E}_{a'\sim \pi _ { h } ( s )}[Q _ { h } ^ { \pi } ( s , a' ) }] \\
			{ Q _ { h } ^ { \pi } ( s , a ) : =  r _ { h }(s,a) + \mathbb { P } _ { s,a,h } V _ { h + 1 } ^ { \pi } }
		\end{array}  \text { and } \left\{\begin{array}{l}
			V_h^{\star}(s)=\max _{a' \in \mathcal{A}} Q_h^{\star}(s, a') \\
			Q_h^{\star}(s, a):=r_h(s,a)+\mathbb{P}_{s,a,h} V_{h+1}^{\star}.
		\end{array}\right.\right.
	\end{equation}
For any algorithm over $K$ episodes, let $\pi^{k}$ be the policy used in the $k$-th episode, and $s_1^{k}$ be the corresponding initial state. The regret  over $T=HK$ steps is $\mbox{Regret}(T) := \sum_{k=1}^{K}(V_1^\star - V_1^{\pi^{k}})(s_1^{k}).$

\textbf{Suboptimality Gap.} We begin by formally defining the suboptimality gap for any given MDP.
\begin{definition}\label{def_sub}
    For any $(s,a,h)$, the suboptimality gap is defined as
    $\Delta_h(s,a) := V_h^\star(s) - Q_h^\star(s,a).$
\end{definition}
The Bellman optimality equation (\ref{eq_Bellman}) implies that $\Delta_h(s,a) \geq 0$ for any $(s,a,h) \in \sah$. Building on this, we introduce the minimum gap at each step.
\begin{definition}\label{def_minsubh}
Define the \textbf{minimum gap at step $h$} as
    $\Delta_{\textnormal{min},h} := \min_{(s,a) \in\mathcal{S} \times \mathcal{A},\Delta_h(s,a)>0}\Delta_h(s,a)$
    for any $h \in[H]$. If the set $\{\Delta_h(s,a)|\Delta_h(s,a)>0, (s,a) \in\mathcal{S} \times \mathcal{A}\}$ is empty, we set $\Delta_{\textnormal{min},h} = \infty$.
\end{definition}
We further define the global minimum gap as follows.
\begin{definition}\label{def_minsub}
The \textbf{minimum gap} is defined as $\dmin:=\min_{(s,a,h) \in\mathcal{S} \times \mathcal{A} \times[H],\Delta_h(s,a)>0}\Delta_h(s,a)$. 
\end{definition}
If $\{\Delta_h(s,a) \mid \Delta_h(s,a)>0,(s,a,h)\in \sah\} = \emptyset,$ then all policies are optimal, leading to a degenerate MDP. Therefore, we assume that the set is nonempty and $\dmin > 0$. Definitions \ref{def_sub} and \ref{def_minsub} and the assumption $\dmin > 0$ are standard in the gap-dependent analysis \citep{simchowitz2019non, yang2021q, xu2021fine,zheng2024gap,zhang2025q,zhang2025regret}.

\textbf{Local Switching Cost.} Following~\citet{bai2019provably,zhang2020almost}, we define the local switching cost of any algorithm run over $U>1$ episodes as follows.
\begin{definition}\label{def_switching}
The local switching cost for any algorithm with $U$ episodes is defined as $N_\textnormal{switch} := \sum_{u=1}^{U-1}\sum_{h=1}^H\sum_{s \in \mathcal{S}}\mathbb{I}\left[\pi^{u+1}_h(s) \neq \pi^{u}_h(s)\right],$ where $\pi^u$ is the policy used in the $u$-th episode.
\end{definition}

\section{Main Results}
\subsection{UCB-Bernstein+ Algorithm}
In this section, we introduce UCB-Bernstein+ in \Cref{alg1}, which enhances UCB-Bernstein \citep{jin2018q} by refining its exploration bonus and incorporating the stage design \citep{zhang2020almost}.

We first introduce the stage design. For any triple $(s, a, h)$, the samples collected for the triple are divided into consecutive \emph{stages}. The length of each stage grows approximately exponentially with a factor of $(1 + 1/H)$. Specifically, we define $e_{1} = H$ and $e_{i+1} = \left \lfloor (1 + \frac{1}{H}) e_{i} \right \rfloor$ for $i \ge 1$, and let $e_i$ represent the length of stage $i$. We also define the set of stage endpoints as $\mathcal{L} := \{\sum_{i=1}^{j} e_i \mid j\in\mathbb{N}_+\}$. 

Next, we discuss the algorithm design. UCB-Bernstein+ proceeds in episodes. At the beginning of each episode, in addition to maintaining the standard value function estimates $Q_h(s,a)$ and $V_h(s)$, the algorithm tracks the following quantities: $N_h(s,a)$, the total number of visits to $(s,a,h)$ across all episodes, and $n_h(s,a)$, the number of visits to $(s,a,h)$ in the current stage. It also tracks two additional stage-wise accumulators: $\mu_h(s,a)$ and $\sigma_h(s,a)$. They are reset to zero at the end of each stage and updated at every step as (with the convention $V_{H+1}(s) = 0$ for all $s\in\mathcal S$.)
\begin{equation}
(\mu_h, \sigma_h)(s_h,a_h) \stackrel{+}{\leftarrow}\left(V_{h+1},(V_{h+1})^2\right)(s_{h+1}).
\label{eq:alg-acc}
\end{equation}
After each transition, for the visited triple $(s_h,a_h,h)$, the algorithm updates $Q_h(s,a),V_h(s)$ and $\pi_h(s)$ only if $N_h(s,a)$ reaches a stage endpoint in $\mathcal{L}$. Only the samples for $(s,a,h)$ collected during the current stage, tracked by $n_h(s,a)$, $\mu_h(s,a)$, and $\sigma_h(s,a)$, are used in updates. Specifically,
\begin{equation}\label{update_q_bern}
Q_h(s,a) \leftarrow \min\{r_h+\mu_h/n_h+b_h, Q_h\}(s,a),\quad V_{h}(s)\leftarrow \max_{a'}Q_{h}(s,a'),
\end{equation}
and $\pi_h(s)\leftarrow \arg\max_{a'}Q_{h}(s,a')$,
where the variance-based bonus $b_h(s,a)$ in (\ref{update_q_bern}) is defined as
\begin{equation}
    \label{betabern}
        \min\left\{C_1\sqrt{\frac{W_h(s,a)\iota}{n_h(s,a)}}+\frac{C_2H\iota}{n_h(s,a)},C_3\sqrt{\frac{H^2\iota}{n_h(s,a)}}\right\}.
    \end{equation}
 Here, $C_1,C_2,C_3 > 0$ are constants, and $W_h(s,a)$ denotes the stage-wise empirical variance of the $V$-estimates collected for the triple $(s,a,h)$:
    $$W_h(s,a) =  \sigma_h(s,a)/n_h(s,a)- \left(\mu_h(s,a)/n_h(s,a)\right)^2.$$
$\iota = \log(2SAT/\delta)$, where $\delta \in (0,1)$ is the failure probability, and we treat $\iota$ as a positive tuning parameter in practice.  Here, we remark that the stages are unaligned across different triples $(s,a,h)$, meaning that the stage renewal of one triple does not necessarily imply the renewal of another.

Compared to the original UCB-Bernstein, the stage design leads to infrequent policy switching, and the refined bonus removes the polynomial dependence on $S$ and $A$.
\begin{algorithm}[H]
	\caption{UCB-Bernstein+}
	\begin{algorithmic}\label{alg1}
		\STATE{\textbf{Initialize:} For all $(s,a,h)$, let $V_{h}(s)\leftarrow H$,  $Q_{h}(s,a)\leftarrow H$,  and $(n_h,N_h,\mu_h,\sigma_h)(s,a)\leftarrow 0$. 
	   }
	   \FOR{episodes $k \leftarrow 1,2,\dots,K$}
	   \STATE{Observe $s_{1}$.}
	   \FOR{$h \leftarrow 1,2,\dots,H$}
	   \STATE{Take action $a_{h} = \pi_h(s_h)\leftarrow  \arg\max_{a'}Q_{h}(s_{h},a')$, and observe the next state $s_{h+1}$.}
	   \STATE{
	   Update the accumulators via $N_{h}(s_{h},a_{h})\stackrel{+}{\leftarrow}1$, $n_{h}(s_{h},a_{h})\stackrel{+}{\leftarrow}1$ and \Cref{eq:alg-acc}.}
	   \IF{$N_h(s_{h},a_{h})\in \mathcal{L} $ }  
	   \STATE{Update the $Q$-estimates and $V$-estimates for triple $(s_h,a_h,h)$ according to \Cref{update_q_bern}.}
	   \STATE{Set $(n_h, \mu_h, \sigma_h)(s_h, a_h)\leftarrow  0$ to reset stage-wise accumulators.} 
	   \ENDIF
	   \ENDFOR
	   \ENDFOR
	\end{algorithmic}
\end{algorithm}

\subsection{Theoretical Guarantees}
We now present our main result: the first fine-grained gap-dependent regret bound for model-free algorithms that employ variance-based bonuses. We first define the maximum conditional variance at step $h \in [H]$ as $\mathbb{Q}_h^{\star} := \max_{s,a}\{\mathbb{V}_{s,a,h}(V_{h+1}^{\star})\} \in [0, H^2]$. Then we have \Cref{thm_gapregret}.
\begin{theorem}
\label{thm_gapregret} 
Let $\iota_1 = \log(SAT)$. For UCB-Bernstein+, set the failure probability $\delta = 1/(8T)$. Then there exist numerical constants $C_1,C_2,C_3>0$ such that $\mathbb{E}\left(\textnormal{Regret}(T)\right)$ can be bounded by
\begin{align*}
O\Bigg(\inf_{\beta \in (0,H]}\Bigg\{ \sum_{h = 1}^H\sum_{\Delta_{h}(s,a)>0}\frac{(\mathbb{Q}_{h}^\star+\beta^2)H^3\iota_1}{\Delta_{h}(s,a)} +\sum_{h = 1}^H\frac{H(\sum_{t = h + 1}^H\sqrt{(\mathbb{Q}_{t}^\star+\beta^2)|Z_{\opt,t}|})^2\iota_1}{\Delta_{\min,h}}+ \frac{H^6SA\iota_1^2}{\beta}\Bigg\} \Bigg)
\end{align*}
where $Z_{\textnormal{opt},h}$ collects optimal state-action pairs at step $h$:
$Z_{{\textnormal{opt},h}} = \{(s,a) \in \mathcal{S} \times \mathcal{A}|\Delta_h(s,a) = 0\}.$
\end{theorem}
In \Cref{thm_gapregret}, $\beta \in (0,H]$ is introduced by our theoretical analysis and is not a tuning parameter; details can be found in \Cref{pseudo}. The full proof of \Cref{thm_gapregret} can be found in Appendix \ref{gapregretproof}.

Compared with the fine-grained gap-dependent regret upper bounds for the UCB-Hoeffding \citep{zhang2025q} and AMB \citep{xu2021fine} algorithms, our result is variance-aware and can yield sharper bounds in low-variance regimes. Using the worst-case bound $\mathbb{Q}_h^\star\leq H^2$ for all $h\in[H]$ and setting $\beta=H$, the leading gap-dependent terms in our bound recover the same horizon dependence as those in \citet{zhang2025q}, up to logarithmic factors. In contrast, in the ideal scenario of deterministic MDPs with $\mathbb{Q}_h^\star = 0$ for any $h \in [H]$, the AM--GM inequality implies that, for appropriate choices of $\beta$, our bound depends sublinearly on the inverse of the suboptimality gaps, improving upon the linear dependence established by \citet{zhang2025q,xu2021fine}. This sublinear improvement can be made explicit in the coarse-grained setting by substituting all $\Delta_h(s,a)$ and $\Delta_{\min,h}$ with $\dmin$. In this case, the regret bound in \Cref{thm_gapregret} becomes $O(H^{16/3}SA\iota_1^{5/3}/\dmin^{1/3})$ for $\beta = \Theta((H^2\dmin\iota_1)^{1/3})$ with $\dmin \leq O(H/\iota_1)$, exhibiting a sublinear dependence on $\dmin^{-1}$.

With the stage design, we also establish a gap-dependent upper bound on the local switching cost.
\begin{theorem}
\label{thm_gapcost}
For any $\delta\in(0,1)$, run UCB-Bernstein+ with failure-probability parameter $\delta/8$, so that $\iota=\log(16SAT/\delta)$. Then there exist numerical constants $C_1,C_2,C_3>0$ such that, with probability at least $1-\delta$, the local switching cost $N_{\textnormal{switch}}$ is upper bounded by
\begin{align*}
O\Bigg(H|Z_{\textnormal{opt}}|\log \left(\frac{K}{|Z_{\textnormal{opt}}|}+1\right) + H|Z_{\textnormal{opt}}^c| \log \Bigg(\frac{H^{\frac{5}{2}}S^{\frac{1}{2}}A^{\frac{1}{2}}\log(\frac{SAT}{\delta})
   }{|Z_{\textnormal{opt}}^c|^{\frac{1}{2}}\Delta_{\textnormal{min}}}\Bigg)\Bigg).
\end{align*}
Here, $Z_{\textnormal{opt}} = \{(s,a,h)\in \sah\mid a \in \mathcal{A}_h^\star(s)\}$, where $\mathcal{A}_h^\star(s) = \{a\mid a  \text{ maximizes }Q_h^\star(s,\cdot)\}$, and $Z_{\textnormal{opt}}^c = (\sah)\backslash Z_{\textnormal{opt}}$.
\end{theorem}
The proof of \Cref{thm_gapcost} is provided in Appendix \ref{switchcost}. Compared with \citet{zheng2024gap}, our result improves the dependence on $H$ and $S$ in the gap-dependent logarithmic factors. Our fine-grained analysis yields a variance-aware bound on visits to suboptimal triples (Lemma \ref{clipqb}), which gives the stated switching-cost bound after a worst-case relaxation. This improvement is particularly pronounced if each state-step pair has a unique optimal action, in which case the bound simplifies to
$$ O(H^2S\log (K/HS+1) + H^2SA \log (H^2\log(SAT/\delta)/\Delta_{\textnormal{min}}) ).$$
In this scenario, our result removes the $S\log S$ dependence and improves the dependence on $H$ compared to the corresponding result for UCB-Advantage in \citet{zheng2024gap}.

To complete the theoretical analysis of UCB-Bernstein+, we also provide worst-case guarantees for both regret and local switching cost in the following two theorems.
\begin{theorem}
\label{thm_regret} 
For any $\delta\in(0,1)$, run UCB-Bernstein+ with failure-probability parameter $\delta/8$, so that $\iota=\log(16SAT/\delta)$. Then there exist constants $C_1,C_2,C_3>0$ such that, with probability at least $1-\delta$, we have
$\textnormal{Regret}(T) \leq O(\sqrt{H^3SAT\iota} + H^{\frac{9}{2}}SA\iota^2).$
\end{theorem}
The proof is deferred to Appendix \ref{worstregretproof}. UCB-Bernstein \citep{jin2018q} achieves a worst-case regret bound of $O(\sqrt{H^3SAT\iota} + \sqrt{H^9S^3A^3\iota^4}).$ Our bound removes their superlinear dependence on the state and action space dimensions $S$ and $A$, thanks to the refined bonus design in \Cref{betabern}.

UCB2-Bernstein \citep{bai2019provably} achieves a logarithmic local switching cost of $O(H^3SA\log(K/A))$ while maintaining the same worst-case regret as UCB-Bernstein. Our analysis further shows that UCB-Bernstein+, with the stage design, improves upon the local switching cost of UCB2-Bernstein. The worst-case local switching cost bound is provided in the following theorem.
\begin{theorem}
\label{thm_cost} 
The local switching cost of UCB-Bernstein+ is bounded by $O(H^2SA\log(\frac{K}{SAH}+1))$.
\end{theorem}
The proof of \Cref{thm_cost} is given in Appendix \ref{worstswitchproof}. Compared to UCB2-Bernstein, our result improves the dependence on the horizon length by a factor of $H$. Moreover, it matches the best-known worst-case local switching cost achieved by UCB-Advantage \citep{zhang2020almost}.

\section{Our Technical Tool: Pseudo Value Functions}\label{pseudo}
We introduce a key technical tool for the proofs of \Cref{thm_gapregret,thm_regret}: the pseudo value functions. To facilitate the subsequent analysis, we first introduce some notation.

Let $Q_h^k(s,a)$ and $V_h^k(s)$ denote the values of $Q_h(s,a)$ and $V_h(s)$ at the beginning of episode $k$, respectively. Let $W_h^k(s,a)$ and $b_h^k(s,a)$ be the corresponding bonus terms used to update $Q_h^k(s,a)$, and let $n_h^k(s,a)$ denote the number of visits to $(s,a,h)$ during the previous stage of the triple in episode $k$. For $i \in [n_h^k(s,a)]$, let $k_{l_i(s,a,h,k)}$ be the episode index of the $i$-th visit to $(s,a,h)$ within the most recently completed stage. When $(s,a,h,k)$ is clear from the context, we use $k_{l_i}$ for brevity.

To upper bound the regret, a key step is to control the sum of value estimation errors $(Q_h^k - Q_h^\star)(s_h^k, a_h^k)$. The worst-case analysis considers the total sum over all episodes, whereas both coarse-grained and fine-grained gap-dependent analyses use a weighted sum. After the triple $(s_h^k,a_h^k,h)$ has completed its first stage, by the update rule in Equation (3) and the Bellman Equation (1), the estimation error admits the following decomposition:
\begin{align*}
    &(Q_h^{k} - Q_h^{\star})(s_h^k, a_h^k) \leq \frac{\sum_{i=1}^{n_h^k}(V_{h+1}^{k_{l_i}}-V_{h+1}^\star)(s_{h+1}^{k_{l_i}})}{n_h^k(s_h^{k},a_h^{k})} +\frac{\sum_{i=1}^{n_h^k}(\mathbbm{1}_{s_{h+1}^{k_{l_i}}} - \mathbb{P}_{s_h^{k}, a_h^{k},h})V_{h+1}^\star}{n_h^k(s_h^{k},a_h^{k})}+ b_h^{k}(s_h^{k},a_h^{k}).
\end{align*}
Here, the first RHS term establishes a recursive relationship between estimation errors at steps $h$ and $h+1$, while the second can be controlled using standard concentration inequalities. The main challenge lies in controlling the variance-based bonus in the last term. Lemma C.3 of \citet{jin2018q} directly approximates this bonus by the conditional variance of the optimal value function $V_{h+1}^\star$. However, summing the approximation errors introduces superlinear dependence on $S$ and $A$, which leads to the superlinear dependence in the regret bound of UCB-Bernstein. Moreover, it also remains unclear whether this direct approximation is suitable for fine-grained gap-dependent analysis.

To address this, we introduce the \emph{pseudo value functions} $\hat{V}_h^k(s)$ for any $(s,h,k) \in \mathcal{S} \times [H] \times [K] $:
\begin{equation}\label{eq_def_surro}
    \hat{V}_h^{k}(s) :=\max\left\{V_h^\star(s), \min \{V_h^\star(s)+\beta, V_h^{k}(s)\}\right\},
\end{equation} 
where $\beta \in (0,H]$ is an arbitrary constant. The optimism of UCB-Bernstein+ (Lemma \ref{Q} in Appendix \ref{valueproperty}), which ensures that $V_h^k(s)\geq V_h^\star(s)$ for all $s\in\mathcal S$ with high probability, implies that
\begin{equation}\label{eq_def_surro2}
    V_h^\star(s)\leq \hat{V}_h^{k}(s) = \min\left\{V_h^\star(s) + \beta, V_h^{k}(s)\right\}.
\end{equation}
Next, we discuss how the pseudo value function helps bound the variance-based bonus, which satisfies
\[
b_h^k(s,a)\leq C_b\left(\sqrt{\frac{W_h^k(s,a)\iota}{n_h^k(s,a)}}+\frac{H\iota}{n_h^k(s,a)}\right)
\]
for some universal constant $C_b>0$. Since $H\iota/n_h^k$ does not involve variance estimations, its cumulative sum can be controlled directly using Lemma \ref{wN}, and the corresponding weighted sum can be handled similarly when the weights are uniformly bounded. Therefore, we focus on the first term with the variance estimation $W_h^k(s,a)$. For a sequence $(f^j)_{j=1}^K$ given by $(V_{h+1}^j)$, $(\widehat V_{h+1}^j)$, or the constant sequence $V_{h+1}^\star$, define, when $n_h^k(s,a)>0$, the empirical mean and variance over the $n_h^k(s,a)$ visits to $(s,a,h)$ as $\hat{\mathbb{E}}[f] = \sum_{i=1}^{n_h^k(s,a)} f^{k_{l_i}}(s_{h+1}^{k_{l_i}})/n_h^k(s,a),\
\hat{\mathbb{V}}[f] = \hat{\mathbb{E}}[f^2] - (\hat{\mathbb{E}}[f])^2.$ Since $\hat{V}_h^j(s) \leq V_h^j(s)$ for any $(s,j) \in \mathcal{S}\times[K]$ by Lemma \ref{Q}, $W_h^k(s,a)$ can be decomposed as:
\begin{align}
&W_h^k(s,a) = \hat{\mathbb{V}}[V_{h+1}]\leq \hat{\mathbb{E}}\big[(V_{h+1})^2\big] - \big(\hat{\mathbb{E}}[\hat{V}_{h+1}]\big)
^2 \notag\\
&=(\hat{\mathbb{E}}[(V_{h+1})^2] -\hat{\mathbb{E}}[(\hat{V}_{h+1})^2])+(\hat{\mathbb{V}}[\hat{V}_{h+1}]-2\hat{\mathbb{V}}[V_{h+1}^\star])  + 2(\hat{\mathbb{V}}[V_{h+1}^\star]-\mathbb{V}[V_{h+1}^\star]) + 2\mathbb{V}[V_{h+1}^\star]. \label{decompsew}
\end{align}
Next, we control the terms on the right-hand side of \Cref{decompsew} separately. 
The first term satisfies
\begin{align}
    \hat{\mathbb{E}}\left[(V_{h+1})^2\right] -\hat{\mathbb{E}}\big[(\hat{V}_{h+1})^2\big]
    &\leq 2H\hat{\mathbb{E}}[V_{h+1}-\hat{V}_{h+1}] \nonumber\\ 
    &\leq 2H\hat{\mathbb{E}}\left[\left(V_{h+1}-V_{h+1}^\star\right)\mathbb{I}\left[V_{h+1}-V_{h+1}^\star\geq \beta\right]\right].\label{temp_lable_first_term}
\end{align}
Here, the first inequality follows from factoring the difference of two squares and the fact that $V_{h+1}^j, \hat{V}_{h+1}^j \leq H$ for any $j \in [K]$.  The last inequality holds because for any $j \in [K]$, $V_{h+1}^j - \hat{V}_{h+1}^j > 0$ can occur only when
$V_{h+1}^j - V_{h+1}^\star \geq \beta$ due to \Cref{eq_def_surro2}, and
$V_{h+1}^j - \hat{V}_{h+1}^j \leq V_{h+1}^j - V_{h+1}^\star$. The RHS of \Cref{temp_lable_first_term} characterizes the value estimation error during the early exploration phase and can be properly controlled (see Lemma \ref{clipq} in Appendix \ref{valueproperty} for details). In addition to the refined bonus design, this step is another key to avoiding the superlinear dependence on the state and action space dimensions in the worst-case regret of \citet{jin2018q}. We also remark that this step resembles the control of the reference settling error for algorithms using variance reduction \citep{zhang2020almost,li2023breaking,zhang2025regret}. However, it differs in that the target value gap $\beta$ is conceptual and does not appear in the UCB-Bernstein+ design.

The second term in \Cref{decompsew}, $\hat{\mathbb{V}}[\hat{V}_{h+1}] - 2 \hat{\mathbb{V}}[V_{h+1}^\star]$, can be bounded by $2 \hat{\mathbb{V}}[\hat{V}_{h+1} - V_{h+1}^\star]$ with the Cauchy–Schwarz inequality. It can be further bounded by $O(\beta^2)$ since $|\hat{V}_{h+1}^j - V_{h+1}^\star| \le \beta$ for any $j \in[K]$ by \Cref{eq_def_surro2}. The third term in \Cref{decompsew}, $2(\hat{\mathbb{V}}[V_{h+1}^\star] - \mathbb{V}[V_{h+1}^\star])$, focuses on the empirical estimation error of the variance and can be bounded with the Azuma-Hoeffding inequality. See \Cref{emdiff} in Appendix \ref{gapregretproof} for details.

Finally, bounding the last term, $\mathbb{V}[V_{h+1}^\star]$, highlights the advantage of variance-based exploration bonuses. In the worst-case analysis, we approximate it with $\mathbb{V}[V_{h+1}^\pi]$ as in \citet{jin2018q}, achieving improved $H$-dependence compared to algorithms using Hoeffding-type bonuses. For the gap-dependent analysis, we bound it with $\mathbb{Q}_h^\star$, which captures the MDP's variance structure. This completes the explanation of how the variance-aware bonus is controlled.

Using the pseudo value function and the bounding procedure for variance-based exploration bonuses described above, we derive the fine-grained gap-dependent regret bound by separately analyzing optimal and suboptimal state-action-step triples and recursively summing the weighted value estimation errors across steps (Lemmas \ref{recurQb} and \ref{recurQbtotal}), following the framework in \citet{zhang2025q}. Full details are provided in Appendix~\ref{gapregretproof}. Similarly, the worst-case regret can be obtained by controlling the full sum of value estimation errors using the recursive procedure in \citet{jin2018q,zhang2020almost}, with full details in Appendix~\ref{worstregretproof}.

\section{Numerical Experiments}
\label{mainexp}
In this section, we conduct numerical experiments.
We first evaluate UCB-Bernstein+ against representative model-free algorithms using Hoeffding-type or Bernstein-type variance-based bonuses, including UCB-Hoeffding~\citep{jin2018q}, UCB-Bernstein~\citep{jin2018q}, and UCB2-Bernstein~\citep{bai2019provably}, in a synthetic environment. We assess their performance in terms of both regret and local switching cost, with particular emphasis on the comparison with UCB2-Bernstein, the low-switching-cost variant of UCB-Bernstein. For evaluation, $V^\star$ and each episode's greedy-policy value are computed exactly by backward dynamic programming using the true MDP, which is unavailable to the learning algorithms. The local switching cost is computed according to Definition \ref{def_switching} by comparing consecutive greedy policies over all state--step pairs.

We further provide additional comparisons of the same four algorithms in the GridWorld environment \citep{Dominguesrlberry2021} in Appendix~\ref{exp:grid}. Finally, in Appendix~\ref{exp:baseline}, we compare UCB-Bernstein+ with two state-of-the-art baselines: Q-EarlySettled-LowCost~\citep{zhang2025regret}, which achieves the best-known regret guarantee, and UCB-Advantage~\citep{zhang2020almost}, which achieves the best-known local switching cost guarantee. These additional comparisons further demonstrate the empirical effectiveness of UCB-Bernstein+.

We now describe the setup for comparing the four algorithms in the synthetic environment. We consider four \textbf{experiment scales} with $(H,S,A,K) = (5,3,2,3\times10^4)$, $(5,5,5,10^5)$, $(7,8,6,6 \times 10^5)$, and $(10,15,10,2 \times 10^6)$. For each $(s,a,h)$, the reward $r_h(s,a)$ is sampled independently from the uniform distribution over $[0,1]$, and the transition kernel $\mathbb{P}_h(\cdot \mid s,a)$ is drawn uniformly from the probability simplex over S states. The initial state of each episode is chosen uniformly at random from the state space. We set $\iota = 1$ and use the bonus coefficient $\sqrt{2}$ for the leading term that scales as $\sqrt{1/n}$ in all four algorithms. Here, $n$ represents the number of visits. For the lower-order term that scales as $1/n$ in the three algorithms with Bernstein-type bonuses, we use the coefficient $1$. Accordingly, for the UCB-Bernstein+ algorithm, we set $C_1 = C_3 = \sqrt{2}$ and $C_2 = 1$ in \Cref{betabern}.
\begin{figure}[H]
    \centering
    \begin{subfigure}[b]{0.475\textwidth}
        \centering
        \includegraphics[width=\linewidth]{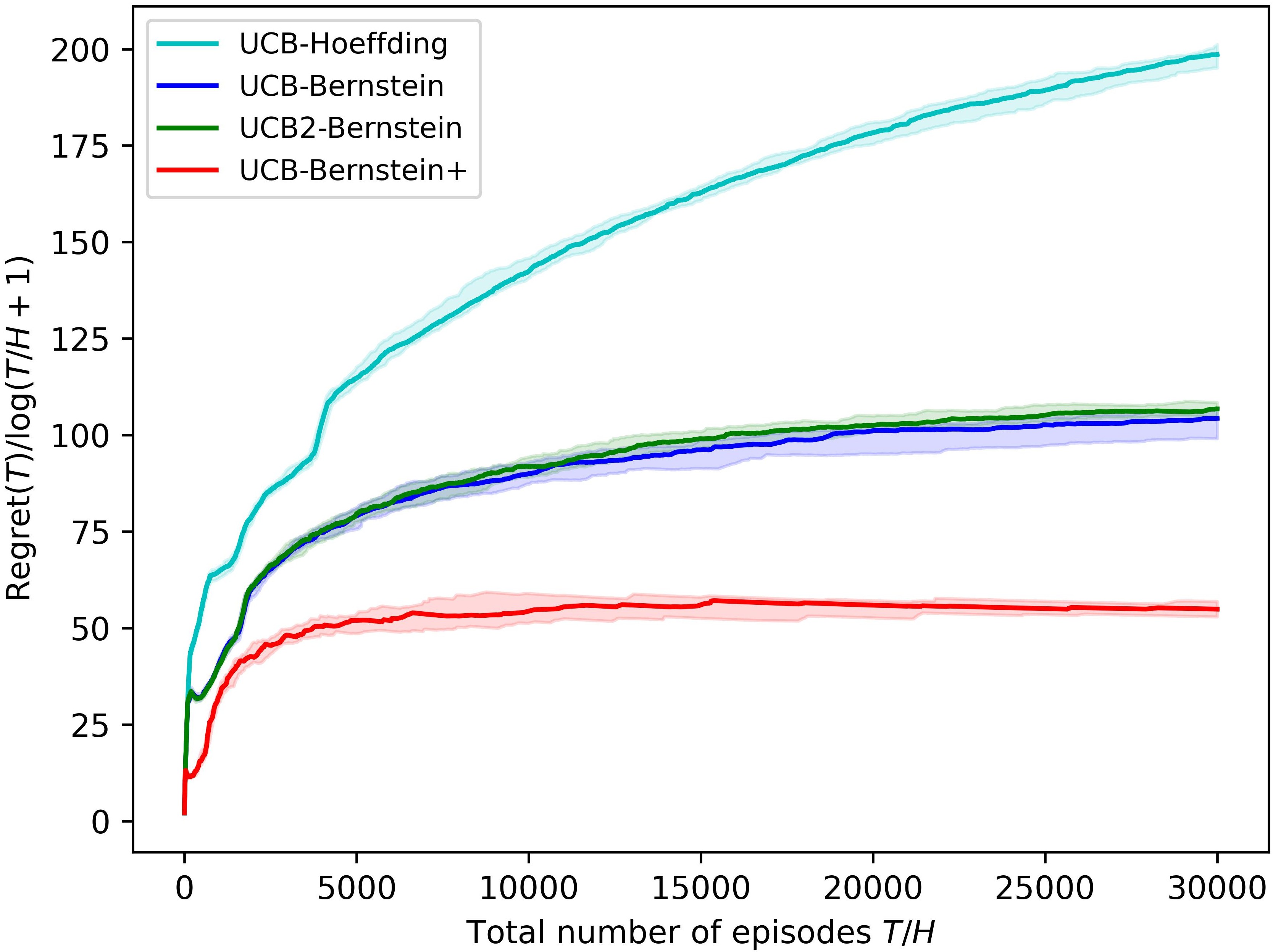}
        \caption{Regret for $(H,S,A) = (5,3,2)$}
    \end{subfigure}
    \hfill
    \begin{subfigure}[b]{0.485\textwidth}
        \centering
        \includegraphics[width=\linewidth]{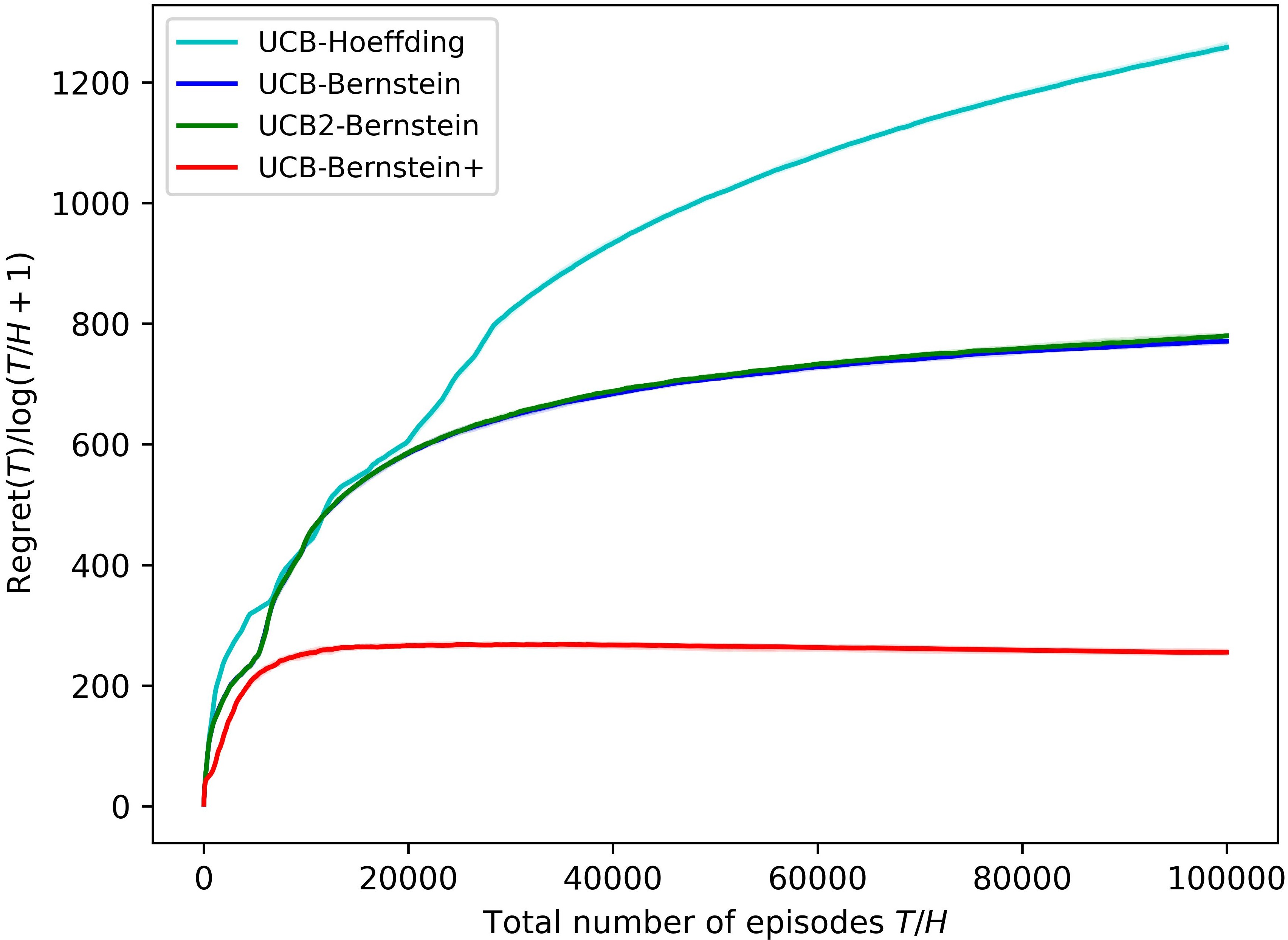}
        \caption{Regret for $(H,S,A) = (5,5,5)$}
    \end{subfigure}  
    \begin{subfigure}[b]{0.475\textwidth}
        \centering
        \includegraphics[width=\linewidth]{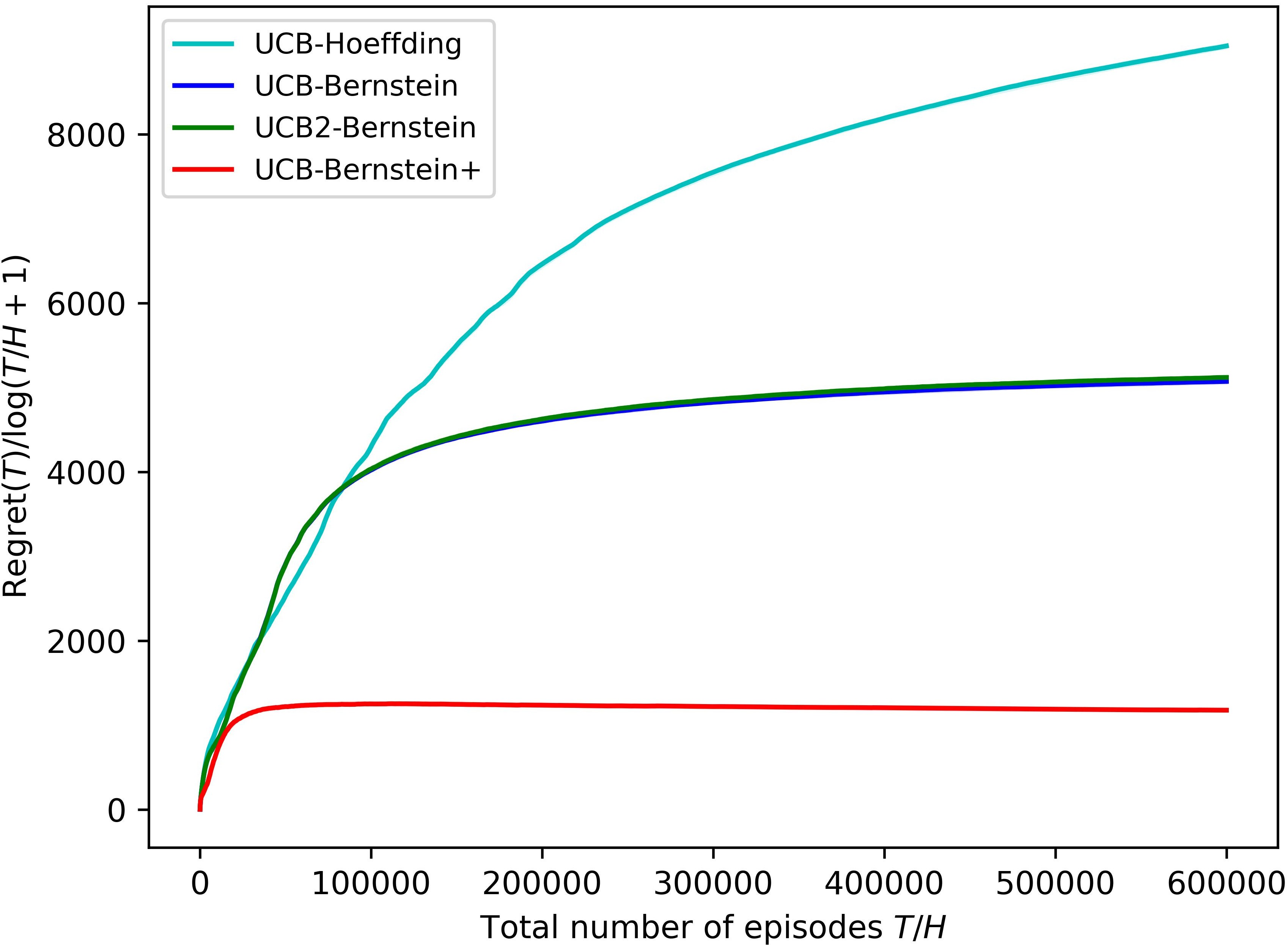}
        \caption{Regret for $(H,S,A) = (7,8,6)$}
    \end{subfigure}
    \hfill
    \begin{subfigure}[b]{0.485\textwidth}
        \centering
        \includegraphics[width=\linewidth]{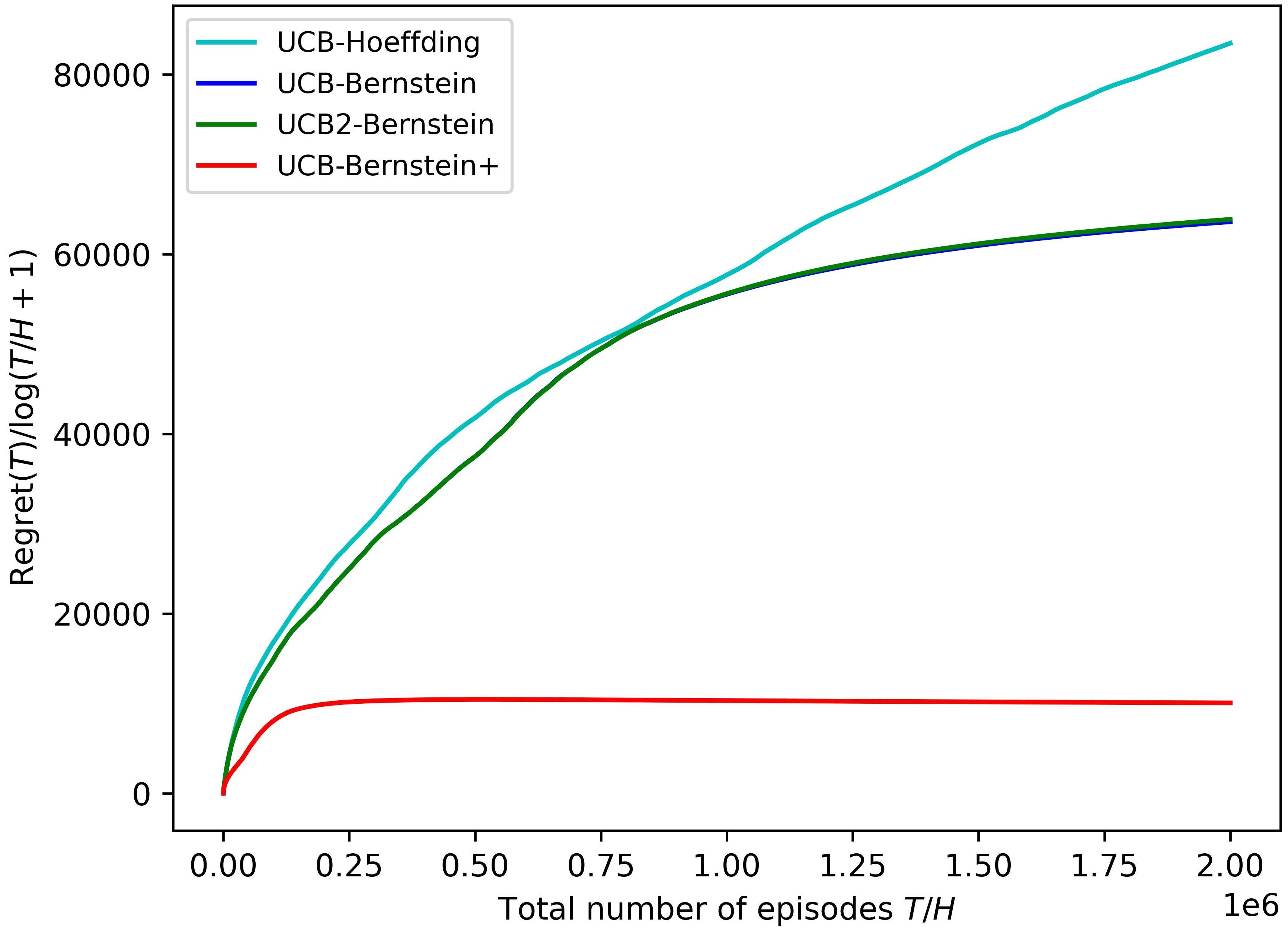}
        \caption{Regret for $(H,S,A) = (10,15,10)$}
    \end{subfigure}
    \caption{Regret comparison of different algorithms.}
    \label{fig:regret_comparison}
\end{figure}

To quantify uncertainty, for each experiment scale, we generate one MDP and run 10 independent trials for each algorithm on this fixed MDP. In the following figures, solid lines represent the median over the 10 trials, and shaded regions indicate the 10th--90th percentile intervals.

\Cref{fig:regret_comparison} presents the regret comparison results, where we plot $\textnormal{Regret}(T)/\log(K+1)$ against the number of episodes $K$. We observe that UCB-Bernstein+ achieves the best regret among the compared four algorithms. In addition, across all experiment scales, its curve flattens as $K$ increases, suggesting approximately logarithmic regret growth, qualitatively consistent with \Cref{thm_gapregret}.

We further compare the local switching costs of UCB2-Bernstein and UCB-Bernstein+, with the results shown in \Cref{fig:costcomparison}. We plot $\textnormal{Local Switching Cost}/\log(K+1)$ against the number of episodes $K$. The results show that UCB-Bernstein+ consistently incurs a substantially lower local switching cost than UCB2-Bernstein across the experiment scales. Moreover, the local switching-cost curve of UCB-Bernstein+ gradually stabilizes as $K$ increases, suggesting approximately logarithmic growth of its local switching cost with the number of episodes. These empirical observations are consistent with both the gap-dependent and worst-case local switching cost guarantees in \Cref{thm_gapcost,thm_cost}.
\begin{figure}[H]
    \centering
    \begin{subfigure}[b]{0.48\textwidth}
        \centering
        \includegraphics[width=\linewidth]{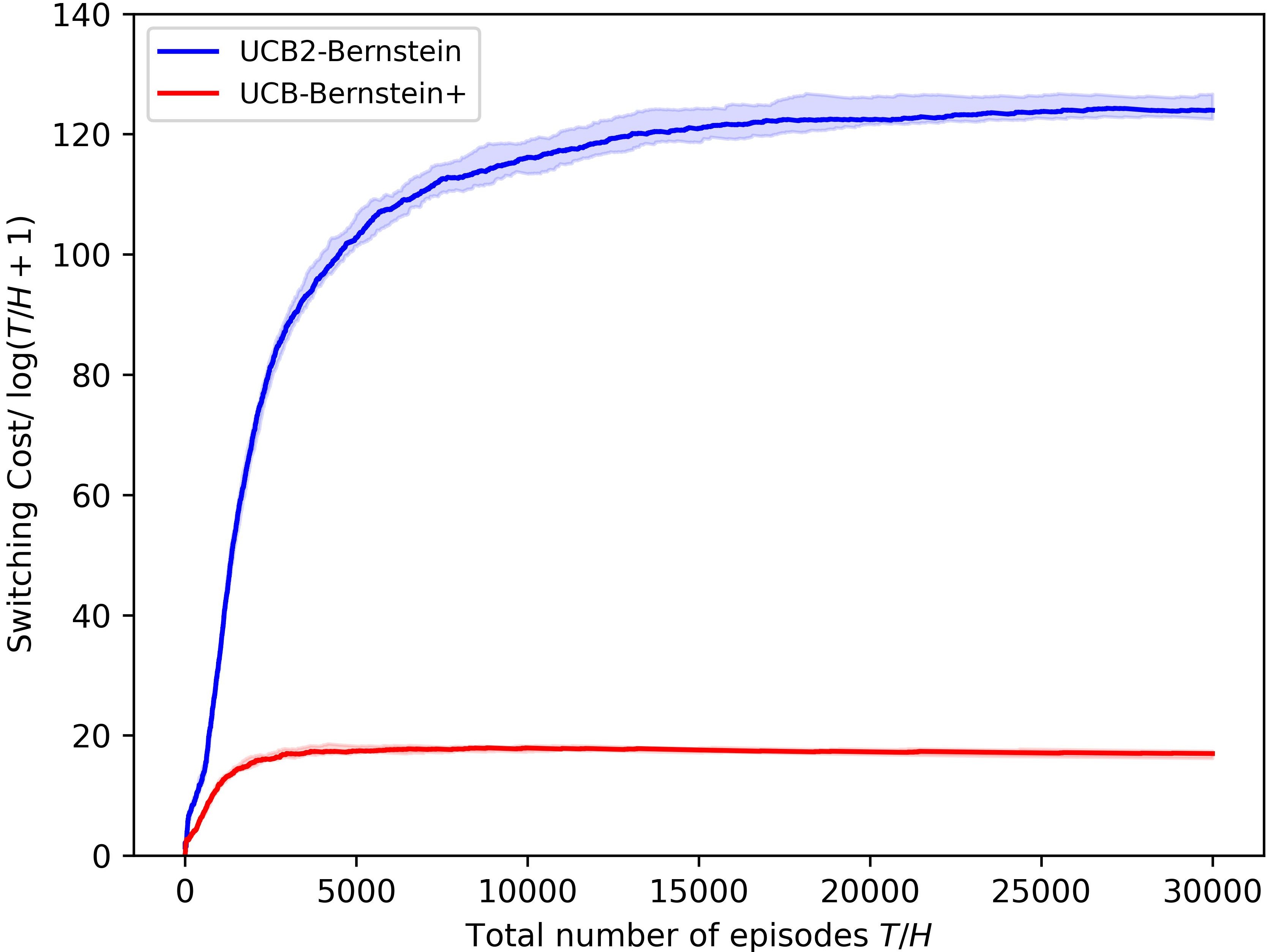}
        \caption{Local switching cost for $(H,S,A) = (5,3,2)$}
    \end{subfigure}
    \hfill
    \begin{subfigure}[b]{0.49\textwidth}
        \centering
        \includegraphics[width=\linewidth]{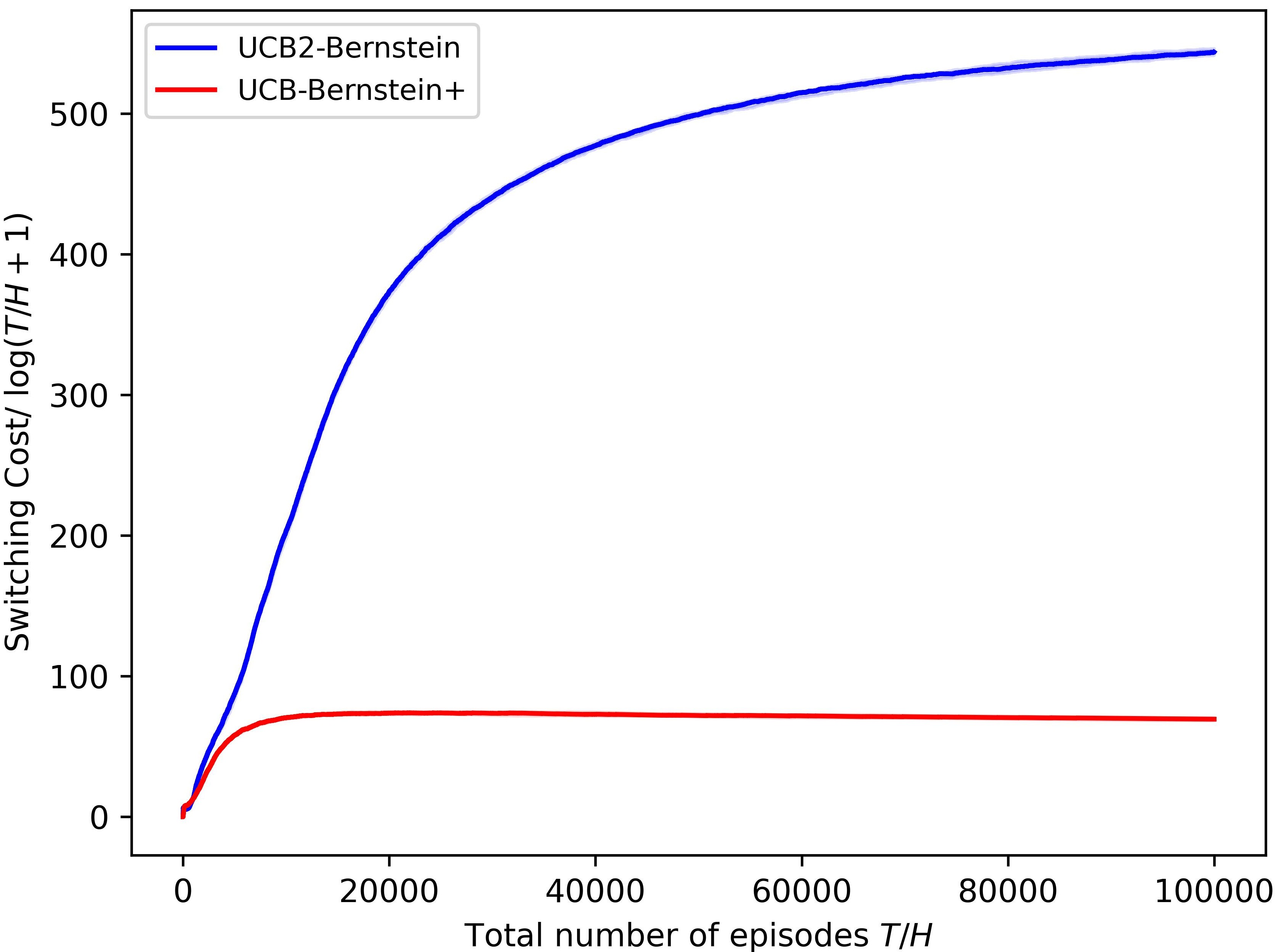}
        \caption{Local switching cost for $(H,S,A) = (5,5,5)$}
    \end{subfigure}    
    \vspace{0.2em}
    
    \begin{subfigure}[b]{0.48\textwidth}
        \centering
        \includegraphics[width=\linewidth]{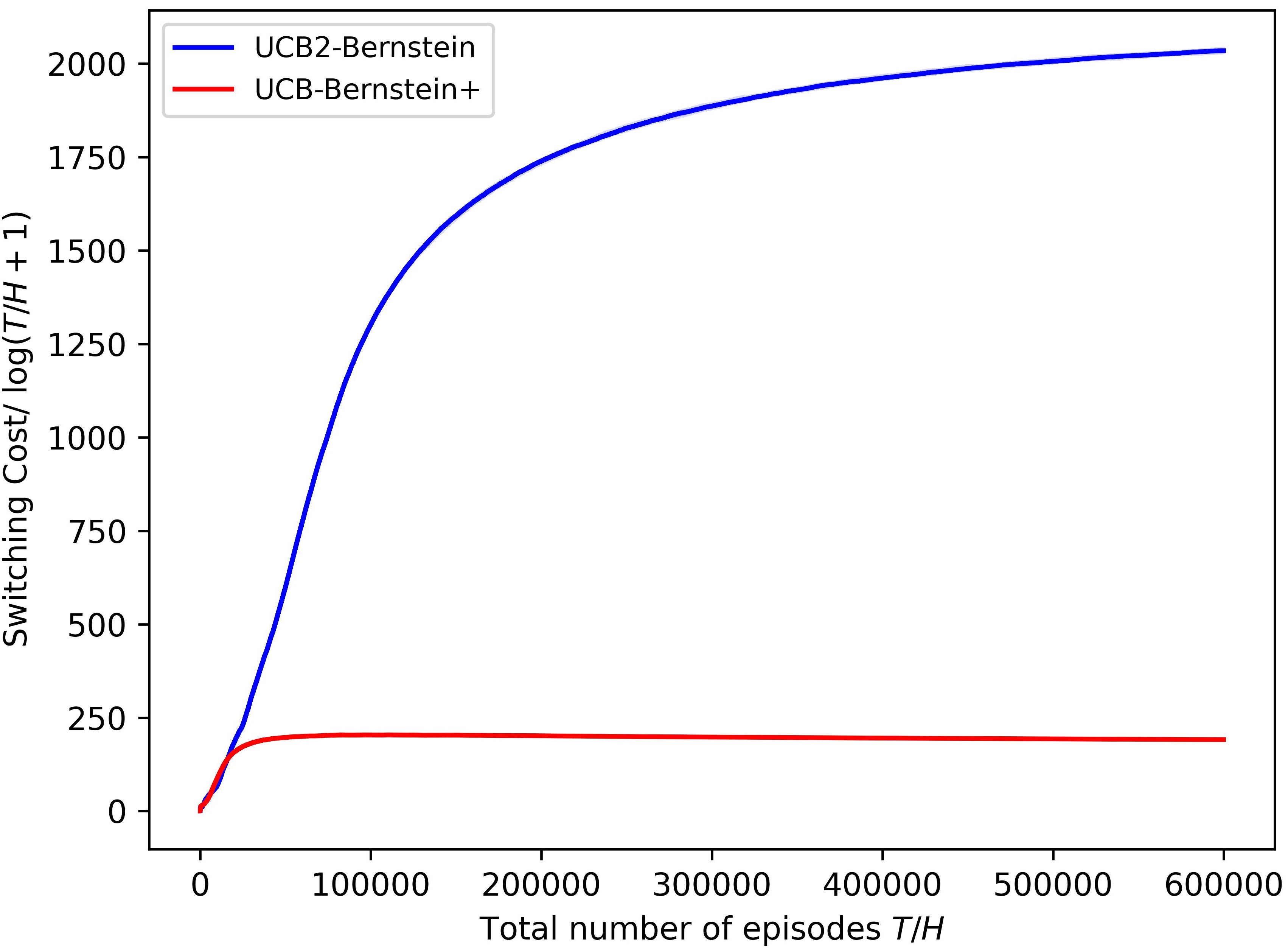}
        \caption{Local switching cost for $(H,S,A) = (7,8,6)$}
    \end{subfigure}
    \hfill
    \begin{subfigure}[b]{0.49\textwidth}
        \centering
        \includegraphics[width=\linewidth]{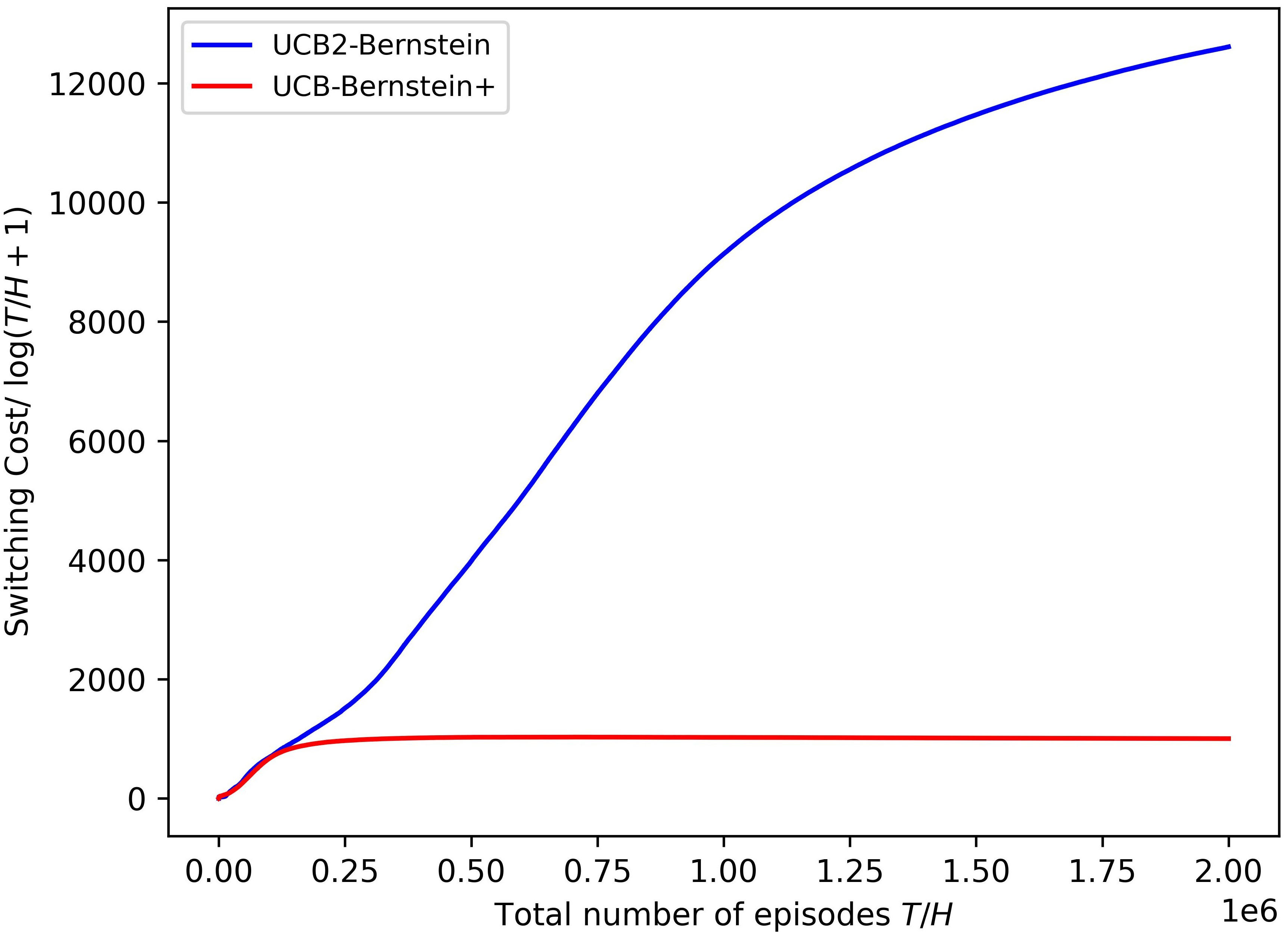}
        \caption{\footnotesize Local switching cost for $(H,S,A) = (10,15,10)$}
    \end{subfigure}
    \caption{Local switching cost comparison of different algorithms.}
    \label{fig:costcomparison}
\end{figure}
\section{Conclusion}
In this work, we study model-free online RL for episodic tabular MDPs, focusing on gap-dependent regret and policy switching cost. We propose UCB-Bernstein+, a refined UCB-Bernstein algorithm with variance-based exploration bonuses. Our analysis establishes the first fine-grained gap-dependent regret upper bound for variance-aware model-free online RL, closing the gap left by prior Hoeffding-based analyses. By incorporating a stage-wise policy update design, UCB-Bernstein+ further achieves the best-known gap-dependent local switching cost to date. We also establish improved worst-case guarantees for both regret and local switching cost over the original UCB-Bernstein algorithm. Numerical experiments further demonstrate the empirical effectiveness of UCB-Bernstein+.

\bibliography{main}
\bibliographystyle{plainnat}

\newpage
\appendix
In the appendix, Appendix \ref{related} reviews related work. Appendix \ref{experiment} presents additional numerical experiments. Appendix \ref{notation} summarizes the notation used throughout the paper. Appendix \ref{keylemma} provides several auxiliary lemmas that support our analysis. Appendix \ref{property} summarizes basic properties of UCB-Bernstein+, while Appendix \ref{probabilityevent} introduces the high-probability events used in the proofs. Appendix \ref{valueproperty} establishes several properties of the value function estimates. The proofs of \Cref{thm_gapregret} and \Cref{thm_gapcost} are provided in Appendix \ref{gapregretproof} and Appendix \ref{switchcost}, respectively, whereas Appendix \ref{worstregretproof} and Appendix \ref{worstswitchproof} contain the proofs of \Cref{thm_regret} and \Cref{thm_cost}, respectively.

\section{Related Work}
\label{related}

\textbf{Online RL for Tabular Episodic MDPs.}  
For tabular episodic MDPs in the worst-case regret setting, \citet{auer2008near}, \citet{agrawal2017optimistic}, \citet{azar2017minimax}, \citet{agarwal2020model}, \citet{zanette2019tighter}, \citet{zhang2021reinforcement}, and \citet{zhang2024settling} study model-based algorithms. Notably, \citet{zhang2024settling} propose the MVP algorithm that achieves a regret of $\tilde{O}(\min\{\sqrt{SAH^2T},T\})$, matching the information-theoretic lower bound up to logarithmic factors. Model-free algorithms have been studied by \citet{jin2018q}, \citet{zhang2020almost}, \citet{li2023breaking}, \citet{menard2021ucb}, and \citet{zhang2025regret}, with the latter four achieving a near-optimal regret of $\tilde{O}(\sqrt{SAH^2T})$. Several works also consider online federated RL, including \citet{liang2026provably}, \citet{zheng2023federated}, \citet{labbi2024federated}, \citet{zheng2024federated}, and \citet{zhang2025regret}, with the last three attaining near-optimal regret bounds up to logarithmic factors.

\textbf{Gap-Dependent RL.}  
When a strictly positive suboptimality gap exists, logarithmic regret becomes achievable. Early studies establish asymptotic logarithmic regret bounds \citep{auer2007logarithmic, tewari2008optimistic}, while more recent work develops non-asymptotic bounds \citep{jaksch2010near, ok2018exploration, simchowitz2019non, he2021logarithmic}. Specifically, \citet{jaksch2010near} design a model-based algorithm whose regret depends on the policy gap rather than the action gap considered in this paper. \citet{ok2018exploration} derive problem-specific logarithmic-type lower bounds for both structured and unstructured MDPs, and \citet{simchowitz2019non} extend the model-based algorithm of \citet{zanette2019tighter} to obtain logarithmic regret bounds. More recently, \citet{chen2025sharp} further improve model-based gap-dependent results. Logarithmic regret bounds have also been established under linear function approximation \citep{he2021logarithmic}, and \citet{nguyen2023instance} provide gap-dependent guarantees for offline RL with linear function approximation.

For model-free algorithms, \citet{yang2021q} show that UCB-Hoeffding \citep{jin2018q} achieves a gap-dependent regret bound of $\tilde{O}(H^6SA/ \Delta_{\textnormal{min}})$. This bound is subsequently refined by \citet{xu2021fine} with the Adaptive Multi-step Bootstrap (AMB) algorithm. \citet{zheng2024gap} provide gap-dependent analyses for algorithms with reference-advantage decomposition \citep{zhang2020almost, li2023breaking}, and  \citet{zhang2025gapdependent} and \citet{zhang2025regret} extend gap-dependent analysis to federated $Q$-learning. More recently,  \citet{zhang2025q} provide the first fine-grained gap-dependent regret upper bound for UCB-Hoeffding.

For RL with linear function approximation, gap-dependent regret bounds have also been studied. For online RL, \citet{he2021logarithmic} provided the first gap-dependent regret bound, showing that LSVI-UCB \citep{jin2020provably} achieves $\tilde{O}(d^3 H^5 / \dmin)$ in linear MDPs, and UCRL-VTR~\citep{ayoub2020model} achieves $\tilde{O}(d^2 H^5 / \dmin)$ in linear mixture MDPs. \cite{papini2021reinforcement} established constant-regret guarantees for LSVI-UCB under the UNISOFT condition. Subsequently, \cite{zhang2024achieving} showed that Cert-LSVI-UCB can also achieve constant regret with high probability under their corresponding structural conditions, independent of the total number of steps $T$. More recently, \cite{zhang2026gap} established a gap-dependent analysis for the nearly minimax-optimal LSVI-UCB++ algorithm.

\textbf{RL with Low Switching Cost.}  
Research on low switching cost aims to minimize the number of policy switches while maintaining regret comparable to fully adaptive algorithms. \citet{bai2019provably} introduce this problem and propose two $Q$-learning algorithms with delayed updates and achieve a local switching cost of $\tilde{O}(SAH^3 \log T)$. This line of work is further advanced by \citet{zhang2020almost}, which improves both the regret upper bound and the local switching cost. More recently, \citet{qiao2022sample} propose a model-based algorithm with the global switching cost $\tilde{O}(\log \log T)$, and \citet{zhang2025regret} introduce Q-EarlySettled-LowCost, achieving the best-known regret among model-free algorithms while maintaining logarithmic global switching cost. \citet{gao2021provably} also provide a logarithmic global switching cost bound in RL with linear function approximation.

\section{Additional Numerical Experiments}
\label{experiment}

\subsection{GridWorld Experiments}
\label{exp:grid}
We also compare the four algorithms based on Hoeffding-type or Bernstein-type variance-based bonuses in the GridWorld environment, including UCB-Hoeffding, UCB-Bernstein, UCB2-Bernstein, and UCB-Bernstein+. In this environment, the agent navigates a grid to reach a target state.

Following a similar experimental setup to~\citet{labbi2024federated}, we consider two grid sizes, $3\times 3$ and $5\times 5$, each with a wall at coordinate $(1,1)$. This gives $S=8$ and $S=24$ valid states, respectively. The horizons are set to $H=10$ and $H=15$, and the numbers of episodes range up to $K=5\times10^4$ and $K=10^5$, respectively. The agent starts from coordinate $(0,0)$ and, at each step, chooses one of four actions: up, down, left, or right. Under the transition kernel, the agent moves to the intended neighboring cell with probability $0.8$ and to a uniformly random other valid neighboring cell with the remaining probability. If the intended action would move the agent to an invalid cell, the agent remains in its current cell with probability $1$. The target states are $(2,2)$ for the $3\times 3$ grid and $(4,4)$ for the $5\times 5$ grid, and each target state is absorbing. The per-step reward is $1$ if the realized next state is the target state and $0$ otherwise.

In these experiments, the algorithms observe both the realized next state and the realized reward. UCB-Hoeffding, UCB-Bernstein, and UCB2-Bernstein use the observed reward in each Q-update. For UCB-Bernstein+, the observed rewards are accumulated within each stage, and their empirical mean replaces $r_h(s,a)$ in the stage update in \Cref{update_q_bern}; the next-state value samples are averaged separately. The empirical variance used in the Bernstein-type bonuses is computed from next-state value samples only, without a separate reward-variance or reward--value covariance term. For these four algorithms, we adopt the same parameter choices as specified in~\Cref{mainexp}. Because the reward can be stochastic conditional on $(s,a)$, this experiment lies outside the deterministic-reward setting of our theoretical guarantees and serves as an empirical robustness check.

To quantify uncertainty, for each experiment scale, we run 10 independent trials for each algorithm on the same grid. In the figures, solid lines represent the median over the 10 trials, and shaded regions indicate the 10th--90th percentile intervals.
\begin{figure}[H]
    \centering
    \begin{subfigure}[b]{0.48\textwidth}
        \centering
        \includegraphics[width=\linewidth]{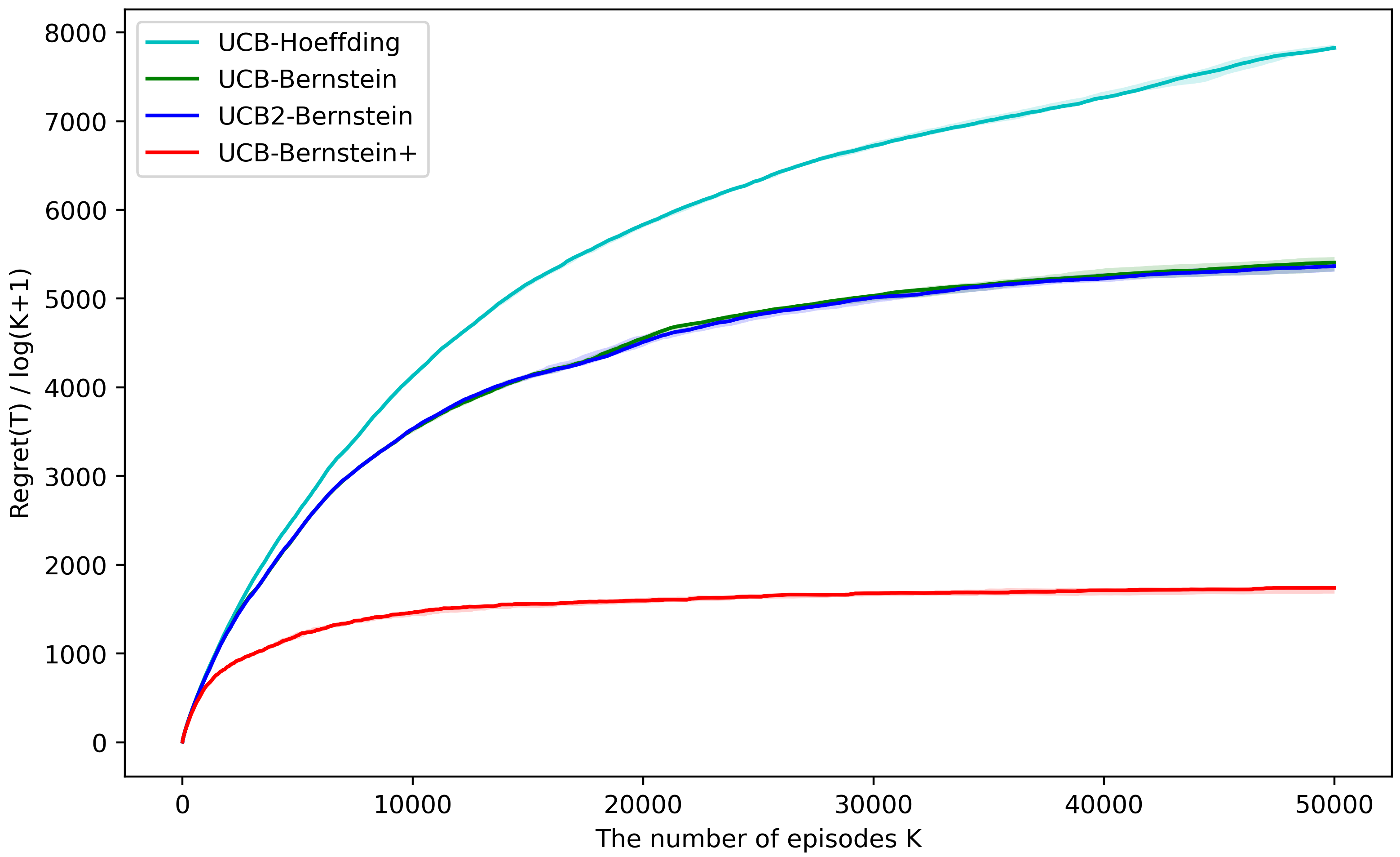}
        \caption{Regret on the $3\times 3$ GridWorld}
    \end{subfigure}
    \hfill
    \begin{subfigure}[b]{0.48\textwidth}
        \centering
        \includegraphics[width=\linewidth]{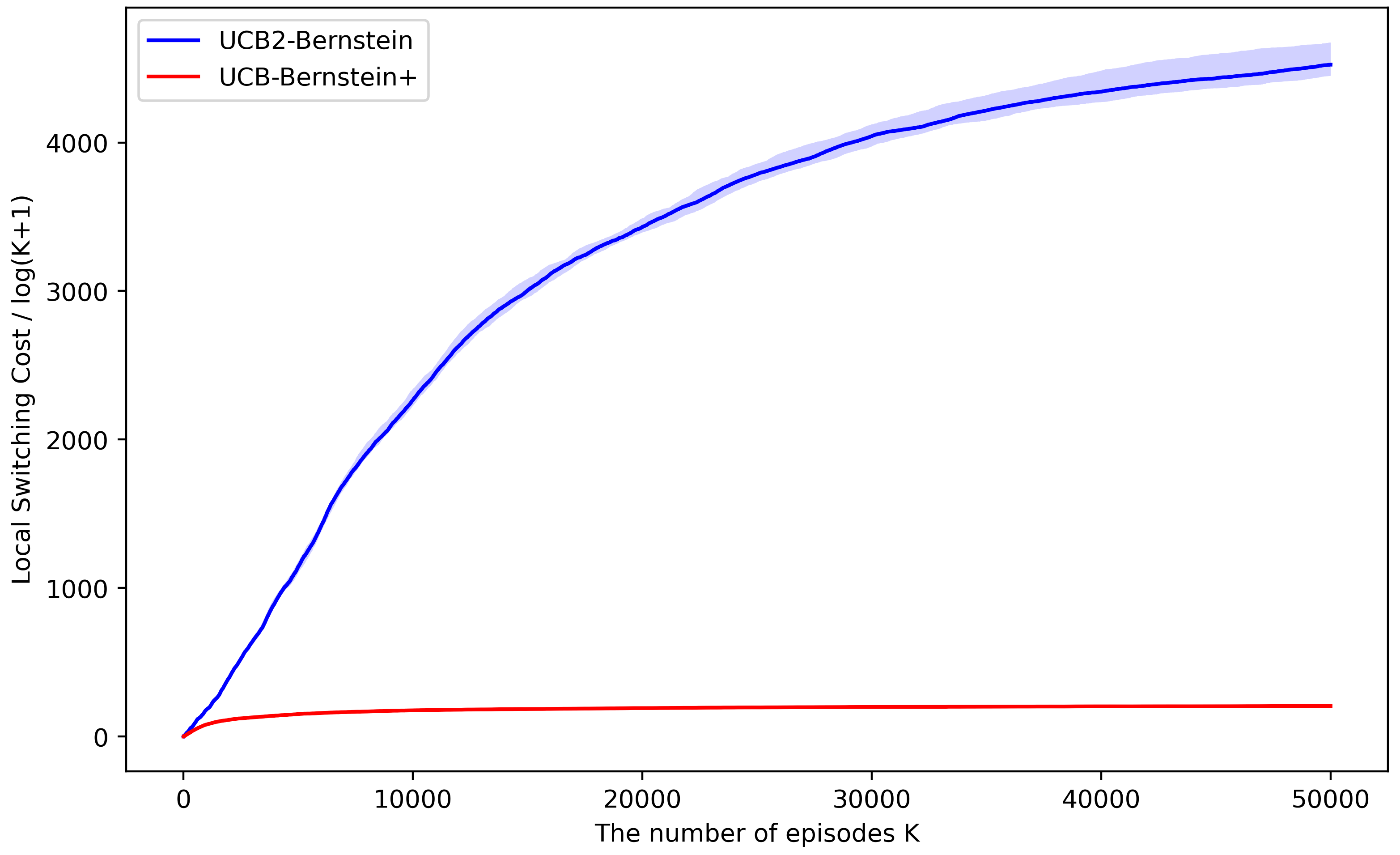}
        \caption{Local switching cost on the $3\times 3$ GridWorld}
    \end{subfigure}

    \vspace{0.3em}
    \begin{subfigure}[b]{0.48\textwidth}
        \centering
        \includegraphics[width=\linewidth]{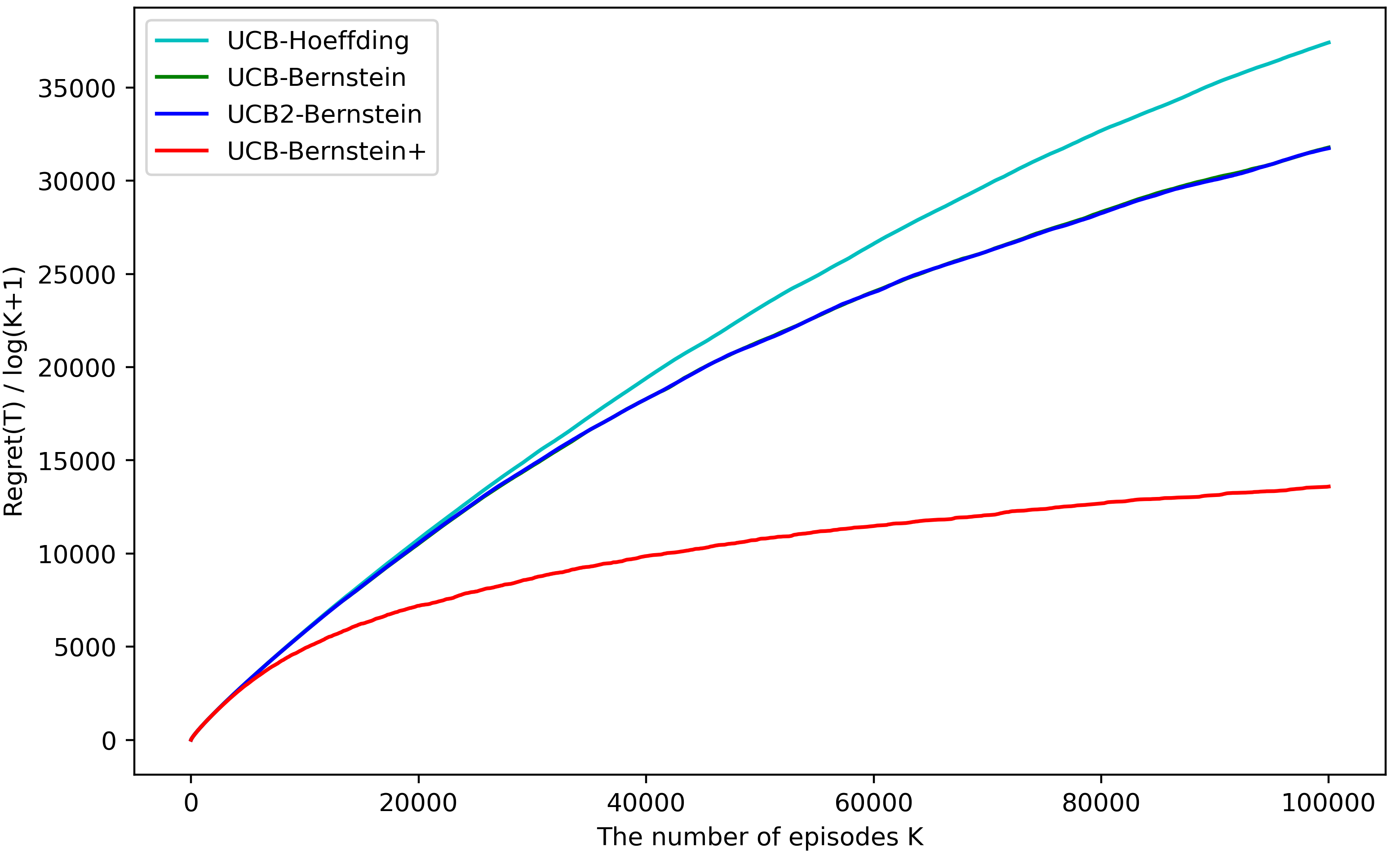}
        \caption{Regret on the $5\times 5$ GridWorld}
    \end{subfigure}
    \hfill
    \begin{subfigure}[b]{0.48\textwidth}
        \centering
        \includegraphics[width=\linewidth]{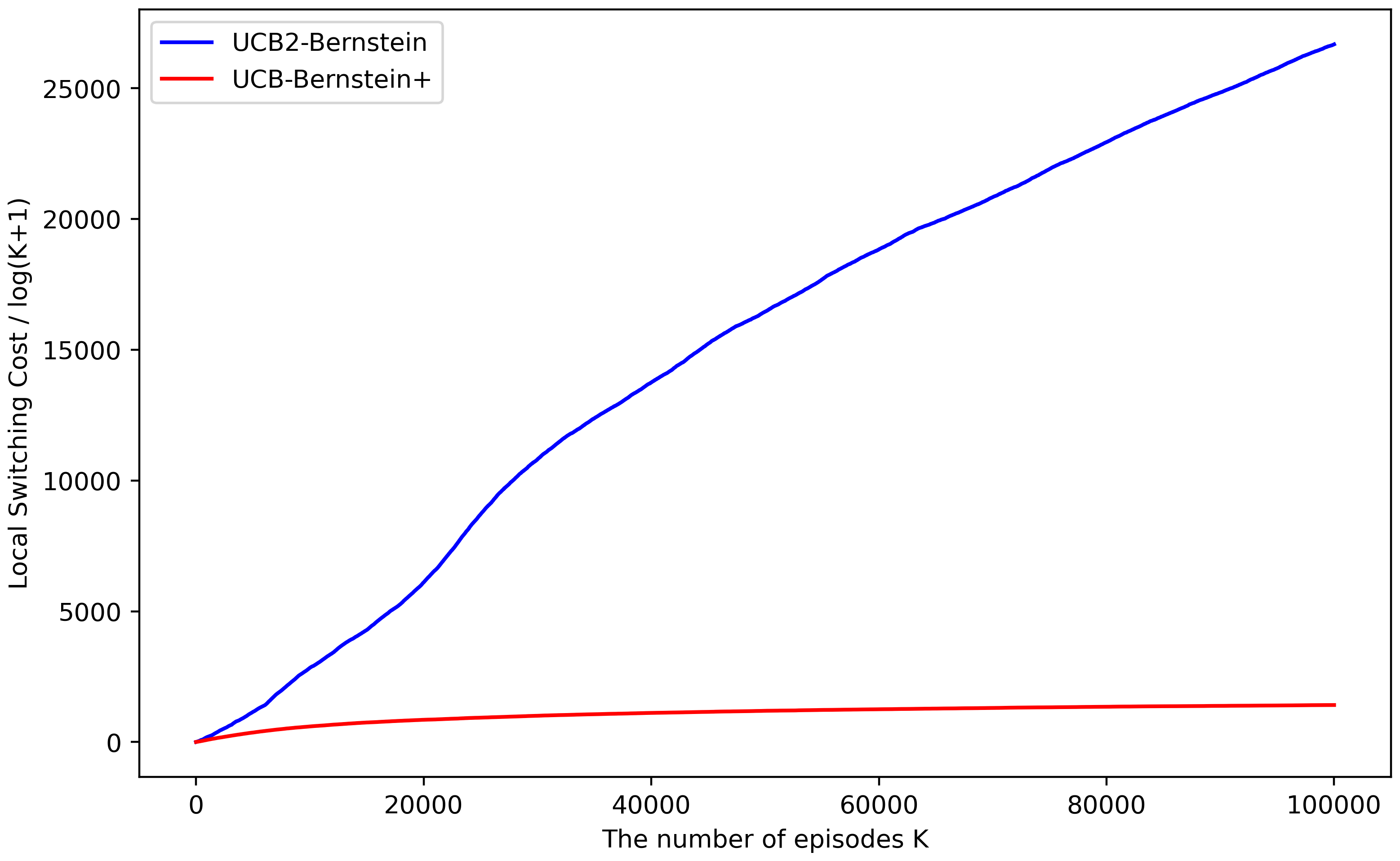}
        \caption{Local switching cost on the $5\times 5$ GridWorld}
    \end{subfigure}
    \caption{GridWorld results.}
    \label{fig:gridworld_comparison}
\end{figure}
\Cref{fig:gridworld_comparison} presents the GridWorld comparison results for the two grid sizes. Plots (a) and (c) show the normalized regret $\textnormal{Regret}(T)/\log(K+1)$ as a function of the number of episodes $K$ for each algorithm, while plots (b) and (d) show the normalized local switching cost $\textnormal{Local Switching Cost}/\log(K+1)$ as a function of $K$.

The empirical findings are consistent with the synthetic-environment results in \Cref{mainexp}. From the regret plots (a) and (c), UCB-Bernstein+ achieves the lowest regret among the four compared algorithms. Moreover, across the two experiment scales, its normalized regret curve becomes flatter as $K$ increases, suggesting approximately logarithmic cumulative regret growth. From the local switching-cost plots (b) and (d), UCB-Bernstein+ incurs a substantially lower normalized local switching cost than UCB2-Bernstein. Its normalized switching-cost curve also stabilizes as $K$ increases, suggesting approximately logarithmic growth of the cumulative local switching cost.

\subsection{Experimental Comparison with State-of-the-Art Baselines}
\label{exp:baseline}
In this subsection, we further evaluate the empirical performance of UCB-Bernstein+ by comparing it with two state-of-the-art model-free RL baselines in a synthetic environment: Q-EarlySettled-LowCost \citep{zhang2025regret}, which achieves the best-known regret guarantee, and UCB-Advantage \citep{zhang2020almost}, which achieves the best-known local switching cost guarantee.

Similar to the synthetic environment introduced in Appendix B.1 of \citet{zhang2025regret}, we consider two experiment scales with $(H,S,A,K) = (5,3,2,3\times10^5)$, and $(7,10,5,10^6)$. For each $(s,a,h)$, the reward $r_h(s,a)$ is sampled independently from the uniform distribution on $[0,1]$, and the transition kernel $\mathbb{P}_h(\cdot \mid s,a)$ is drawn uniformly from the probability simplex over S states. The initial state of each episode is sampled uniformly from the state space. We set $\iota = 1$ and use the bonus coefficients specified in Appendix~B.1 of \citet{zhang2025regret} for Q-EarlySettled-LowCost and UCB-Advantage. For UCB-Bernstein+, we still set $C_1=C_3=\sqrt{2}$ and $C_2=1$ in \Cref{betabern}.

\begin{figure}[H]
    \centering
    \begin{subfigure}[b]{0.49\textwidth}
        \centering
        \includegraphics[width=\linewidth]{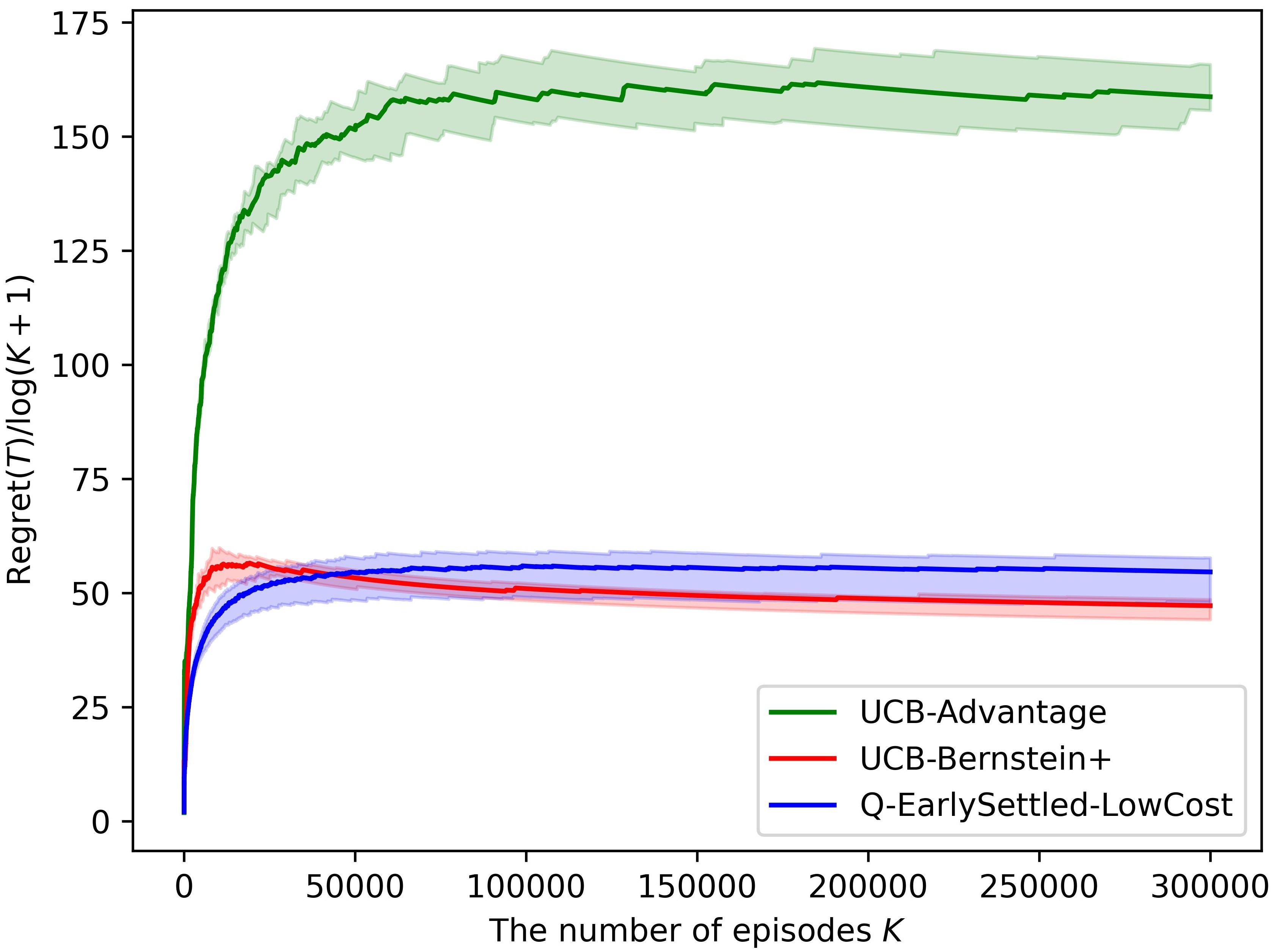}
        \caption{Regret for $(H,S,A)=(5,3,2)$}
    \end{subfigure}
    \hfill
    \begin{subfigure}[b]{0.49\textwidth}
        \centering
        \includegraphics[width=\linewidth]{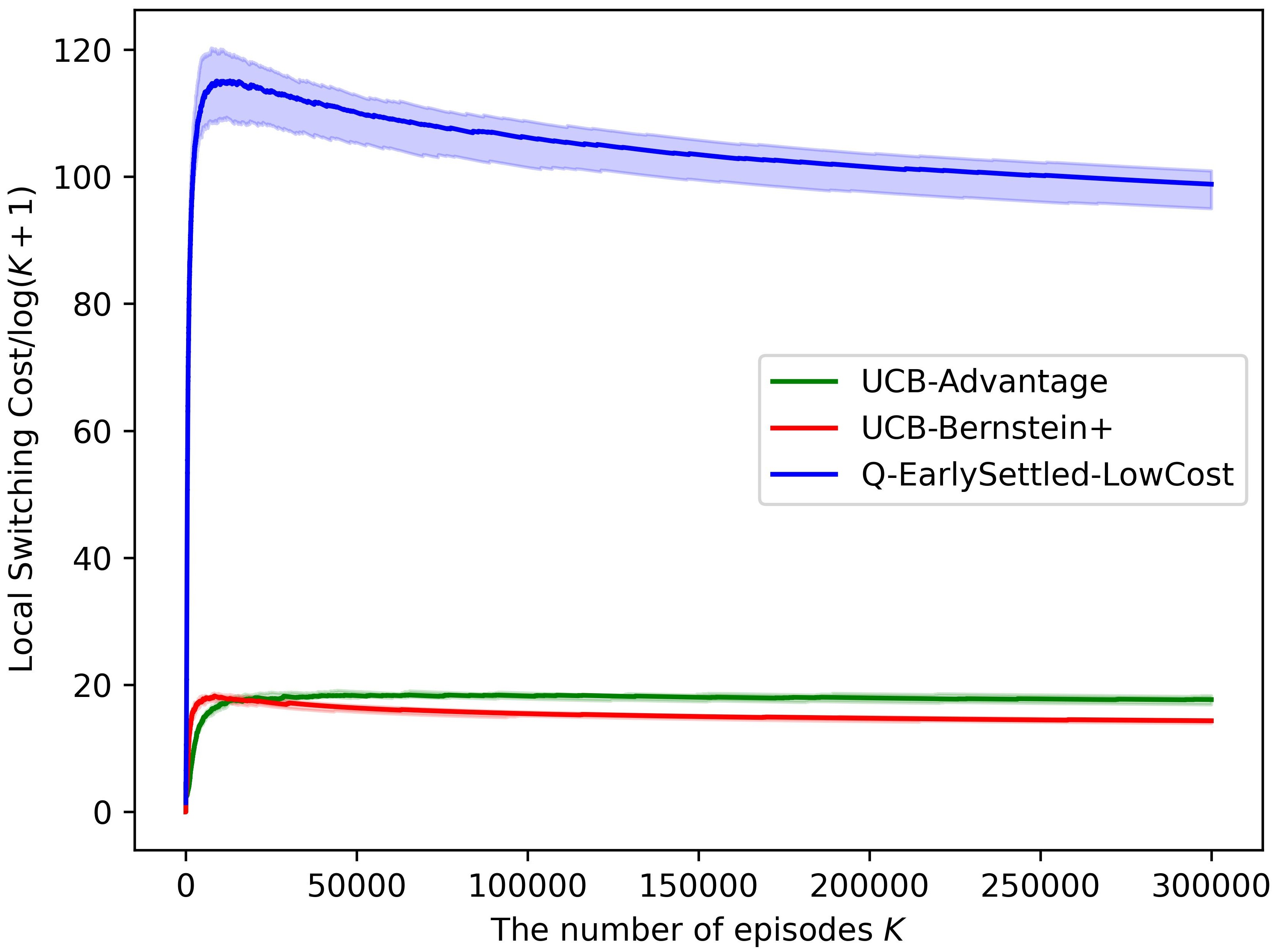}
        \caption{Local switching cost for $(H,S,A)=(5,3,2)$}
    \end{subfigure}

    \begin{subfigure}[b]{0.49\textwidth}
        \centering
        \includegraphics[width=\linewidth]{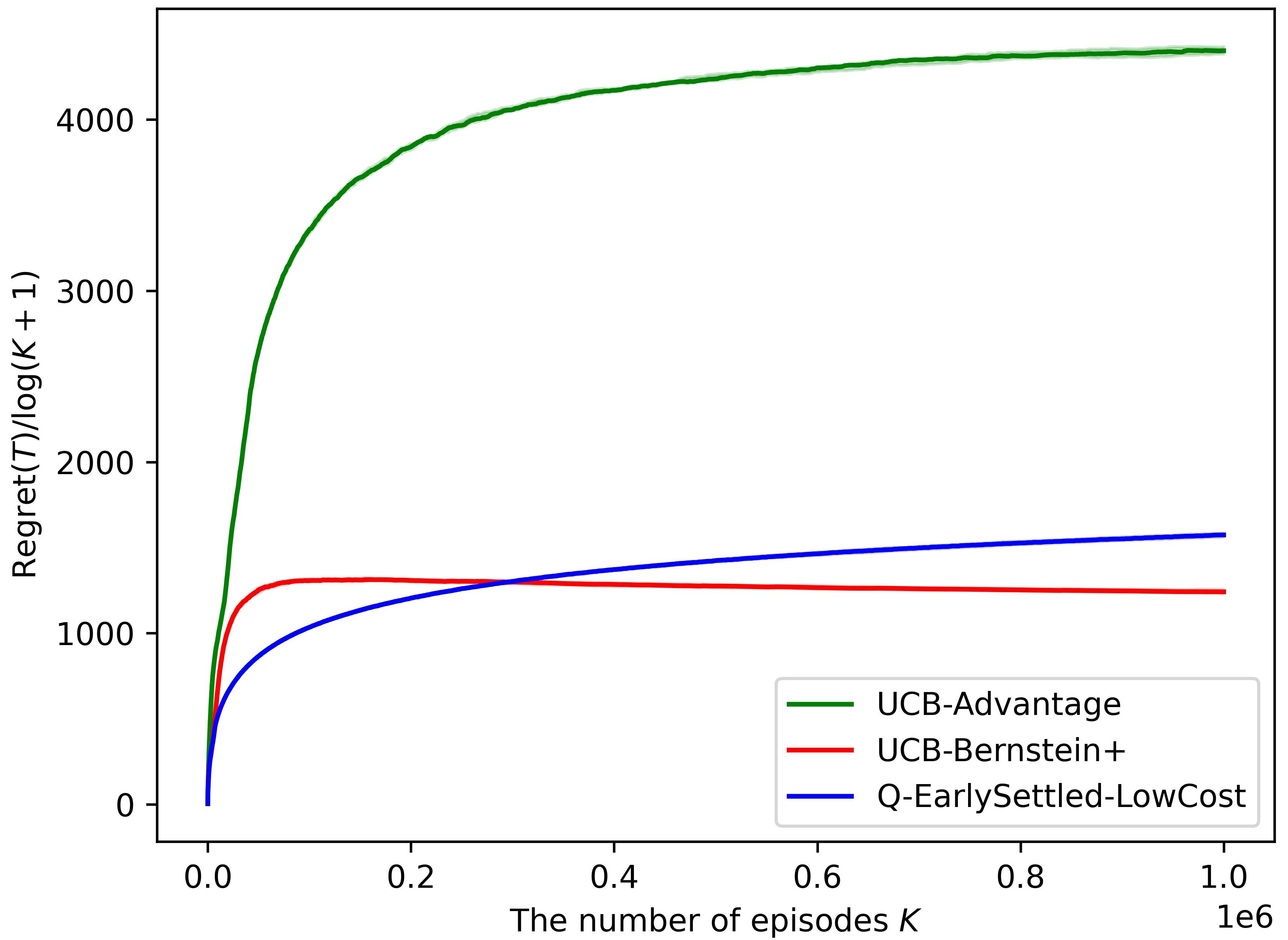}
        \caption{Regret for $(H,S,A)=(7,10,5)$}
    \end{subfigure}
    \hfill
    \begin{subfigure}[b]{0.49\textwidth}
        \centering
        \includegraphics[width=\linewidth]{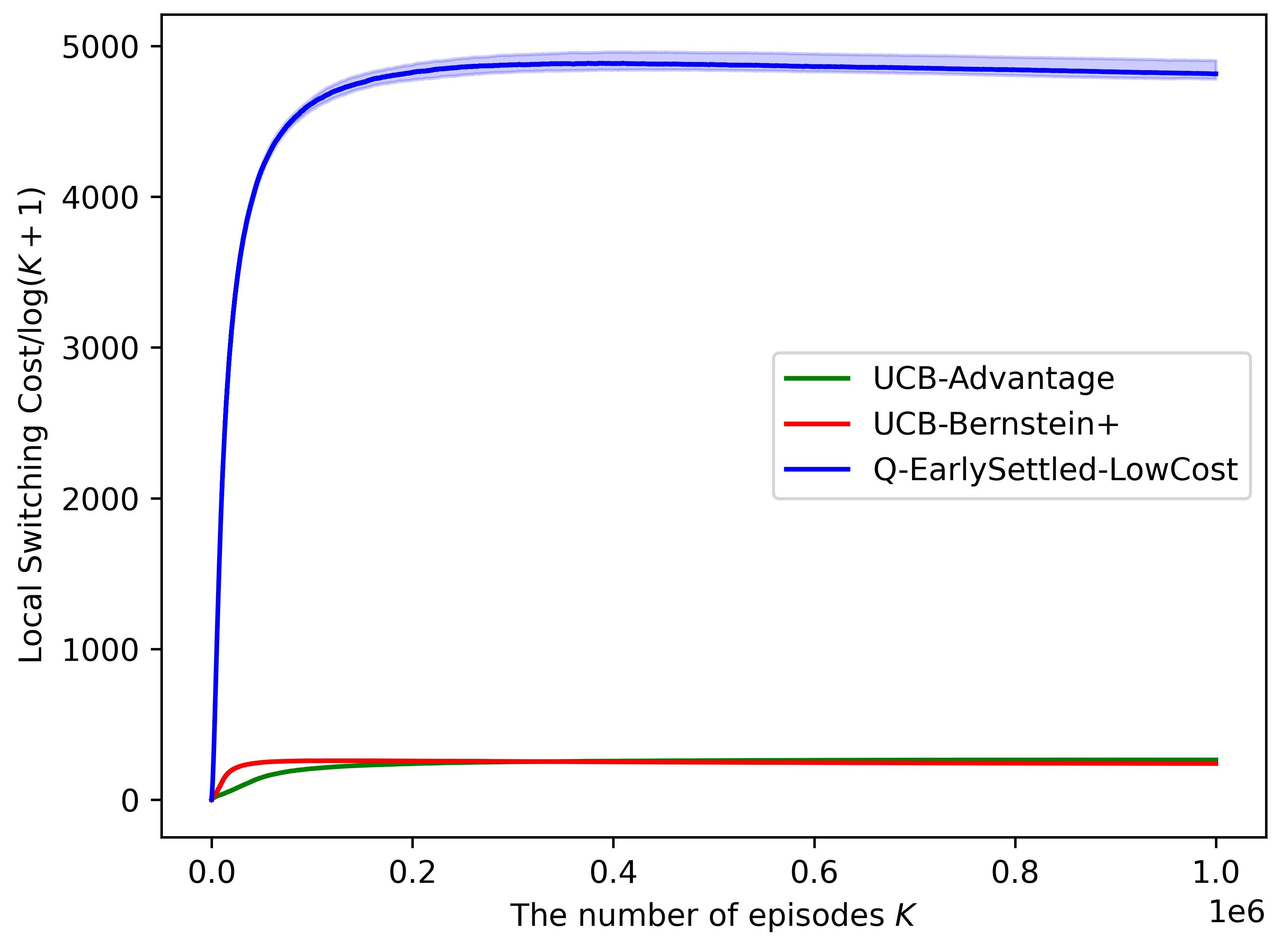}
        \caption{Local switching cost for $(H,S,A)=(7,10,5)$}
    \end{subfigure}
    \caption{Experimental comparison with state-of-the-art baselines.}
    \label{fig:baseline_comparison}
\end{figure}
To quantify uncertainty, for each experiment scale, we generate one MDP and run 10 independent trials for each algorithm on this fixed MDP. In the following figures, solid lines represent the median over the 10 trials, and shaded regions indicate the 10th--90th percentile intervals.

The results for the two experiment scales are presented in \Cref{fig:baseline_comparison}. We observe that UCB-Bernstein+ achieves better long-term regret performance than Q-EarlySettled-LowCost while incurring a significantly lower local switching cost. Compared with UCB-Advantage, UCB-Bernstein+ attains substantially better long-term regret performance while maintaining a comparable local switching cost. Overall, these results suggest that UCB-Bernstein+ offers a more favorable empirical trade-off between regret and local switching cost among the compared algorithms.

\section{Notation Tables}
\label{notation}
For readability, we summarize below the notation used throughout the algorithm and proofs. 
\begin{table}[H]
\centering
\caption{Basic MDP, value-function, and gap notation.}
\label{tab:notation-mdp}
\small
\renewcommand{\arraystretch}{1.1}
\setlength{\tabcolsep}{5pt}
\begin{tabular}{p{0.1\linewidth} p{0.85\linewidth}}
\toprule
\textbf{Notation} & \textbf{Meaning} \\
\midrule

$\mathcal S,\mathcal A$ &
State and action spaces, with $|\mathcal S|=S$ and
$|\mathcal A|=A$. \\

$H,K,T$ &
Horizon length, number of episodes, and total number of steps, respectively, where $T=KH$. \\

$P_h(\cdot|s,a)$ &
Transition distribution at step $h$ from state-action pair $(s,a)$. \\

$\mathbb P_{s,a,h}f$ &
Conditional expectation $\mathbb P_{s,a,h}f:=\mathbb E_{s'\sim P_h(\cdot\mid s,a)}[f(s')]$.
 \\

$\mathbb V_{s,a,h}(f)$ &
Conditional variance $\mathbb V_{s,a,h}(f):= \mathbb{P}_{s,a,h}f^2-(\mathbb{P}_{s,a,h}f)^2.$
 \\

$r_h(s,a)$ &
Reward associated with the state-action-step triple $(s,a,h)$. \\

$s_h^k,a_h^k$ &
State and action encountered at step $h$ of episode $k$. \\

$\pi^k$ &
Policy executed in episode $k$. \\

$V_h^\pi,Q_h^\pi$ &
State-value and action-value functions under policy $\pi$. \\

$V_h^\star,Q_h^\star$ &
Optimal state-value and action-value functions. \\

$\Delta_h(s,a)$ &
Suboptimality gap $\Delta_h(s,a):=V_h^\star(s)-Q_h^\star(s,a)$.
\\

$\Delta_{\min,h}$ &
Minimum positive suboptimality gap among state-action pairs at step
$h$. \\

$\Delta_{\min}$ &
Minimum positive suboptimality gap among all state-action-step
triples. \\

$ Z_{\text{opt},h}$ &
Set of optimal state-action pairs at step $h$: $Z_{\text{opt},h}
    :=
    \{(s,a)\in\mathcal S\times\mathcal A:
    \Delta_h(s,a)=0\}.$\\

$ Z_{\text{opt}}$ &
Set of optimal state-action-step triples: $ Z_{\text{opt}}:=\{(s,a,h):\Delta_h(s,a)=0\}.$\\

$ Z_{\text{opt}}^c$ &
Set of suboptimal state-action-step triples. \\

$\mathbb Q_h^\star$ &
Maximum conditional variance of the optimal next-step value:
$\mathbb Q_h^\star:=\max_{s,a}\mathbb V_{s,a,h}(V_{h+1}^\star).$
 \\

\bottomrule
\end{tabular}
\end{table}
\begin{table}[H]
\centering
\caption{Algorithmic quantities and stage-related notation.}
\label{tab:notation-algorithm}
\small
\renewcommand{\arraystretch}{1.1}
\setlength{\tabcolsep}{2pt}
\begin{tabular}{p{0.152\linewidth} p{0.827\linewidth}}
\toprule
\textbf{Notation} & \textbf{Meaning} \\
\midrule

$Q_h^k(s,a),V_h^k(s)$ &
Values of the algorithmic estimates $Q_h(s,a)$ and $V_h(s)$ at the
beginning of episode $k$. \\

$\pi_h^k(s)$ &
Greedy action induced by $Q_h^k$: $\pi_h^k(s)\in\arg\max_{a\in\mathcal A}Q_h^k(s,a).$
 \\

$\mu_h(s,a)$ &
Stage-wise accumulator of the next-step value estimates.\\

$\sigma_h(s,a)$ &
Stage-wise accumulator of the squared next-step value estimates.\\

$t_h^k(s,a)$ &
Index of the stage associated with triple $(s,a,h)$ at episode $k$.\\

$N_h^k(s,a)$ &
Total number of visits to $(s,a,h)$ in the first $t_h^k(s,a)-1$ completed stages for $k \in [K]$. \\

$N_h^{K+1}(s,a)$ & Total number of visits to $(s,a,h)$ over the first $K$ episodes.\\

$n_h^k(s,a)$ &
Number of visits to $(s,a,h)$ in the stage $t_h^k(s,a) -1$. \\

$W_h^k(s,a)$ &
Stage-wise empirical variance used to construct $Q_h^k(s,a)$. \\

$b_h^k(s,a)$ &
Exploration bonus associated with the update producing
$Q_h^k(s,a)$. \\

$\iota$ & 
Logarithmic confidence factor, typically $\iota=\log(2SAT/\delta)$. \\

$e_j$ &
Length of stage $j$, defined recursively by $e_1=H, e_{j+1}=\left\lfloor\left(1+\frac{1}{H}\right)e_j\right\rfloor.$ \\

$\mathcal L$ &
Set of stage endpoints:
$\mathcal L:=\left\{\sum_{j=1}^{m}e_j:m\in\mathbb N_+\right\}. $ \\

$e_{t_h^k(s,a)}$ &
Length of the stage indexed by $t_h^k(s,a)$. \\

$l_i(s,a,h,k)$ &
The index of the $i$-th visit to the state-action-step triple
$(s,a,h)$ within stage $t_h^k(s,a)-1$. \\

$k_{l_i(s,a,h,k)}$ &
Episode in which the sample indexed by $l_i(s,a,h,k)$ was collected.
\\

\bottomrule
\end{tabular}
\end{table}

\begin{table}[H]
\centering
\caption{Pseudo value functions and empirical operators used in the
variance analysis.}
\label{tab:notation-pseudo}
\small
\renewcommand{\arraystretch}{1.18}
\setlength{\tabcolsep}{4pt}
\begin{tabular}{p{0.12\linewidth} p{0.83\linewidth}}
\toprule
\textbf{Notation} & \textbf{Meaning} \\
\midrule

$\beta$ &
An analysis parameter in $(0,H]$ used to truncate the
value-estimation error. It is not an algorithmic tuning parameter. \\

$\hat V_h^k(s)$ &
Pseudo value function $\hat V_h^k(s):=\max\!\left\{V_h^\star(s),\min\!\left\{V_h^\star(s)+\beta,V_h^k(s)\right\}\right\}.$\\

$\hat{\mathbb E}_{h,s,a,k}[f]$ &
Empirical average of a function $f$ over the
$n_h^k(s,a)$ samples in the relevant stage:
\[
    \hat{\mathbb E}_{h,s,a,k}[f]
    :=
    \frac{1}{n_h^k(s,a)}
    \sum_{i=1}^{n_h^k(s,a)}
    f^{k_{l_i(s,a,h,k)}}
    \!\left(
        s_{h+1}^{k_{l_i(s,a,h,k)}}
    \right).
\] \\

$\hat{\mathbb V}_{h,s,a,k}[f]$ &
Corresponding stage-wise empirical variance:
$\hat{\mathbb V}_{h,s,a,k}[f]:=\hat{\mathbb E}_{h,s,a,k}[f^2]-\left(\hat{\mathbb E}_{h,s,a,k}[f]\right)^2.$\\

$\hat{\mathbb E}[f]$,
$\hat{\mathbb V}[f]$ &
Shorthand for
$\hat{\mathbb E}_{h,s,a,k}[f]$ and
$\hat{\mathbb V}_{h,s,a,k}[f]$ when $(h,s,a,k)$ is clear from
the context. \\

\bottomrule
\end{tabular}
\end{table}

\begin{table}[H]
\centering
\caption{Recursive weights for propagating cumulative weighted estimation errors across steps.}
\label{tab:notation-weights}
\small
\renewcommand{\arraystretch}{1.22}
\setlength{\tabcolsep}{4pt}
\begin{tabular}{p{0.17\linewidth} p{0.78\linewidth}}
\toprule
\textbf{Notation} & \textbf{Meaning} \\
\midrule

$\{\omega_h^k\}_{k=1}^K$ &
A nonnegative initial weight sequence assigned to the estimation
errors at the starting step $h$. \\

$\omega_h(k',h)$ &
Initial propagated weight at step $h$, defined by $\omega_h(k',h):=\omega_h^{k'}.$ \\

$\omega_h(k',h'+1)$ &
Weight assigned to the estimation error at episode $k'$ and step
$h'+1$ after recursively propagating the weights from step $h'$. It captures the contribution of the episode-$k'$, step-$(h'+1)$
estimation error when recursively bounding the cumulative weighted
estimation error from step $h'$ to step $h'+1$. \\

$\omega_h(k',h',s,a)$ &
State-action-specific propagated weight: $\omega_h(k',h',s,a):=\omega_h(k',h')\mathbb I[
(s_{h'}^{k'},a_{h'}^{k'})=(s,a)].$
It isolates the contribution associated with state-action pair
$(s,a)$ at step $h'$. \\

$\tilde\omega_h(k',h'+1,s,a)$ &
The state-action-specific propagated weight assigned to the
episode-$k'$, step-$(h'+1)$ estimation error when bounding the
cumulative weighted estimation error associated with $(s,a)$ at step
$h'$: $\tilde\omega_h(k',h'+1,s,a):=\omega_h(k',h'+1)\mathbb I[(s_{h'}^{k'},a_{h'}^{k'})=(s,a)].$ \\

$\|\omega_h(\cdot,h')\|_1$ &
The $\ell_1$-norm of the propagated weights at step $h'$:
$\|\omega_h(\cdot,h')\|_1:=\sum_{k'=1}^K\omega_h(k',h').$ \\

$\|\omega_h(\cdot,h')\|_\infty$ &
The $\ell_\infty$-norm of the propagated weights at step $h'$:\newline
$\|\omega_h(\cdot,h')\|_\infty:=\max_{k'\in[K]}\omega_h(k',h').$ \\

$\|\omega_h(\cdot,h',s,a)\|_1$ &
The state-action-specific $\ell_1$-norm of the propagated weights at
step $h'$:\newline
$\|\omega_h(\cdot,h',s,a)\|_1:=\sum_{k'=1}^K\omega_h(k',h',s,a).$ \\

$\|\omega_h(\cdot,h',s,a)\|_\infty$ &
The state-action-specific $\ell_\infty$-norm of the propagated weights
at step $h'$:\newline
$\|\omega_h(\cdot,h',s,a)\|_\infty:=\max_{k'\in[K]}\omega_h(k',h',s,a).$ \\

\bottomrule
\end{tabular}
\end{table}
\section{Auxiliary Lemmas}
\label{keylemma}
\begin{lemma}
\label{Azuma} \textnormal{(Azuma-Hoeffding inequality)} Suppose {$\left\{X_k\right\}_{k=0}^\infty$} is a martingale and $|X_k-X_{k-1}|\leq c_k$, $\forall k\in\mathbb{N}_+$ almost surely. Then for any positive integer $N$ and any $\epsilon > 0$, it holds that:
$$\mathbb{P}\left(|X_N-X_{0}|\geq \epsilon\right) \leq 2\exp \left(-\frac{\epsilon^2}{2\sum_{k=1}^{N}c_k^2}\right).$$
\end{lemma}

\begin{lemma}
\label{Freedman}
\textnormal{(Freedman's inequality, Theorem EC.1 in \citet{li2024q}.)} Consider a filtration \( \mathcal{F}_0 \subset \mathcal{F}_1 \subset \mathcal{F}_2 \subset \cdots \), and let \( \mathbb{E}_k \) stand for the expectation conditioned on \( \mathcal{F}_k \). Suppose that $Y_n = \sum_{k=1}^n X_k \in \mathbb{R}$, where \( \{X_k\} \) is a real-valued scalar sequence adapted to \( \{\mathcal{F}_k\} \) and obeying
\[
|X_k| \leq R \quad \text{and} \quad \mathbb{E}_{k-1}[X_k] = 0 \quad \text{for all } k \geq 1
\]
for some quantity \( R < \infty \). We also define
$W_n := \sum_{k=1}^n \mathbb{E}_{k-1}[X_k^2].$ In addition, suppose that \( W_n \leq \sigma^2 \) holds deterministically for some given quantity \( \sigma^2 < \infty \). Then for any positive integer \( m \geq 1 \), with probability at least \( 1 - \delta \), one has
\[
|Y_n| \leq \sqrt{8 \max \left\{ W_n, \frac{\sigma^2}{2^m} \right\} \log \frac{2m}{\delta} } + \frac{4}{3} R \log \frac{2m}{\delta}.
\]
\end{lemma}

\section{Basic Properties of UCB-Bernstein+}
\label{property}
Let $Q_h^k(s,a)$ and $V_h^k(s)$ denote the values of $Q_h(s,a)$ and $V_h(s)$ at the beginning of the $k$-th episode, respectively. We note that the definition of stages is with respect to each triple $(s,a,h)$. For every $(s,a,h,k)$, define $t_h^k(s,a)$ as the unique stage index $j\geq 1$ such that the number of visits to $(s,a,h)$ before episode $k$ lies in the interval $[\sum_{i=1}^{j-1}e_i,\sum_{i=1}^{j}e_i)$. Thus, if $(s_h^k,a_h^k)=(s,a)$, the sample collected at step $h$ of episode $k$ belongs to stage $t_h^k(s,a)$. When there is no ambiguity, we write $t_h^k$ for short.

Let $N_h^k(s,a)$ denote the total number of visits to $(s,a,h)$ in the first $t_h^k(s,a)-1$ completed stages, and let $n_h^k(s,a)$ denote the total number of visits to $(s,a,h)$ in stage $t_h^k(s,a)-1$. Separately, for notational convenience, $N_{h}^{K+1}(s,a)$ denotes the total number of visits to $(s,a,h)$ over the first $K$ episodes. We also let $l_i(s,a,h,k)$, $i \in [n_h^k(s,a)]$, denote the $i$-th visit to $(s,a,h)$ during stage $t_h^k(s,a)-1$, and let $k_{l_i(s,a,h,k)}$ be the corresponding episode index of the visit $l_i(s,a,h,k)$. When $(s,a,h,k)$ is clear from the context, we use $l_i$ and $k_{l_i}$ for simplicity.

Throughout the appendices, whenever a summand contains the indicator $\mathbb{I}[t_h^k(s,a)>1]$ together with quantities associated with the preceding stage, such as $n_h^k(s,a)$ or $k_{l_i}$, the summand is understood to be evaluated only when $t_h^k(s,a)>1$ and to be zero otherwise. Equivalently, such sums may be restricted to indices satisfying $t_h^k(s,a)>1$.

For any given step $h$ and non-negative weight sequence $\{\omega_h^k\}_{k=1}^K$, we define the following weights for any $k' \in [K],\ h \leq h' < H$:
\begin{align}
\label{recurweight}
&\omega_h(k',h) := \omega_h^{k'};\notag \\
&\omega_{h}(k',h'+1) :=  \sum_{k=1}^K\omega_h(k,h')\frac{\sum_{i=1}^{n_{h'}^k(s_{h'}^{k},a_{h'}^{k})}\mathbb{I}\left[k_{l_i(s_{h'}^{k},a_{h'}^{k},h',k)}=k',t_{h'}^{k}(s_{h'}^{k},a_{h'}^{k})>1\right]}{n_{h'}^k(s_{h'}^{k},a_{h'}^{k})}
\end{align}
with the norms
\begin{equation}
\label{weightnorm}
\|\omega_{h}(\cdot,h')\|_{\infty} := \max_{k' \in [K]} \omega_{h}(k',h'),\ \|\omega_{h}(\cdot,h')\|_{1} := \sum_{k'=1}^K \omega_{h}(k',h').\end{equation}
The weight $\omega_h(k',h'+1)$ captures the contribution of the term $(Q_{h'+1}^{k'} - Q_{h'+1}^\star)(s_{h'+1}^{k'}, a_{h'+1}^{k'})$ when recursively bounding the cumulative weighted estimation error from step $h'$ to $h'+1$ as shown in the second conclusion of Lemma \ref{recurQb} later.

For each state-action pair $(s, a)$, we define the state-action specific weight at any step $h \leq h'\leq H$ as
$\omega_{h}(k',h', s, a) := \omega_{h}(k',h') \cdot \mathbb{I}[(s_{h'}^{k'}, a_{h'}^{k'}) = (s, a)]$
with the corresponding norms given by 
$$
    \|\omega_h(\cdot,h', s, a)\|_{\infty} := \max_{{k'} \in [K]} \omega_{h}({k'},h', s, a),\ 
\|\omega_h(\cdot,h', s, a)\|_{1} := \sum\nolimits_{{k'}=1}^K \omega_{h}({k'},h', s, a).$$
Additionally, for any state-action pair $(s, a)$, we also define $$\tilde{\omega}_{h}(k',h'+1, s, a) = \omega_{h}(k',h'+1)\mathbb{I}[(s_{h'}^{k'}, a_{h'}^{k'}) = (s, a)].$$ The weight $\tilde{\omega}_h(k',h'+1,s,a)$ characterizes the weight of the term $(Q_{h'+1}^{k'} - Q_{h'+1}^\star)(s_{h'+1}^{k'}, a_{h'+1}^{k'})$ when bounding the cumulative weighted estimation error for each state-action pair $(s, a)$ at step $h'$ as shown in the first conclusion of Lemma \ref{recurQb} later. For convenience, we set $\omega_h(k,H+1)=0$ and $\widetilde{\omega}_h(k,H+1,s,a)=0$ for all $k\in[K]$ and $(s,a)\in\mathcal{S}\times\mathcal{A}$.
\begin{lemma}
\label{weight}
For any given non-negative weight sequence $\{\omega_h^k\}_{k \in [K]}$ at step $h$, the following relationships hold for any $k' \in [K]$ and $h \leq h' <H$:
\begin{itemize}
    \item[(a)] $\sum_{s,a} \tilde{\omega}_{h}(k',h'+1, s, a) = \omega_{h}(k',h'+1)$.
    \item[(b)] $\|\omega_h(\cdot,h', s, a)\|_{\infty} \leq \|\omega_h(\cdot,h')\|_{\infty}$.
    \item[(c)]$\|\omega_h(\cdot,h',s,a)\|_{1}  \leq \|\omega_h(\cdot,h')\|_{\infty}N_{h'}^{K+1}(s,a)$.
    \item[(d)] $\|\omega_h(\cdot,h')\|_{1} = \sum_{s,a} \|\omega_h(\cdot,h',s,a)\|_{1}$.
    \item[(e)]
$\|\omega_h(\cdot,h'+1)\|_{\infty}\leq \left(1+\frac{1}{H}\right)\|\omega_h(\cdot,h')\|_{\infty},\ \|\omega_h(\cdot,h'+1)\|_{1} \leq \|\omega_h(\cdot,h')\|_{1}.$
\end{itemize}
\end{lemma}
\begin{proof}[Proof of Lemma \ref{weight}]
(a) is because $$\sum_{s,a} \tilde{\omega}_{h}(k',h'+1, s, a) = \sum_{s,a} \omega_{h}(k',h'+1)\mathbb{I}[(s_{h'}^{k'}, a_{h'}^{k'}) = (s, a)] = \omega_{h}(k',h'+1).$$
(b) is because for any $k' \in [K]$
$$\omega_{h}(k',h', s, a) = \omega_{h}(k',h') \cdot \mathbb{I}\left[(s_{h'}^{k'}, a_{h'}^{k'}) = (s, a)\right]\leq \omega_{h}(k',h') \leq \|\omega_h(\cdot,h')\|_{\infty}.$$
(c) is because 
$$\|\omega_h(\cdot,h',s,a)\|_{1} \leq \|\omega_h(\cdot,h')\|_{\infty}\sum_{k'=1}^K  \mathbb{I}\left[(s_{h'}^{k'}, a_{h'}^{k'}) = (s, a)\right]  = \|\omega_h(\cdot,h')\|_{\infty}N_{h'}^{K+1}(s,a).$$
(d) is because
$$\sum_{s,a} \|\omega_h(\cdot,h',s,a)\|_{1} = \sum_{s,a} \sum_{k'=1}^K \omega_{h}(k',h') \cdot \mathbb{I}[(s_{h'}^{k'}, a_{h'}^{k'}) = (s, a)] = \sum_{k'=1}^K \omega_{h}(k',h') = \|\omega_h(\cdot,h')\|_{1}.$$
For (e), note that
$\sum_{i=1}^{n_{h'}^k}\mathbb{I}\left[k_{l_i}=k^{\prime},t_{h'}^{k}>1\right] = 1$
if and only if 
$(s_{h'}^{k},a_{h'}^{k}) = (s_{h'}^{k^{\prime}},a_{h'}^{k^{\prime}})$ and  $1<t_{h'}^k(s_{h'}^{k},a_{h'}^{k}) = t_{h'}^{k^{\prime}}(s_{h'}^{k},a_{h'}^{k})+1.$ In this case, we have $t_{h'}^{k^{\prime}}(s_{h'}^{k},a_{h'}^{k}) >0 \textnormal{ and } n_{h'}^k(s_{h'}^{k},a_{h'}^{k}) = e_{t_{h'}^{k^{\prime}}}.$ Then for any given $k^{\prime}$, it holds that:
\begin{align*}
&\sum_{k=1}^K\sum_{i=1}^{n_{h'}^k}\mathbb{I}\left[k_{l_i}=k^{\prime},t_{h'}^{k}>1\right]\\
&= \sum_{k=1}^K\mathbb{I}\left[(s_{h'}^{k},a_{h'}^{k}) = (s_{h'}^{k^{\prime}},a_{h'}^{k^{\prime}}), t_{h'}^k(s_{h'}^{k},a_{h'}^{k}) = t_{h'}^{k^{\prime}}(s_{h'}^{k},a_{h'}^{k})+1\right] \leq e_{t_{h'}^{k^{\prime}}+1}.
\end{align*} 
Therefore, we have:
\begin{align}
\label{inftygoal}
\sum_{k=1}^K\frac{\sum_{i=1}^{n_{h'}^k}\mathbb{I}[k_{l_i}=k^{\prime},t_{h'}^{k}>1]}{n_{h'}^k(s_{h'}^{k},a_{h'}^{k})} 
\leq\frac{e_{t_{h'}^{k^{\prime}}+1}}{e_{t_{h'}^{k^{\prime}}}}\leq 1+\frac{1}{H}.
\end{align}
Then according to the definition of $\|\omega_h(\cdot,h'+1)\|_{\infty}$, we have:
$$\|\omega_h(\cdot,h'+1)\|_{\infty} \leq \|\omega_h(\cdot,h')\|_{\infty}\max_{{k'} \in [K]}\sum_{k=1}^K\frac{\sum_{i=1}^{n_{h'}^k}\mathbb{I}[k_{l_i}=k^{\prime},t_{h'}^{k}>1]}{n_{h'}^k(s_{h'}^{k},a_{h'}^{k})} \\
\leq \left(1+\frac{1}{H}\right)\|\omega_h(\cdot,h')\|_{\infty}.$$
We also have:
\begin{align*}
   \|\omega_h(\cdot,h'+1)\|_1
   &= \sum_{k^{\prime}=1}^K\sum_{k=1}^K\omega_h(k,h')\sum_{i=1}^{n_{h'}^k}\frac{\mathbb{I}[k_{l_i}=k^{\prime},t_{h'}^{k}>1]}{n_{h'}^k(s_{h'}^{k},a_{h'}^{k})} \\&= \sum_{k=1}^K\omega_h(k,h')\sum_{i=1}^{n_{h'}^k}\sum_{k^{\prime}=1}^K\frac{\mathbb{I}[k_{l_i}=k^{\prime},t_{h'}^{k}>1]}{n_{h'}^k(s_{h'}^{k},a_{h'}^{k})} \\
   &\leq \|\omega_h(\cdot,h')\|_1.
\end{align*}
\end{proof}

\begin{lemma}
\label{wN} For any non-negative weight sequence $\left\{\omega_{h}^{k}\right\}_{k}$ at step $h \in [H]$, any subsequent step $h' \in [h,H]$, any state-action pair $(s,a) \in \mathcal{S} \times \mathcal{A}$, and any $\alpha \in (0,1)$, it holds that:
\begin{equation*}
    \sum_{k=1}^K\frac{\omega_{h}(k,h',s,a)\mathbb{I}\left[t_{h'}^k(s_{h'}^k,a_{h'}^k) > 1\right]}{n_{h'}^k(s_{h'}^k,a_{h'}^k)^{\alpha}} \leq \frac{(4H)^\alpha}{1-\alpha}\|\omega_h(\cdot,{h'})\|_{\infty}^{\alpha}\|\omega_h(\cdot,{h'},s,a)\|_{1}^{1-\alpha},
\end{equation*}
and
\begin{equation*}
    \sum_{k=1}^K\frac{\omega_{h}(k,h')\mathbb{I}\left[t_{h'}^k(s_{h'}^k,a_{h'}^k) > 1\right]}{n_{h'}^k(s_{h'}^k,a_{h'}^k)^{\alpha}}\leq \frac{(4H)^\alpha}{1-\alpha}(SA\|\omega_h(\cdot,h')\|_{\infty})^{\alpha}\|\omega_h(\cdot,h')\|_{1}^{1-\alpha},
\end{equation*}
In particular, for $\alpha = 1$ and any $h \in [H]$, we have the following conclusion:
\begin{equation*}
    \sum_{k=1}^K\frac{\mathbb{I}\left[t_{h}^k(s_{h}^k,a_{h}^k) > 1\right]}{n_h^k(s_h^{k},a_h^{k})} \leq 10HSA\iota.
\end{equation*}
\end{lemma}

\begin{proof}[Proof of Lemma \ref{wN}]
We first note that
\begin{align}
\label{wNproof}
        &\sum_{k=1}^K\frac{\omega_{h}(k,h',s,a)\mathbb{I}\left[t_{h'}^k(s_{h'}^k,a_{h'}^k) > 1\right]}{n_{h'}^k(s_{h'}^k,a_{h'}^k)^{\alpha}} = \sum_{k=1}^K\frac{\omega_{h}(k,h')\mathbb{I}[t_{h'}^k(s,a) > 1, (s_{h'}^k,a_{h'}^k) = (s,a)]}{n_{h'}^k(s,a)^{\alpha}}.
\end{align} 
 Then we have 
$$\Omega_h(s,a):=\sum_{k=1}^{K}\omega_{h}(k,h')\mathbb{I}\left[t_{h'}^k(s,a) > 1\right]\mathbb{I}[(s_{h'}^k,a_{h'}^k) = (s,a)] \leq \|\omega_h(\cdot,h',s,a)\|_{1}.$$
If $\Omega_h(s,a)=0$, the left-hand side of \Cref{wNproof} is zero, and the desired bounds hold trivially. Hence, in the remainder of the proof, we assume $\Omega_h(s,a)>0$.

Given the term on RHS of \Cref{wNproof}, when the weights $\omega_{h}(k,h')\mathbb{I}\left[t_{h'}^k(s,a) > 1\right]\mathbb{I}[(s_{h'}^k,a_{h'}^k) = (s,a)]$ concentrate on the earlier terms with smaller index $k \geq 1$, we can obtain the largest value. Let 
\begin{equation*}
    c_{s,a,h'} = \max\left\{c:\|\omega_h(\cdot,{h'})\|_{\infty}\sum_{i=2}^{c}e_i \leq \Omega_h(s,a)\right\},\  d_{s,a,h'} = \Omega_h(s,a) - \|\omega_h(\cdot,{h'})\|_{\infty}\sum_{i=2}^{c_{s,a,h'}}e_i.
\end{equation*}
Then we have:
\begin{align}
    &\sum_{k=1}^K\frac{\omega_{h}(k,h',s,a)\mathbb{I}\left[t_{h'}^k(s_{h'}^k,a_{h'}^k) > 1\right]}{n_{h'}^k(s_{h'}^k,a_{h'}^k)^{\alpha}}
    \notag\\
    &\leq \sum_{i=1}^{c_{s,a,h'}-1}\frac{e_{i
    +1}\|\omega_h(\cdot,h')\|_{\infty}}{e_i^{\alpha}} +\frac{d_{s,a,h'}}{e_{c_{s,a,h'}}^{\alpha}} \nonumber\\
    &\leq (4H) ^\alpha\left(\sum_{i=1}^{c_{s,a,h'}-1}\frac{e_{i
    +1}\|\omega_h(\cdot,h')\|_{\infty}}{\left(\sum_{j=2}^{i+1}e_j\right)^{\alpha}} +\frac{d_{s,a,h'}}{\left(\sum_{i=2}^{c_{s,a,h'}+1}e_i\right)^{\alpha}}\right) \label{split0}\\
    &\leq (4H) ^\alpha\left(\|\omega_h(\cdot,h')\|_{\infty}\sum_{i=1}^{c_{s,a,h'}-1}\frac{(\sum_{j=2}^{i+1}e_j)^{1-\alpha}-(\sum_{j=2}^{i}e_j)^{1-\alpha}}{1-\alpha} +\frac{d_{s,a,h'}}{\left(\sum_{i=2}^{c_{s,a,h'}+1}e_i\right)^{\alpha}}\right) \label{split1}\\
    &= (4H) ^\alpha\left(\frac{\|\omega_h(\cdot,h')\|_{\infty}\left(\sum_{i=2}^{c_{s,a,h'}}e_i\right)^{1-\alpha}}{1-\alpha}+\frac{d_{s,a,h'}}{\left(\sum_{i=2}^{c_{s,a,h'}+1}e_i\right)^{\alpha}}\right) \nonumber\\
    &= (4H) ^\alpha\|\omega_h(\cdot,h')\|_{\infty}^{\alpha}\left(\frac{\left(\|\omega_h(\cdot,h')\|_{\infty}\sum_{i=2}^{c_{s,a,h'}}e_i\right)^{1-\alpha}}{1-\alpha} + \frac{d_{s,a,h'}}{\left(\|\omega_h(\cdot,h')\|_{\infty}\sum_{i=2}^{c_{s,a,h'}+1}e_i\right)^{\alpha}}\right)\nonumber\\
    &\leq (4H) ^\alpha\|\omega_h(\cdot,h')\|_{\infty}^{\alpha}\left(\frac{\left(\|\omega_h(\cdot,h')\|_{\infty}\sum_{i=2}^{c_{s,a,h'}}e_i\right)^{1-\alpha}}{1-\alpha} + \frac{d_{s,a,h'}}{\Omega_h(s,a)^{\alpha}}\right). \label{Nsplit}
\end{align}
Here the inequality in \Cref{split0} is because for any $i \in \mathbb{N}_+$, $\sum_{j=2}^{i+1} e_j \leq 4H e_i$.  The inequality in \Cref{split1} is because for any $0< y <x$ and $\alpha \in (0,1)$, we have:
\begin{equation}
\label{diffxy}
    \frac{x-y}{x^{\alpha}} \leq \frac{1}{1-\alpha}(x^{1-\alpha}-y^{1-\alpha}).
\end{equation}
Then, let $x = \sum_{j=2}^{i+1}e_j$ and $y = \sum_{j=2}^{i}e_j$, it holds that:
\begin{align*}
    \frac{e_{i+1}}{\left(\sum_{j=2}^{i+1}e_j\right)^{\alpha}} \leq \frac{1}{1-\alpha}\left(\left(\sum_{j=2}^{i+1}e_j\right)^{1-\alpha}-\left(\sum_{j=2}^{i}e_j\right)^{1-\alpha}\right).
\end{align*}
The last inequality in \Cref{Nsplit} is because $\|\omega_h(\cdot,h')\|_{\infty}\sum_{i=2}^{c_{s,a,h'}+1}e_i\geq \Omega_h(s,a)$.

For \Cref{Nsplit}, let $x = \Omega_h(s,a)$ and $y = \|\omega_h(\cdot,{h'})\|_{\infty}\sum_{i=2}^{c_{s,a,h'}}e_i$ in \Cref{diffxy}, we have:
$$ \frac{d_{s,a,h'}}{\Omega_h(s,a)^{\alpha}} + \frac{\left(\|\omega_h(\cdot,{h'})\|_{\infty}\sum_{i=2}^{c_{s,a,h'}}e_i\right)^{1-\alpha}}{1-\alpha} \leq \frac{\Omega_h(s,a)^{1-\alpha}}{1-\alpha}\leq\frac{\|\omega_h(\cdot,h',s,a)\|_{1}^{1-\alpha}}{1-\alpha}.$$
Applying this inequality to \Cref{Nsplit}, we have:
\begin{align*}
    \sum_{k=1}^K\frac{\omega_{h}(k,h',s,a)\mathbb{I}\left[t_{h'}^k(s_{h'}^k,a_{h'}^k) > 1\right]}{n_{h'}^k(s_{h'}^k,a_{h'}^k)^{\alpha}} \leq \frac{(4H)^\alpha}{1-\alpha}\|\omega_h(\cdot,h')\|_{\infty}^{\alpha}\|\omega_h(\cdot,h',s,a)\|_{1}^{1-\alpha}. 
\end{align*}
Therefore, we have proved the first conclusion. By summing this conclusion for all state-action pairs $(s,a)$, we reach:
\begin{align*}
     \sum_{k=1}^K\frac{\omega_{h}(k,h')\mathbb{I}\left[t_{h'}^k(s_{h'}^k,a_{h'}^k) > 1\right]}{n_{h'}^k(s_{h'}^k,a_{h'}^k)^{\alpha}} &\leq \sum_{s,a} \frac{(4H)^\alpha}{1-\alpha}\|\omega_h(\cdot,h')\|_{\infty}^{\alpha}\|\omega_h(\cdot,h',s,a)\|_{1}^{1-\alpha} \\
    &\leq \frac{(4H)^\alpha}{1-\alpha}(SA\|\omega_h(\cdot,h')\|_{\infty})^{\alpha}\|\omega_h(\cdot,h')\|_{1}^{1-\alpha}.
\end{align*}
The last inequality is by Hölder's inequality, as $$\sum_{s,a}\|\omega_h(\cdot,h',s,a)\|_{1}^{1-\alpha} \leq (SA)^{\alpha}\|\omega_h(\cdot,h')\|_{1}^{1-\alpha}.$$
For $\alpha = 1$, it holds that:
\begin{align}
\label{special1}
    \sum_{k=1}^K\frac{\mathbb{I}\left[t_{h}^k(s_{h}^k,a_{h}^k) > 1\right]}{n_h^k(s_h^{k},a_h^{k})}&= \sum_{s,a}\sum_{k=1}^{K}\frac{\mathbb{I}\left[t_{h}^k(s,a) > 1, (s_{h}^k,a_{h}^k) = (s,a)\right]}{n_h^k(s,a)}\notag\\
    &\leq \sum_{s,a}\sum_{i=1}^{t_h^K-1}\frac{e_{i+1}}{e_i} \leq 2\sum_{s,a}(t_h^K(s,a)-1).
\end{align}
The last inequality is because $e_{i+1} \leq 2e_i$ by the definition of the sequence $\{e_i\}_{i \in \mathbb{N}_+}$. Since at each stage of any given state-action pair $(s,a)$, the number of visits increases by at least $(1+1/2H)$ times, for any $t_h^K(s,a) \geq 1$, the total number of visits to $(s,a)$ is at least
$$\left[H\left(1+\frac{1}{2H}\right)^{t_h^K(s,a)-1}-H\right]$$
Therefore, for any step $h \in [H]$, we have: 
$$\sum_{s,a}\left[H\left(1+\frac{1}{2H}\right)^{t_h^K(s,a)-1}-H\right] \leq K$$
By Jensen's inequality, we have
$$K +HSA \geq H\sum_{s,a}\left(1+\frac{1}{2H}\right)^{t_h^K(s,a)-1} \geq  HSA\left(1+\frac{1}{2H}\right)^{\sum_{s,a}t_h^K(s,a)/SA-1}$$
Solving the inequality yields
\begin{equation}
\label{sumnumstage}
    \sum_{s,a}t_h^K(s,a) \leq SA\frac{\log(K/(HSA) + 1)}{\log\left(1+\frac{1}{2H}\right)} + SA\leq 4HSA\log\left(\frac{K}{HSA} + 1\right) + SA\leq  5HSA\iota.
\end{equation}
The second inequality is because
$$\log\left(1+\frac{1}{2H}\right) \geq \frac{1}{4H}.$$
Plugging \Cref{sumnumstage} into \Cref{special1}, we finish the proof for the conclusion when $\alpha = 1$.
\end{proof}

\section{Probability Events}
\label{probabilityevent}
Let $\mathcal{X} = (\mathcal{S},\mathcal{A},H,T,\iota)$. The notation $f(\mathcal{X}) \lesssim g(\mathcal{X})$ means that there exists a universal constant $C>0$ such that  $f(\mathcal{X})\leq Cg(\mathcal{X})$. We first provide a lemma of concentration inequalities.
\begin{lemma}\label{concentrations}
Let $\iota = \log (2SAT/\delta)$ with $\delta\in (0,1)$. We use $\forall (s,a,h,k)$ as the simplified notation for $\forall (s,a,h,k)\in \sah\times [K]$. Whenever $n_h^k(s,a)$ appears in a denominator, the corresponding statement
is restricted to tuples satisfying $t_h^k(s,a)>1$, for which $n_h^k(s,a)>0$. For any function $f: \mathcal{S} \rightarrow \mathbb{R}$, we denote $\mathbb{V}_{s,a,h}(f) = \mathbb{P}_{s,a,h}f^2-(\mathbb{P}_{s,a,h}f)^2$. Then we have the following conclusions.

(a) With probability at least $1-\delta$, the following event holds:
$$\mathcal{E}_1 = \left\{\left|\frac{1}{n_h^k}\sum_{i=1}^{n_h^k} \left(\mathbbm{1}_{s_{h+1}^{k_{l_i}}}-\mathbb{P}_{s,a,h}\right) V_{h+1}^{k_{l_i}} \right| \leq \frac{5}{n_h^k}\sqrt{ \sum_{i
=1}^{n_h^k}\mathbb{V}_{s,a,h}(V_{h+1}^{k_{l_i}})\iota} + \frac{9H\iota}{n_h^k},\forall (s,a,h,k) \right\}.$$
(b) With probability at least $1-\delta$, the following event holds:
$$\mathcal{E}_2 = \left\{\left|\frac{1}{n_h^k}\sum_{i=1}^{n_h^k} (\mathbbm{1}_{s_{h+1}^{k_{l_i}}}-\mathbb{P}_{s,a,h}) (V_{h+1}^{k_{l_i}})^2 \right| \leq \frac{5}{n_h^k}\sqrt{\sum_{i
=1}^{n_h^k}\mathbb{V}_{s,a,h}((V_{h+1}^{k_{l_i}})^2)\iota} + \frac{9H^2\iota}{n_h^k},\forall (s,a,h,k) \right\}.$$
(c) With probability at least $1-\delta$, the following event holds:
$$\mathcal{E}_3 = \left\{ \left|\frac{\sum_{i=1}^{n_h^k}\left(\mathbbm{1}_{s_{h+1}^{k_{l_i}}}-\mathbb{P}_{s,a,h}\right)V_{h+1}^\star}{n_h^k(s,a)}\right| \leq H\sqrt{\frac{2\iota}{n_h^k(s,a)}}, \ \forall (s,a,h,k)\right\}.$$
(d) With probability at least $1-\delta$, the following event holds:
$$\mathcal{E}_4 = \left\{\sum_{h' = h}^H\big(1+\frac{1}{H}\big)^{h'-h}\sum_{k=1}^K\big(\mathbb{P}_{s_{h'}^{k}, a_{h'}^{k},{h'}}-\mathbbm{1}_{s_{{h'}+1}^{k}}\big)(V_{{h'}+1}^\star - V_{{h'}+1}^{\pi^k}) \leq 3\sqrt{2H^2T\iota},\ \forall h \in [H]\right\}.$$
(e) With probability at least $1-\delta$, the following event holds:
$$\mathcal{E}_5 = \left\{\left|\frac{1}{n_h^k}\sum_{i=1}^{n_h^k} \left(\mathbbm{1}_{s_{h+1}^{k_{l_i}}}-\mathbb{P}_{s,a,h}\right) V_{h+1}^{\star} \right| \leq 5\sqrt{\frac{\iota}{n_h^k}\mathbb{V}_{s,a,h}(V_{h+1}^{\star})} + \frac{9H\iota}{n_h^k},\forall (s,a,h,k) \right\}.$$
(f) With probability at least $1-\delta$, the following event holds:
$$\mathcal{E}_6 = \left\{ \frac{\sum_{i=1}^{n_h^k}\left(\mathbbm{1}_{s_{h+1}^{k_{l_i}}}-\mathbb{P}_{s,a,h}\right)(V_{h+1}^\star)^2}{n_h^k(s,a)} \leq H^2\sqrt{\frac{2\iota}{n_h^k(s,a)}}, \ \forall (s,a,h,k)\right\}.$$
(g) \textnormal{(Lemma C.5 in \citet{jin2018q})} With probability at least $1-\delta$, the following event holds:
$$\mathcal{E}_7 = \left\{\sum_{h=1}^H\sum_{k=1}^K\mathbb{V}_{s_h^{k}, a_h^{k}, h} (V_{h+1}^{\pi^k}) \lesssim HT +H^3\iota\right\}.$$
(h) With probability at least $1-\delta$, the following event holds:
$$\mathcal{E}_8 = \left\{\sum_{h=1}^H\sum_{k=1}^K\left(\mathbb{P}_{s_h^{k}, a_h^{k},h}-\mathbbm{1}_{s_{h+1}^{k}}\right)\left(V_{h+1}^\star - V_{h+1}^{\pi^k}\right) \leq \sqrt{2H^2T\iota}\right\}.$$
\end{lemma}

\begin{proof}[Proof of Lemma \ref{concentrations}]
(a) By Freedman's inequality in Lemma \ref{Freedman}, for any $\delta \in (0,1)$, with probability at least $1-\delta/SAT^2$, setting $\sigma^2 = TH$, $R=H$, and $m=\max\{1,\lceil \log_2(TH)\rceil\}$, it holds for a given $n_h^k(s,a)=n\in\mathbb{N}_{+}$ that
\begin{align*}
\left|\sum_{i=1}^{n}\left(\mathbbm{1}_{s_{h+1}^{k_{l_i}}}-\mathbb{P}_{s,a,h}\right)V_{h+1}^{k_{l_i}}\right|
&\leq \sqrt{8\max\left\{\sum_{i=1}^{n}\mathbb{V}_{s,a,h}(V_{h+1}^{k_{l_i}}),\frac{TH}{2^m}\right\}
\log\left(\frac{2mSAT^2}{\delta}\right)} \\
&\quad+\frac{4}{3}H\log\left(\frac{2mSAT^2}{\delta}\right).
\end{align*}
Since $H\leq T$, we have $m\leq 1+2\log_2 T\leq 2^{3/2}\sqrt{T}$. Together with $\iota=\log(2SAT/\delta)$, this implies
\[
\log\left(\frac{2mSAT^2}{\delta}\right)
=\iota+\log(mT)
\leq \frac{5}{2}\iota.
\]
Moreover, by the definition of $m$, $TH/2^m\leq 1$. Therefore, using $H\geq1$ and $\iota\geq\log 2$, the preceding inequality yields
    $$\frac{1}{n}\left|\sum_{i=1}^{n}\left(\mathbbm{1}_{s_{h+1}^{k_{l_i}}}-\mathbb{P}_{s,a,h}\right) V_{h+1}^{k_{l_i}}\right|\leq \frac{5}{n}\sqrt{ \sum_{i=1}^{n}\mathbb{V}_{s,a,h}(V_{h+1}^{k_{l_i}})\iota} + \frac{9H\iota}{n}.$$
    For any $k \in [K]$, we have $n_h^k(s,a) \in [\frac{T}{H}]$. The total number of possible combinations $(s,a,h,k,n) \in \mathcal{S} \times \mathcal{A} \times [H]\times [\frac{T}{H}] \times [\frac{T}{H}]$ is at most $SAT^2$. Hence, taking a union bound, with probability at least $1-\delta$, it holds simultaneously for all $(s,a,h,k) \in \mathcal{S} \times \mathcal{A} \times [H] \times [K]$ that:
    $$\left|\frac{1}{n_h^k}\sum_{i=1}^{n_h^k} \left(\mathbbm{1}_{s_{h+1}^{k_{l_i}}}-\mathbb{P}_{s,a,h}\right) V_{h+1}^{k_{l_i}} \right| \leq \frac{5}{n_h^k}\sqrt{ \sum_{i
=1}^{n_h^k}\mathbb{V}_{s,a,h}(V_{h+1}^{k_{l_i}})\iota} + \frac{9H\iota}{n_h^k}.$$

(b) The proof is similar to (a).

(c) The sequence
$$\left\{\left(\mathbbm{1}_{s_{h+1}^{k_{l_i}}}-\mathbb{P}_{s,a,h}\right)V_{h+1}^\star\right\}_{i\in \mathbb{N}^+}$$
    is a martingale-difference sequence bounded by $H$. Then according to Azuma-Hoeffding inequality, for any $\delta \in (0,1)$, with probability at least $1-\delta/SAT$, it holds for a given $n_h^k(s,a) =n \in \mathbb{N}_{+}$ that:
    $$\frac{1}{n}\left|\sum_{i=1}^{n}\left(\mathbbm{1}_{s_{h+1}^{k_{l_i}}}-\mathbb{P}_{s,a,h}\right)V_{h+1}^\star\right|\leq H\sqrt{\frac{2\iota}{n}}.$$
    For any $k \in [K]$, we have $n_h^k(s,a) \in [\frac{T}{H}]$. Considering all the possible combinations $(s,a,h,n) \in \mathcal{S} \times \mathcal{A} \times [H] \times [\frac{T}{H}]$, with probability at least $1-\delta$, it holds simultaneously for all $(s,a,h,k) \in \mathcal{S} \times \mathcal{A} \times [H] \times [K]$ that:
    $$\frac{1}{n_h^k}\left|\sum_{i=1}^{n_h^k}\left(\mathbbm{1}_{s_{h+1}^{k_{l_i}}}-\mathbb{P}_{s,a,h}\right)V_{h+1}^\star\right|\leq H\sqrt{\frac{2\iota}{n_h^k(s,a)}}.$$
    
(d) For $\mathcal{E}_{4}$, fix $h\in[H]$ and order the pairs $(k,h')$, $k\in[K]$ and $h\leq h'\leq H$, chronologically, first by episode and then by step. The corresponding martingale-difference sequence
    $$\left\{\left(1+\frac{1}{H}\right)^{h'-h}\left(\mathbb{P}_{s_{h'}^{k}, a_{h'}^{k},{h'}}-\mathbbm{1}_{s_{{h'}+1}^{k}}\right)\left(V_{{h'}+1}^\star - V_{{h'}+1}^{\pi^k}\right)\right\}_{(k,h')}$$
    is bounded by $3H$. According to Azuma-Hoeffding inequality, for any $\delta \in(0,1)$ and given $h \in [H]$, with probability at least $1- \delta/H$, it holds that:
    $$\sum_{h' = h}^H\left(1+\frac{1}{H}\right)^{h'-h}\sum_{k=1}^K\left(\mathbb{P}_{s_{h'}^{k}, a_{h'}^{k},{h'}}-\mathbbm{1}_{s_{{h'}+1}^{k}}\right)\left(V_{{h'}+1}^\star - V_{{h'}+1}^{\pi^k}\right) \leq 3\sqrt{2H^2T\iota}.$$
    Considering all the possible values of $h$, we finish the proof.

(e) The proof is similar to (a).

(f) The proof is similar to (c) with $$\left\{\left(\mathbbm{1}_{s_{h+1}^{k_{l_i}}}-\mathbb{P}_{s,a,h}\right)(V_{h+1}^\star)^2\right\}_{i\in \mathbb{N}^+}$$
    being a martingale-difference sequence bounded by $H^2$.

(h) The conclusion is proved directly by Azuma-Hoeffding inequality.

\end{proof}

\section{Properties of Value Function Estimates}
\label{valueproperty}
In this section, we provide some properties of value function estimates. Throughout the theoretical analysis, we take \(C_1=8\), \(C_2=119\), and \(C_3=2\) in Algorithm 1. Before proceeding, we first introduce some notation.

For any $(s,a,h,k)$ satisfying $t_h^k(s,a)>1$, let $\tilde{\mu}_h^{k}(s,a)$ and $\tilde{\sigma}_h^{k}(s,a)$ denote the values of $\mu_h(s,a)$ and $\sigma_h(s,a)$ at the end of stage $t_h^k(s,a)-1$, respectively. More specifically,
$$\tilde{\mu}_h^{k}(s,a) = \sum_{i=1}^{n_h^k} V_{h+1}^{k_{l_i}}(s_{h+1}^{k_{l_i}}),\ \tilde{\sigma}_h^{k}(s,a) = \sum_{i=1}^{n_h^k} \left(V_{h+1}^{k_{l_i}}(s_{h+1}^{k_{l_i}})\right)^2.$$
We also denote 
\begin{equation}
    \label{betabernnew}
        \tilde{b}_h^{k+1}(s,a) = \min\left\{8\sqrt{\frac{\tilde{W}_h^{k+1}(s,a)\iota}{n_h^{k+1}(s,a)}}+\frac{119H\iota}{n_h^{k+1}(s,a)},2\sqrt{\frac{H^2\iota}{n_h^{k+1}(s,a)}}\right\},
    \end{equation}
     with
     \begin{equation}
\label{wdef}
    \tilde{W}_h^k(s,a) =  \frac{\tilde{\sigma}_h^{k}(s,a)}{n_h^k(s,a)}- \left(\frac{\tilde{\mu}_h^{k}(s,a)} {n_h^k(s,a)}\right)^2.
\end{equation}
For $t_h^k(s,a)>1$, define
\[\widetilde Q_h^k(s,a):=r_h(s,a)+\frac{\widetilde\mu_h^k(s,a)}{n_h^k(s,a)}+\widetilde b_h^k(s,a).\]
Then by induction and the update rule in \Cref{update_q_bern}, it holds that
\begin{align}
\label{eq_bern_newupdate}
    Q_h^{k+1}(s,a) &=\mathbb{I}\left[t_h^{k+1}(s,a)=1\right]H \notag\\
    &\quad +\mathbb{I}\left[t_h^{k+1}(s,a)>1\right] \min\left\{r_h(s,a)+\frac{\tilde{\mu}_h^{k+1}(s,a)}{n_h^{k+1}(s,a)}+\tilde{b}_h^{k+1}(s,a),Q_h^k(s,a)\right\} \notag\\
    & = \mathbb I[t_h^{k+1}=1]H+\mathbb I[t_h^{k+1}>1]\min\{\widetilde Q_h^{k+1}(s,a),Q_h^{k}(s,a)\}.
\end{align}
We then prove the optimism properties of the $Q$-estimates.
\begin{lemma}
\label{Q}
    Under the event $\bigcap_{i=1}^{3}\mathcal{E}_i$ in Lemma \ref{concentrations}, it holds that for any $(s,a,h,k) \in  \mathcal{S} \times \mathcal{A} \times [H]\times[K]$:
    $$Q_h^k(s,a) \geq Q_h^{\star}(s,a).$$
    Then we have $V_h^k(s) \geq V_h^\star(s)$ for any $(s,h,k) \in  \mathcal{S}  \times [H]\times[K]$.
\end{lemma}
\begin{proof}[Proof of Lemma \ref{Q}]
We use mathematical induction on $k$ to prove $Q_h^k(s,a) \geq Q_h^{\star}(s,a)$ for any $(s,a,h,k) \in \mathcal{S} \times \mathcal{A} \times [H] \times [K]$. For $k=1$, $Q_h^1(s,a) =H \geq Q_h^{\star}(s,a)$ for any $(s,a,h) \in \mathcal{S} \times \mathcal{A} \times [H]$. For $k \geq 2$, assume we already have $Q_h^{k^\prime}(s,a) \geq Q_h^{\star}(s,a)$ for any $(s,a,h,k^\prime) \in \mathcal{S} \times \mathcal{A} \times [H] \times [k-1]$, then we will prove for any $(s,a,h) \in \mathcal{S} \times \mathcal{A} \times [H]$, $Q_h^{k}(s,a) \geq Q_h^{\star}(s,a)$. According to \Cref{eq_bern_newupdate}, the following relationship holds: 
$$ Q_h^{k}(s,a) =\mathbb{I}\left[t_h^{k}=1\right]H+\mathbb{I}\left[t_h^{k}>1\right] \min\{\tilde{Q}_h^{k}(s,a),Q_h^{k-1}(s,a)\}.$$
Then for any given $(s,a,h) \in \mathcal{S} \times \mathcal{A} \times [H]$, we have the following three cases:

(a) If $t_h^{k}(s,a)=1$, then $Q_h^k(s,a) =H \geq Q_h^\star(s,a)$.

(b) If $t_h^{k}(s,a)>1$ and $Q_h^k(s,a)= Q_h^{k-1}(s,a)$, then the conclusion holds. 

(c) If $t_h^{k}(s,a)>1$ and $Q_h^k(s,a) =r_h(s,a)+\tilde{\mu}_h^{k}(s,a)/n_h^k(s,a)+\tilde{b}_h^{k}(s,a)$, then by \Cref{eq_Bellman}, we have the following equality:
$$Q_h^{\star}(s,a) = r_h(s,a)+\mathbb{P}_{s,a,h}V_{h+1}^{\star}.$$
Then we have:
\begin{align}
\label{Q1middle}
    Q_h^k(s,a)-Q_h^{\star}(s,a) &=\tilde{\mu}_h^{k}(s,a)/n_h^k(s,a)+\tilde{b}_h^{k}(s,a)-\mathbb{P}_{s,a,h}V_{h+1}^{\star} \nonumber\\
    &= \frac{1}{n_h^k}\sum_{i=1}^{n_h^k}\left(V_{h+1}^{k_{l_i}}(s_{h+1}^{k_{l_i}})-\mathbb{P}_{s,a,h}V_{h+1}^{\star}\right)+\tilde{b}_h^{k}(s,a).
\end{align}
According to the definition of $l_i(s,a,h,k)$, we know $k_{l_i} < k$ for $i \in [n_h^k(s,a)]$. Then $Q_h^{k_{l_i}}(s,a) \geq Q_h^{\star}(s,a)$ based on the induction. Therefore, according to the update rule \Cref{update_q_bern} and \Cref{eq_Bellman}, for any $(s,h) \in \mathcal{S}\times [H]$ and any $i \in [n_h^k(s,a)]$, we have: 
\begin{equation*}
    V_{h+1}^{k_{l_i}}(s) = \max_{a \in \mathcal{A}}Q_{h+1}^{k_{l_i}}(s,a) \geq \max_{a \in \mathcal{A}} Q_{h+1}^{\star}(s,a) = V_{h+1}^\star(s),
\end{equation*} 
and thus for any $i \in [n_h^k(s,a)]$ it holds:
\begin{equation*}
    \mathbb{P}_{s,a,h}V_{h+1}^{k_{l_i}} \geq \mathbb{P}_{s,a,h}V_{h+1}^{\star}.
\end{equation*}
If the second branch in \Cref{betabernnew} attains the minimum, then by the pointwise optimism above and the event $\mathcal{E}_3$ in Lemma \ref{concentrations},
\begin{align*}
Q_h^k(s,a)-Q_h^\star(s,a)
&\geq \frac{1}{n_h^k}\sum_{i=1}^{n_h^k}\left(\mathbbm{1}_{s_{h+1}^{k_{l_i}}}-\mathbb{P}_{s,a,h}\right)V_{h+1}^\star+2H\sqrt{\frac{\iota}{n_h^k}}\\
&\geq -H\sqrt{\frac{2\iota}{n_h^k}}+2H\sqrt{\frac{\iota}{n_h^k}}\geq 0.
\end{align*}
If the first branch in \Cref{betabernnew} attains the minimum, applying the inequality above to \Cref{Q1middle} and using the event $\mathcal{E}_1$ in Lemma \ref{concentrations}, we have
\begin{align}
\label{optmiddle}
    Q_h^{k}(s, a) - Q_h^\star(s, a) &\geq  \frac{1}{n_h^k}\sum_{i=1}^{n_h^k}\left(V_{h+1}^{k_{l_i}}(s_{h+1}^{k_{l_i}})-\mathbb{P}_{s,a,h}V_{h+1}^{k_{l_i}}\right)+\tilde{b}_h^{k}(s,a) \notag\\
    &\geq 8\sqrt{\frac{\tilde{W}_h^{k}(s,a)\iota}{n_h^k}}-\frac{5}{n_h^k}\sqrt{ \sum_{i
=1}^{n_h^k}\mathbb{V}_{s,a,h}(V_{h+1}^{k_{l_i}})\iota}+\frac{110H\iota}{n_h^{k}}.
\end{align}
Next we will bound the difference $$I_1 = \frac{1}{n_h^k}\sum_{i
=1}^{n_h^k}\mathbb{V}_{s,a,h}(V_{h+1}^{k_{l_i}}) - \tilde{W}_h^{k}(s,a) \overset{\triangle}{=} X- \tilde{W}_h^{k}(s,a),$$
where $$X = \frac{1}{n_h^k}\sum_{i=1}^{n_h^k}\mathbb{V}_{s,a,h}(V_{h+1}^{k_{l_i}}).$$

According to the definition of $\mathbb{V}_{s,a,h}$, we also have
\begin{equation}
    \label{defV}
    \mathbb{V}_{s,a,h}V_{h+1}^{k_{l_i}} = \mathbb{P}_{s,a,h}(V_{h+1}^{k_{l_i}})^2 - \left(\mathbb{P}_{s,a,h} V_{h+1}^{k_{l_i}}\right)^2.
\end{equation}
Combining the results of \Cref{wdef} and \Cref{defV}, we can decompose the difference $I_1$ as:
\begin{align}
    \label{i1middle}
    I_1 &= \sum_{i = 1}^{n_h^k}\frac{1}{n_h^k} \Big( \mathbb{P}_{s,a,h} - \mathbbm{1}_{s_{h+1}^{k_{l_i}}}\Big)(V_{h+1}^{k_{l_i}})^2+ \left[\left(\sum_{i = 1}^{n_h^k}\frac{1}{n_h^k}V_{h+1}^{k_{l_i}}\left(s_{h+1}^{k_{l_i}}\right)\right)^2- \sum_{i = 1}^{n_h^k}\frac{1}{n_h^k}\Big(\mathbb{P}_{s,a,h} V_{h+1}^{k_{l_i}}\Big)^2\right].
\end{align}
For the first term, by the event $\mathcal{E}_2$ in Lemma \ref{concentrations} we have
\begin{align}
    \label{i11}
    &\left|\sum_{i = 1}^{n_h^k}\frac{1}{n_h^k} \left( \mathbb{P}_{s,a,h} - \mathbbm{1}_{s_{h+1}^{k_{l_i}}}\right)(V_{h+1}^{k_{l_i}})^2\right|\\
    &\leq 5\sqrt{\frac{\iota}{n_h^k} \sum_{i = 1}^{n_h^k} \frac{1}{n_h^k} \mathbb{V}_{s,a,h}((V_{h+1}^{k_{l_i}})^2)} + \frac{9H^2\iota}{n_h^k}\nonumber\\
    &\leq 10\sqrt{\frac{H^2\iota}{n_h^k}\sum_{i = 1}^{n_h^k}\frac{1}{n_h^k} \mathbb{V}_{s,a,h}V_{h+1}^{k_{l_i}}} + \frac{9H^2\iota}{n_h^k} \nonumber\\
    & = 10\sqrt{\frac{H^2\iota}{n_h^k}X} + \frac{9H^2\iota}{n_h^k}.
\end{align}
The last inequality is by $\mathbb{V}_{s,a,h}(X^2) \leq 4C^2\mathbb{V}_{s,a,h}(X)$ for $|X| \leq C$. For the second term of \Cref{i1middle}, by the Cauchy--Schwarz inequality, we reach
\begin{align}
    \label{i12}
    &\left(\sum_{i = 1}^{n_h^k}\frac{1}{n_h^k}V_{h+1}^{k_{l_i}}\left(s_{h+1}^{k_{l_i}}\right)\right)^2 - \sum_{i = 1}^{n_h^k}\frac{1}{n_h^k}\left(\mathbb{P}_{s,a,h} V_{h+1}^{k_{l_i}} \right)^2 \notag\\
    & \leq \left(\sum_{i = 1}^{n_h^k}\frac{1}{n_h^k}V_{h+1}^{k_{l_i}}\left(s_{h+1}^{k_{l_i}}\right)\right)^2 - \left(\sum_{i = 1}^{n_h^k}\frac{1}{n_h^k}\mathbb{P}_{s,a,h}V_{h+1}^{k_{l_i}}\right)^2 \nonumber\\
    & \leq 2H\left|\sum_{i = 1}^{n_h^k}\frac{1}{n_h^k} \left( \mathbbm{1}_{s_{h+1}^{k_{l_i}}}-\mathbb{P}_{s,a,h}\right)V_{h+1}^{k_{l_i}}\right| \nonumber\\
    &\leq 10\sqrt{\frac{H^2\iota}{n_h^k}X} + \frac{18H^2\iota}{n_h^k}.
\end{align}
The last inequality is by the event $\mathcal{E}_1$. Combining \Cref{i11} and \Cref{i12}, back to \Cref{i1middle}, we know
$$I_1 = X- \tilde{W}_h^{k}(s,a) \leq  20\sqrt{\frac{H^2\iota}{n_h^k}X} + \frac{27H^2\iota}{n_h^k}.$$
Solving the inequality, we have
$$X \leq 2\tilde{W}_h^{k}(s,a) + \frac{454H^2\iota}{n_h^k}.$$
Applying this inequality to \Cref{optmiddle} proves optimism for the first branch. Together with the second-branch argument above, this establishes $Q_h^k(s,a)\geq Q_h^\star(s,a)$. Finally, since
$$V_{h+1}^{k}(s) = \max_{a \in \mathcal{A}}Q_{h+1}^{k}(s,a) \geq \max_{a \in \mathcal{A}} Q_{h+1}^{\star}(s,a) = V_{h+1}^\star(s),$$
the proof is complete.
\end{proof}

The following lemma upper bounds the weighted cumulative estimation errors of $Q_h^k(s,a) - Q_h^\star(s,a)$.
\begin{lemma}
\label{recurQ}
For UCB-Bernstein+, under the event $\bigcap_{i=1}^{3}\mathcal{E}_i$ in Lemma \ref{concentrations}, for any given $h \in [H]$ and non-negative weight sequence $\{\omega_{h}^{k} = \omega_h(k,h)\}_{k=1}^K$ at step $h$, we have the following conclusion:
\begin{align*}
         &\sum_{k = 1}^K \omega_{h}(k,h) \left(Q_{h}^k - Q_{h}^{\star}\right)(s_{h}^{k},a_{h}^{k})\\ 
         &\lesssim \sum_{h_1 = h}^H\sum_{k=1}^K\omega_{h}(k,h_1)\mathbb{I}\left[t_{h_1}^{k}=1\right]H + \sqrt{H^5SA\|\omega_h(\cdot,{h})\|_{\infty}\|\omega_h(\cdot,{h})\|_{1}\iota}.
\end{align*}
\end{lemma}
\begin{proof}[Proof of Lemma \ref{recurQ}]
    By \Cref{eq_bern_newupdate}, we have the following relationship
\begin{align*}
Q_{h}^k(s_{h}^{k},a_h^{k}) &\leq \mathbb{I}\left[t_{h}^{k}(s_{h}^{k},a_{h}^{k})=1\right]H+\mathbb{I}\left[t_{h}^{k}(s_{h}^{k},a_{h}^{k})>1\right]\Big(r_{h}(s_{h}^{k},a_{h}^{k})+\frac{\tilde{\mu}_{h}^{k}(s_{h}^{k},a_{h}^{k})}{n_{h}^{k}(s_{h}^{k},a_{h}^{k})}+\tilde{b}_h^{k}(s_{h}^{k},a_{h}^{k})\Big)
\end{align*}
and according to the Bellman equality \Cref{eq_Bellman}, we have:
\begin{align*}
 Q_h^{\star}(s_h^{k},a_h^{k})=r_h(s_h^{k},a_h^{k}) + \mathbb{P}_{s_h^{k},a_h^{k},h}V_{h+1}^{\star}.
\end{align*}

Combining these two inequalities, it holds for any $h \in [H]$ that:
\begin{align}
    &(Q_h^{k} - Q_h^{\star})(s_h^{k}, a_h^{k})\leq \mathbb{I}[t_h^{k}=1]H + \nonumber\\
    &\mathbb{I}\left[t_h^{k}>1\right]\Bigg(\frac{\sum_{i=1}^{n_h^k}(V_{h+1}^{k_{l_i}}-V_{h+1}^\star)(s_{h+1}^{k_{l_i}})}{n_h^k(s_h^{k},a_h^{k})} +\frac{\sum_{i=1}^{n_h^k}\Big(\mathbbm{1}_{s_{h+1}^{k_{l_i}}} - \mathbb{P}_{s_h^{k}, a_h^{k},h}\Big)V_{h+1}^\star}{n_h^k(s_h^{k},a_h^{k})}+ \tilde{b}_h^{k}(s_h^{k},a_h^{k})\Bigg) \label{singlerecurinitial} \\
    &\leq \mathbb{I}\left[t_h^{k}=1\right]H+\mathbb{I}\left[t_h^{k}>1\right]\frac{\sum_{i=1}^{n_h^k}\left(V_{h+1}^{k_{l_i}}-V_{h+1}^\star\right)(s_{h+1}^{k_{l_i}})}{n_h^k(s_h^{k},a_h^{k})} + 4\mathbb{I}\left[t_h^{k}>1\right] \sqrt{\frac{H^2\iota}{n_h^k(s_h^{k},a_h^{k})}}.\label{singlerecur}
\end{align}
The last inequality is by the definition of $\tilde{b}_{h}^k$ in \Cref{betabernnew} and the event $\mathcal{E}_3$ in Lemma \ref{concentrations}. Then for any $h \in [H]$, we have
\begin{align}
\label{omega-V}
    &\sum_{k = 1}^K \omega_{h}(k,h)(Q_{h}^k - Q_{h}^{\star})(s_{h}^{k},a_{h}^{k}) \notag\\
    &\leq \sum_{k=1}^K\omega_{h}(k,h)\mathbb{I}\left[t_{h}^{k}=1\right]H +\sum_{k=1}^K\omega_{h}(k,h)\mathbb{I}\left[t_{h}^{k}>1\right] \frac{\sum_{i=1}^{n_{h}^k}\left(V_{h+1}^{k_{l_i}}-V_{h+1}^{\star}\right)(s_{h+1}^{k_{l_i}})}{n_{h}^k(s_{h}^{k},a_{h}^{k})} \notag\\
    &\quad + 4\sum_{k=1}^K\omega_{h}(k,h)\mathbb{I}\left[t_{h}^{k}>1\right]\sqrt{\frac{H^2\iota}{n_{h}^k(s_{h}^{k},a_{h}^{k})}}.
\end{align}

For the second term in \Cref{omega-V}, we have:
\begin{align}
    &\sum_{k=1}^K\omega_{{h}}(k,h)\mathbb{I}[t_{h}^{k}>1]\frac{\sum_{i=1}^{n_{h}^k}\left(V_{{h}+1}^{k_{l_i}}-V_{{h}+1}^{\star}\right)(s_{{h}+1}^{k_{l_i}})}{n_{h}^k(s_{h}^{k},a_{h}^{k})} \notag\\
   &= \sum_{k=1}^K\sum_{i=1}^{n_{h}^k}\omega_{h}(k,h)\frac{\left(V_{{h}+1}^{k_{l_i}}-V_{{h}+1}^{\star}\right)(s_{{h}+1}^{k_{l_i}})}{n_{h}^k(s_{h}^{k},a_{h}^{k})}\mathbb{I}[t_{h}^{k}>1]\cdot \sum_{k^{\prime}=1}^K\mathbb{I}\left[k_{l_i}=k^{\prime}\right] \notag\\
   &= \sum_{k=1}^K\sum_{i=1}^{n_{h}^k}\sum_{k^{\prime}=1}^K\omega_{h}(k,h)\frac{\left(V_{{h}+1}^{k^\prime}-V_{{h}+1}^{\star}\right)(s_{{h}+1}^{k^{\prime}})}{n_{h}^k(s_{h}^{k},a_{h}^{k})}\mathbb{I}[t_{h}^{k}>1] \cdot \mathbb{I}[k_{l_i}=k^{\prime}] \notag\\  
   &=\sum_{k^{\prime}=1}^K\left(\sum_{k=1}^K\omega_{h}(k,h)\frac{\sum_{i=1}^{n_{h}^k}\mathbb{I}\left[k_{l_i}=k^{\prime},t_{h}^{k}>1\right]}{n_{h}^k(s_{h}^{k},a_{h}^{k})}\right)\left(V_{{h}+1}^{k^\prime}-V_{{h}+1}^{\star}\right)(s_{{h}+1}^{k^{\prime}}) \notag\\
   &= \sum_{k' = 1}^K \omega_{h}(k',h+1)\left(V_{{h}+1}^{k^\prime}-V_{{h}+1}^{\star}\right)(s_{{h}+1}^{k^{\prime}}) \notag\\
   & \leq \sum_{k' = 1}^K \omega_{h}(k',h+1) (Q_{{h}+1}^{k'} - Q_{{h}+1}^{\star})(s_{{h}+1}^{k'},a_{{h}+1}^{k'}). \label{term2}
\end{align}
The inequality is by $Q_{{h}+1}^{k'}(s_{{h}+1}^{k'},a_{{h}+1}^{k'}) = V_{{h}+1}^{k'}(s_{{h}+1}^{k'})$, $Q_{{h}+1}^{\star}(s_{{h}+1}^{k'},a_{{h}+1}^{k'}) \leq V_{{h}+1}^{\star}(s_{{h}+1}^{k'})$.

For the last term in \Cref{omega-V}, by Lemma \ref{wN} with $\alpha=1/2$, we have:
\begin{align}
   \sum_{k=1}^K\omega_{h}(k,h)\mathbb{I}\left[t_{h}^{k}>1\right]\sqrt{\frac{H^2\iota}{n_{h}^k(s_{h}^{k},a_{h}^{k})}} \leq 4\sqrt{H^3SA\|\omega_h(\cdot,{h})\|_{\infty}\|\omega_h(\cdot,{h})\|_{1}\iota}. \label{term3}
\end{align}
Based on \Cref{omega-V}, by combining the results of \Cref{term2} and \Cref{term3}, we have
\begin{align}
\label{sumrecur}
         \sum_{k = 1}^K \omega_{h}(k,h) &\left(Q_{h}^k - Q_{h}^{\star}\right)(s_{h}^{k},a_{h}^{k}) \leq \sum_{k' = 1}^K \omega_{h}(k',h+1) (Q_{{h}+1}^{k'} - Q_{{h}+1}^{\star})(s_{{h}+1}^{k'},a_{{h}+1}^{k'})    \nonumber\\
         & +\sum_{k=1}^K\omega_{h}(k,h)\mathbb{I}\left[t_{h}^{k}=1\right]H + 16\sqrt{H^3SA\|\omega_h(\cdot,{h})\|_{\infty}\|\omega_h(\cdot,{h})\|_{1}\iota}.
\end{align}
Due to the property (e) of Lemma \ref{weight}, it holds for any $h_1 \in [h,H]$ that:
$$\|\omega_h(\cdot,{h_1})\|_{1}\leq \|\omega_h(\cdot,{h})\|_{1}, \ \|\omega_h(\cdot,{h_1})\|_{\infty} \leq (1+1/H)^{h_1-h}\|\omega_h(\cdot,{h})\|_{\infty} \leq e\|\omega_h(\cdot,{h})\|_{\infty}.$$
Therefore, by recursively applying \Cref{sumrecur} over steps $h,h+1,\ldots,H$ and using $Q_{H+1}^k(s,a) = Q_{H+1}^{\star}(s,a) = 0$, we obtain
\begin{align*}
    &\sum_{k = 1}^K \omega_{h}(k,h) \left(Q_{h}^k - Q_{h}^{\star}\right)(s_{h}^{k},a_{h}^{k})\\
    &\leq \sum_{h_1=h}^H\sum_{k=1}^K\omega_{h}(k,h_1)\mathbb{I}\left[t_{h_1}^{k}=1\right]H
    +16\sum_{h_1=h}^H\sqrt{H^3SA\|\omega_h(\cdot,{h_1})\|_{\infty}\|\omega_h(\cdot,{h_1})\|_{1}\iota}\\
    &\leq \sum_{h_1=h}^H\sum_{k=1}^K\omega_{h}(k,h_1)\mathbb{I}\left[t_{h_1}^{k}=1\right]H
    +16\sqrt{3H^5SA\|\omega_h(\cdot,{h})\|_{\infty}\|\omega_h(\cdot,{h})\|_{1}\iota},
\end{align*}
which proves the lemma.
\end{proof}
With the lemma above, we can further establish the following two conclusions.
\begin{lemma}
\label{clipq}
Under the event $\bigcap_{i=1}^{3}\mathcal{E}_i$ in Lemma \ref{concentrations}, for all $\epsilon \in (0, H]$, we have the following two conclusions:
\begin{align*} 
    \sum_{h=1}^H \sum_{k=1}^K \mathbb{I} \left [Q_h^k(s_h^{k}, a_h^{k}) - Q_h^{\star}(s_h^{k}, a_h^{k}) \geq \epsilon \right]
    &\lesssim \frac{H^6 S A \iota}{\epsilon^2} ,
\end{align*}
and
\begin{align*} 
     &\sum_{h=1}^H \sum_{k=1}^K \left(Q_h^k - Q_h^{\star}\right)(s_h^{k}, a_h^{k}) \mathbb{I}\left[ \left(Q_h^k - Q_h^{\star}\right)(s_h^{k}, a_h^{k}) \geq \epsilon \right] \lesssim \frac{H^6 S A \iota}{\epsilon}.
\end{align*}
\end{lemma}
\begin{proof}[Proof of Lemma \ref{clipq}]
Let $N =\max\{1,\lceil \log_2(H/\epsilon) \rceil\}$. For any given $h \in [H]$, any $k \in [K]$ and $i \in [N-1]$, define:
$$\omega_h^i(k,h) = \mathbb{I}\left[(Q_h^k - Q_h^{\star})(s_h^{k}, a_h^{k}) \in \big[2^{i-1}\epsilon,2^i\epsilon\big)\right] \leq 1,$$
and
$$\omega_h^N(k,h) = \mathbb{I}\left[(Q_h^k - Q_h^{\star})(s_h^{k}, a_h^{k}) \in \big[2^{N-1}\epsilon,H\big]\right] \leq 1.$$
Then for any $k' \in [K]$ and $h \leq h'\leq H$, we recursively define $\omega_h^i(k',h')$ according to \Cref{recurweight}: 
\begin{align}
\label{irecurweight}
\omega_{h}^i(k',h'+1) :=  \sum_{k=1}^K\omega_h^i(k,h')\frac{\sum_{i=1}^{n_{h'}^k(s_{h'}^{k},a_{h'}^{k})}\mathbb{I}\left[k_{l_i(s_{h'}^{k},a_{h'}^{k},h',k)}=k',t_{h'}^{k}(s_{h'}^{k},a_{h'}^{k})>1\right]}{n_{h'}^k(s_{h'}^{k},a_{h'}^{k})}
\end{align}
with norms $\|\omega_h^i(\cdot,h')\|_{1}$ and $\|\omega_h^i(\cdot,h')\|_{\infty}$ defined accordingly by \Cref{weightnorm} as:
\begin{equation}
\label{iweightnorm}
\|\omega_{h}^i(\cdot,h')\|_{\infty} := \max_{k' \in [K]} \omega_{h}^i(k',h'),\ \|\omega_{h}^i(\cdot,h')\|_{1} := \sum_{k'=1}^K \omega_{h}^i(k',h').
\end{equation}
Then we have the following relationship:
\begin{equation}
        \label{difflower}
        \sum_{k=1}^K \omega_{h}^i(k,h) \left(Q_h^k - Q_h^{\star}\right)(s_h^{k}, a_h^{k}) \geq 2^{i-1}\epsilon\|\omega_h^i(\cdot,h)\|_{1}.
    \end{equation}
By Lemma \ref{recurQ}, since $\omega_{h}^i(k,h) \leq 1$, we have
\begin{align}
     \sum_{k = 1}^K \omega_{h}^i(k,h) \left(Q_{h}^k - Q_{h}^{\star}\right)(s_{h}^{k},a_{h}^{k}) 
     &\lesssim  \sum_{h_1 = h}^H\sum_{k=1}^K\omega_{h}^i(k,h_1)\mathbb{I}\left[t_{h_1}^{k}=1\right]H + \sqrt{H^5SA\|\omega_{h}^i(\cdot,{h})\|_{1}\iota}. \label{recurQupper}
\end{align}
Combining the results of \Cref{recurQupper} and \Cref{difflower}, we have the following inequality:
    \begin{align}
    \label{ini}
        2^{i-1}\epsilon\|\omega_{h}^i(\cdot,{h})\|_{1} \lesssim \sqrt{H^5SA\|\omega_{h}^i(\cdot,{h})\|_{1}\iota} + \sum_{h_1 = h}^H\sum_{k=1}^K\omega_{h}^i(k,h_1)\mathbb{I}\left[t_{h_1}^{k}=1\right]H.
    \end{align}

Solving the inequality in \Cref{ini}, we can derive the following relationship:
    \begin{equation}
    \label{answer}
       \|\omega_{h}^i(\cdot,{h})\|_{1} \lesssim \frac{H^5SA\iota}{4^i\epsilon^2} + \frac{\sum_{h_1 = h}^H\sum_{k=1}^K\omega_{h}^i(k,h_1)\mathbb{I}\left[t_{h_1}^{k}=1\right]H}{2^i\epsilon}.
    \end{equation}
Note that for any  $k' \in [K]$ and $h \leq h' \leq H-1$, by \Cref{irecurweight}, we have
$$\sum_{i=1}^{N} \omega_h^i(k',h'+1)  = \sum_{k=1}^K\left(\sum_{i=1}^{N} \omega_h^i(k,h')\right)\frac{\sum_{i=1}^{n_{h'}^k}\mathbb{I}\left[k_{l_i}=k',t_{h'}^{k}>1\right]}{n_{h'}^k(s_{h'}^{k},a_{h'}^{k})}.$$
Then by mathematical induction on $h^\prime \in [h,H]$, it is straightforward to prove that for any $k' \in [K]$,
\begin{equation}
\label{omegaprime}
    \sum_{i=1}^{N} \omega_h^i(k',h') \leq \left( 1+\frac{1}{H}\right)^{h^\prime - h} < 3,
\end{equation} given \Cref{inftygoal} and the base case $\sum_{i=1}^N \omega_h^i(k',h) \leq 1$.

We now prove the first conclusion. By
    $$\mathbb{I} \left [\left(Q_h^k- Q_h^{\star}\right)(s_h^{k}, a_h^{k}) \geq \epsilon \right]  = \sum_{i=1}^N\omega_{h}^i(k,h),$$
    we have
    \begin{align}
    \label{clipkmj1}
        \sum_{h=1}^H\sum_{k=1}^K \mathbb{I} \left [Q_h^k(s_h^{k}, a_h^{k}) - Q_h^{\star}(s_h^{k}, a_h^{k}) \geq \epsilon \right] = \sum_{h=1}^H\sum_{i=1}^N\|\omega_{h}^i(\cdot,{h})\|_{1}. 
    \end{align}
    By \Cref{answer}, it holds that
    \begin{align}
    \label{clipkmj2}
        \sum_{i=1}^N\|\omega_{h}^i(\cdot,{h})\|_{1} & \lesssim \sum_{i=1}^N\frac{H^5SA\iota}{4^i\epsilon^2} + \sum_{i=1}^N\frac{\sum_{h_1 = h}^H\sum_{k=1}^K\omega_{h}^i(k,h_1)\mathbb{I}\left[t_{h_1}^{k}=1\right]H}{2^i\epsilon} \nonumber\\
        &\lesssim \frac{H^5 S A \iota}{\epsilon^2}  + \sum_{i=1}^N\frac{\sum_{h_1 = h}^H\sum_{k=1}^K\mathbb{I}\left[t_{h_1}^{k}=1\right]H}{2^i\epsilon} \nonumber\\
        &\lesssim  \frac{H^5 S A \iota}{\epsilon^2}  + \frac{H^3SA}{\epsilon}.
    \end{align}
    Here, the second inequality is because $0 \leq \omega_{h}^i(k,h_1)< 3$ by \Cref{omegaprime}. The last inequality is because 
    \begin{align}
\label{upperY}
    \sum_{h_1 = 1}^{H}\sum_{k=1}^K\mathbb{I}\left[t_{h_1}^{k}=1\right]H    &=H\sum_{{h_1} = 1}^{H}\sum_{k=1}^K\mathbb{I}\left[t_{h_1}^k(s_{h_1}^{k},a_{h_1}^{k})=1\right]\left(\sum_{s,a}\mathbb{I}\left[(s_{h_1}^{k},a_{h_1}^{k})=(s,a)\right]\right) \nonumber\\
   &= H\sum_{{h_1} = 1}^{H}\sum_{s,a}\sum_{k=1}^K\mathbb{I}\left[(s_{h_1}^{k},a_{h_1}^{k})=(s,a),t_{h_1}^k(s,a)=1\right] \nonumber\\
   & \leq H\sum_{{h_1} = 1}^{H}\sum_{s,a}e_1 =H^3SA.
\end{align}
    Combining the results of \Cref{clipkmj1} and \Cref{clipkmj2}, we reach
    \begin{align*}
        \sum_{h=1}^H \sum_{k=1}^K \mathbb{I} \left [Q_h^k(s_h^{k}, a_h^{k}) - Q_h^{\star}(s_h^{k}, a_h^{k}) \geq \epsilon \right] 
        \lesssim \frac{H^6 S A \iota}{\epsilon^2}.
    \end{align*}
    Now we finish the proof of the first conclusion. Further, noting that
   \begin{align*} 
  &\sum_{h=1}^H \sum_{k=1}^K \left(Q_h^k - Q_h^{\star}\right)(s_h^{k}, a_h^{k}) \mathbb{I}\left[ \left(Q_h^k - Q_h^{\star}\right)(s_h^{k}, a_h^{k}) \geq \epsilon \right] \\
    &\leq \sum_{h=1}^H\sum_{i = 1}^N 2^i\epsilon\|\omega_{h}^i(\cdot,{h})\|_{1} \nonumber\\
    &\lesssim \sum_{h=1}^H\sum_{i = 1}^N  \frac{H^5SA\iota}{2^i\epsilon} + \sum_{h=1}^H\sum_{h_1 = h}^H\sum_{k=1}^K \left(\sum_{i = 1}^N\omega_{h}^i(k,h_1)\right)\mathbb{I}\left[t_{h_1}^{k}=1\right]H \\
    &\lesssim \frac{H^6 S A \iota}{\epsilon} +  \sum_{h=1}^H\sum_{h_1 = h}^H\sum_{k=1}^K\mathbb{I}\left[t_{h_1}^{k}=1\right]H \\
    &\lesssim \frac{H^6 S A \iota}{\epsilon}.
\end{align*}
Here, the second inequality is by \Cref{answer}. The second last inequality is by \Cref{omegaprime}, and the last inequality is because of \Cref{upperY}. 
\end{proof}

The following lemma provides a preliminary upper bound on the regret using the Hoeffding-type bonus.
\begin{lemma}
\label{regreth}
For any $h \in [H]$, under the event $\bigcap_{i=1}^{4}\mathcal{E}_i$ in Lemma \ref{concentrations}, the following conclusion holds:
    $$\sum_{k=1}^K \left(V_{h}^{\star} - V_{h}^{\pi^k}\right)(s_{h}^{k}) \lesssim \sqrt{H^4SAT\iota} +H^3SA.$$
\end{lemma}
\begin{proof}[Proof of Lemma \ref{regreth}]
    For $h \in [H+1]$, denote:
$$\delta_h^k = \left(V_h^k - V_h^{\star}\right)(s_h^{k}),\ \zeta_h^k = \left(V_h^k - V_h^{\pi^k}\right)(s_h^{k}).$$
Here, $\delta_{H+1}^k =\zeta_{H+1}^k = 0$. Because $V_h^{\star}(s)=\sup _\pi V_h^\pi(s)$, we have $\delta_h^k \leq \zeta_h^k$ for any $h \in [H+1]$. In addition, as $V_h^k(s)\geq V_h^{\star}(s)$ for all $(s,h,k)\in \mathcal{S}\times [H]\times [K]$ according to Lemma \ref{Q}, we only need to bound $\sum_{k=1}^K \zeta_h^k$.
\begin{align}
\label{recurbegin}
    \sum_{k=1}^K \zeta_h^k &= \sum_{k=1}^K (Q_h^k - Q_h^{\pi^k})(s_h^{k}, a_h^{k}) \nonumber\\
    &= \sum_{k=1}^K (Q_h^k - Q_h^{\star})(s_h^{k}, a_h^{k}) + \sum_{k=1}^K (Q_h^\star - Q_h^{\pi^k})(s_h^{k}, a_h^{k}) \nonumber\\
    &= \sum_{k=1}^K (Q_h^{k} - Q_h^{\star})(s_h^{k}, a_h^{k}) + \sum_{k=1}^K\mathbb{P}_{s_h^{k}, a_h^{k},h}(V_{h+1}^\star - V_{h+1}^{\pi^k}).
\end{align}
The last equality holds due to \Cref{eq_Bellman}.

By \Cref{singlerecur}, we have
\begin{align}
\label{recur1begin}
        \sum_{k=1}^K(Q_h^{k} - Q_h^{\star})(s_h^{k}, a_h^{k})
        &\leq \sum_{k=1}^K\mathbb{I}[t_h^{k}=1]H + \sum_{k=1}^K\mathbb{I}\left[t_h^{k}>1\right]\frac{\sum_{i=1}^{n_h^k}(V_{h+1}^{k_{l_i}}-V_{h+1}^\star)(s_{h+1}^{k_{l_i}})}{n_h^k(s_h^{k},a_h^{k})}\notag\\
        &\quad+4\sum_{k=1}^K\mathbb{I}\left[t_h^{k}>1\right]\sqrt{\frac{H^2\iota}{n_h^k}}.
\end{align}
For the first term of \Cref{recur1begin}, similar to \Cref{upperY}, we have:
\begin{align}
\label{recur11}
   \sum_{k=1}^K\mathbb{I}\left[t_h^{k}=1\right]H  \leq H\sum_{s,a}e_1 = H^2SA.
\end{align}
For the last term of \Cref{recur1begin}, by Lemma \ref{wN} with $\alpha=1/2$, it holds that
\begin{equation}
    \label{recur12}
    \sum_{k=1}^K\mathbb{I}\left[t_h^{k}>1\right]\sqrt{\frac{H^2\iota}{n_h^k}} \leq 4\sqrt{H^2SAT\iota}.
\end{equation}
For the second term of \Cref{recur1begin}, by setting $\omega_{h}(k,h) = 1$ and applying \Cref{inftygoal} to \Cref{term2}, it holds that
\begin{equation}
    \label{recur13}
    \sum_{k=1}^K\mathbb{I}\left[t_h^{k}>1\right]\frac{\sum_{i=1}^{n_h^k}\left(V_{h+1}^{k_{l_i}}-V_{h+1}^{\star}\right)(s_{h+1}^{k_{l_i}})}{n_h^k(s_h^{k},a_h^{k})}\leq \left(1+\frac{1}{H}\right)\sum_{k=1}^K \delta_{h+1}^k.
\end{equation}
Moreover,
\begin{align*}
    &\sum_{k=1}^K\mathbb{P}_{s_h^{k}, a_h^{k},h}(V_{h+1}^\star - V_{h+1}^{\pi^k})\\
    &=\sum_{k=1}^K\left(\mathbb{P}_{s_h^{k}, a_h^{k},h}-\mathbbm{1}_{s_{h+1}^{k}}\right)\left(V_{h+1}^\star - V_{h+1}^{\pi^k}\right)
    +\sum_{k=1}^K\left(V_{h+1}^\star - V_{h+1}^{\pi^k}\right)(s_{h+1}^{k})\\
    &=\sum_{k=1}^K\left(\mathbb{P}_{s_h^{k}, a_h^{k},h}-\mathbbm{1}_{s_{h+1}^{k}}\right)\left(V_{h+1}^\star - V_{h+1}^{\pi^k}\right)
    +\sum_{k=1}^K\left(\zeta_{h+1}^k-\delta_{h+1}^k\right).
\end{align*}
Taking the above results \Cref{recur11}, \Cref{recur12} and \Cref{recur13} together with \Cref{recur1begin} and \Cref{recurbegin}, we obtain
\begin{align*}
    &\sum_{k=1}^K \zeta_h^k
    \leq \left(1+\frac{1}{H}\right)\sum_{k=1}^K \delta_{h+1}^k
    +\sum_{k=1}^K\left(\zeta_{h+1}^k-\delta_{h+1}^k\right)
    +H^2SA+16\sqrt{H^2SAT\iota}\\
    &\qquad+\sum_{k=1}^K\left(\mathbb{P}_{s_h^{k}, a_h^{k},h}-\mathbbm{1}_{s_{h+1}^{k}}\right)\left(V_{h+1}^\star - V_{h+1}^{\pi^k}\right)\\
    &=\sum_{k=1}^K\zeta_{h+1}^k+\frac{1}{H}\sum_{k=1}^K\delta_{h+1}^k
    +H^2SA+16\sqrt{H^2SAT\iota}+\sum_{k=1}^K\left(\mathbb{P}_{s_h^{k}, a_h^{k},h}-\mathbbm{1}_{s_{h+1}^{k}}\right)\left(V_{h+1}^\star - V_{h+1}^{\pi^k}\right)\\
    &\leq \left(1+\frac{1}{H}\right)\sum_{k=1}^K\zeta_{h+1}^k
    +H^2SA+16\sqrt{H^2SAT\iota}+\sum_{k=1}^K\left(\mathbb{P}_{s_h^{k}, a_h^{k},h}-\mathbbm{1}_{s_{h+1}^{k}}\right)\left(V_{h+1}^\star - V_{h+1}^{\pi^k}\right),
\end{align*}
where the last inequality uses $\delta_{h+1}^k\leq\zeta_{h+1}^k$.

By recursively applying the inequality above over $h,h+1,\ldots,H$, since $\zeta_{H+1}^k = 0$, we have
\begin{align*}
    &\sum_{k=1}^K \zeta_h^k
    \leq \sum_{h'=h}^H\left(1+\frac{1}{H}\right)^{h'-h}
    \left(H^2SA+16\sqrt{H^2SAT\iota}\right)\\
    &\quad+\sum_{h' = h}^H\left(1+\frac{1}{H}\right)^{h'-h}\sum_{k=1}^K
    \left(\mathbb{P}_{s_{h'}^{k}, a_{h'}^{k},{h'}}-\mathbbm{1}_{s_{{h'}+1}^{k}}\right)
    \left(V_{{h'}+1}^\star - V_{{h'}+1}^{\pi^k}\right)\\
    &\leq eH^3SA+16e\sqrt{H^4SAT\iota}
    +\sum_{h' = h}^H\left(1+\frac{1}{H}\right)^{h'-h}\sum_{k=1}^K
    \left(\mathbb{P}_{s_{h'}^{k}, a_{h'}^{k},{h'}}-\mathbbm{1}_{s_{{h'}+1}^{k}}\right)
    \left(V_{{h'}+1}^\star - V_{{h'}+1}^{\pi^k}\right)\\
    &\lesssim \sqrt{H^4SAT\iota}+H^3SA.
\end{align*}
The second inequality uses $\left(1+1/H\right)^H\leq e$, and the last inequality follows from the event $\mathcal{E}_4$ in Lemma \ref{concentrations}. Then we finish the proof.
\end{proof}

\section{Proof of \texorpdfstring{\Cref{thm_gapregret}}{Theorem 3.1}}
\label{gapregretproof}
\begin{lemma}
\label{recurQb}
For the UCB-Bernstein+ algorithm, under the event $\bigcap_{i=1}^{6}\mathcal{E}_i$ in Lemma \ref{concentrations}, 
for any given step $h \in [H]$ and any non-negative weight sequence $\{\omega_{h}^{k}\}_{k=1}^K$ at step $h$, it holds simultaneously for any $(s,a) \in \mathcal{S} \times \mathcal{A} $ and subsequent step $h' \in [h,H]$ that:
\begin{align}
\label{seprecurb}
\sum_{k=1}^K &\omega_{h}(k,{h'},s,a) (Q_{h'}^{k} - Q_{h'}^{\star})(s_{h'}^{k},a_{h'}^{k})  \leq \sum_{k' = 1}^K \tilde{\omega}_{h}(k',h'+1, s, a) (Q_{{h'}+1}^{k'} - Q_{{h'}+1}^{\star})(s_{{h'}+1}^{k'},a_{{h'}+1}^{k'}) \notag\\
     &+ 252\sqrt{(\mathbb{Q}_{h'}^\star+\beta^2)H\|\omega_h(\cdot,{h'})\|_{\infty}\|\omega_h(\cdot,{h'},s,a)\|_{1}\iota}  + 152\sum_{k=1}^K\omega_{h}(k,h',s,a)Z_{h'}^{k}
\end{align}
and
\begin{align}
\label{sumrecurb}
         \sum_{k = 1}^K &\omega_{h}(k,h') \left(Q_{h'}^k - Q_{h'}^{\star}\right)(s_{h'}^{k},a_{h'}^{k}) \leq \sum_{k' = 1}^K \omega_{h}(k',h'+1) (Q_{{h'}+1}^{k'} - Q_{{h'}+1}^{\star})(s_{{h'}+1}^{k'},a_{{h'}+1}^{k'})    \nonumber\\
         &+ 252\sum_{s,a}\sqrt{(\mathbb{Q}_{h'}^\star+\beta^2)H\|\omega_h(\cdot,{h'})\|_{\infty}\|\omega_h(\cdot,{h'},s,a)\|_{1}\iota}  + 152\sum_{k=1}^K\omega_{h}(k,h')Z_{h'}^{k},
\end{align}
where 
\begin{align}
\label{defZ}
    Z_{h'}^{k} = \mathbb{I}\left[t_{h'}^{k}=1\right]H +\mathbb{I}\left[t_{h'}^{k}>1\right]\frac{H^2\iota/\beta+\sqrt{H\Psi_{h'}^k(s_{{h'}}^{k},a_{{h'}}^{k})\iota}}{n_{h'}^k(s_{{h'}}^{k},a_{{h'}}^{k})}.
\end{align}
with
\begin{equation*}
   \Psi_h^k(s_h^{k}, a_h^{k}) = \sum_{i=1}^{n_h^k} \left(V_{h+1}^{k_{l_i}}-\hat{V}_{h+1}^{k_{l_i}}\right)(s_{h+1}^{k_{l_i}}).
\end{equation*}
\end{lemma}
\begin{proof}[Proof of Lemma \ref{recurQb}]
    By \Cref{singlerecurinitial} and the definition of $\tilde{b}_{h'}^k$ in \Cref{betabernnew}, for any $h' \in [h,H]$, we have
\begin{align}
\label{recur2begin}
        &(Q_{h'}^{k} - Q_{h'}^{\star})(s_{h'}^{k}, a_{h'}^{k}) \nonumber\\
        &\leq \mathbb{I}\left[t_{h'}^{k}=1\right]H
        +\mathbb{I}\left[t_{h'}^{k}>1\right]\frac{\sum_{i=1}^{n_{h'}^k}(V_{{h'}+1}^{k_{l_i}}-V_{{h'}+1}^\star)(s_{{h'}+1}^{k_{l_i}})}{n_{h'}^k(s_{h'}^{k},a_{h'}^{k})} \nonumber\\
        &+\mathbb{I}\left[t_{h'}^{k}>1\right]\left(8\sqrt{\frac{\tilde{W}_{h'}^{k}(s_{h'}^{k}, a_{h'}^{k})\iota}{n_{h'}^k}}+\frac{119H\iota}{n_{h'}^k}\right)
        +\mathbb{I}\left[t_{h'}^{k}>1\right]\frac{\sum_{i=1}^{n_{h'}^k}\Big(\mathbbm{1}_{s_{{h'}+1}^{k_{l_i}}} - \mathbb{P}_{s_{h'}^{k}, a_{h'}^{k},{h'}}\Big)V_{{h'}+1}^\star}{n_{h'}^k(s_{h'}^{k},a_{h'}^{k})}.
\end{align}
By the event $\mathcal{E}_5$ in Lemma \ref{concentrations}, we can bound the last term in \Cref{recur2begin} as:
\begin{align}
\label{recur2begin2}
    \frac{\sum_{i=1}^{n_{h'}^k}\left(\mathbbm{1}_{s_{{h'}+1}^{k_{l_i}}} - \mathbb{P}_{s_{h'}^{k}, a_{h'}^{k},{h'}}\right)V_{{h'}+1}^\star}{n_{h'}^k(s_{h'}^{k},a_{h'}^{k})}
    \leq 5\sqrt{\frac{\mathbb{Q}_{h'}^\star\iota}{n_{h'}^k(s_{h'}^{k}, a_{h'}^{k})}} + \frac{9H\iota}{n_{h'}^k(s_{h'}^{k}, a_{h'}^{k})}.
\end{align}
For the variance-dependent bonus term in \Cref{recur2begin}, we will upper bound $\tilde{W}_h^k(s_h^k,a_h^k)$. Define the \textbf{pseudo value function} $\hat{V}_{h}^{k}(s)$ as
$$\hat{V}_{h}^{k}(s) = \max\left\{\min\left\{V_{h}^\star(s) + \beta, V_h^k(s)\right\}, V_{h}^\star(s)\right\}.$$
Due to the optimistic property in Lemma \ref{Q}, it holds that $V_{h}^{k}(s) \geq \hat{V}_{h}^{k}(s)$. Then according to the definition of $\tilde{W}_h^{k}(s_h^{k}, a_h^{k})$ in \Cref{wdef}, we have 
\begin{align}
\label{refestimate}
 \sqrt{\frac{\tilde{W}_h^{k}(s_h^{k}, a_h^{k})}{n_h^k(s_h^{k}, a_h^{k})}} \leq \sqrt{\frac{J_{1,h}^{k}+J_{2,h}^{k}}{n_h^k(s_h^{k}, a_h^{k})}},
\end{align}
where:
\begin{equation*}
    \label{j1hk}
    J_{1,h}^{k} = \frac{\sum_{i=1}^{n_h^k}\left(V_{h+1}^{k_{l_i}}(s_{h+1}^{k_{l_i}})\right)^2}{n_h^k(s_h^{k}, a_h^{k})}- \frac{\sum_{i=1}^{n_h^k}\left(\hat{V}_{h+1}^{k_{l_i}}(s_{h+1}^{k_{l_i}})\right)^2}{n_h^k(s_h^{k}, a_h^{k})}
\end{equation*}
and
\begin{equation*}
    \label{j2hk}
    J_{2,h}^{k} = \frac{\sum_{i=1}^{n_h^k}\left(\hat{V}_{h+1}^{k_{l_i}}(s_{h+1}^{k_{l_i}})\right)^2}{n_h^k(s_h^{k}, a_h^{k})}-\left(\frac{\sum_{i=1}^{n_h^k}\hat{V}_{h+1}^{k_{l_i}}(s_{h+1}^{k_{l_i}})}{n_h^k(s_h^{k}, a_h^{k})}\right)^2.
\end{equation*}
Now we want to bound both $J_{1,h}^{k}$ and $J_{2,h}^{k}$. Note that
\begin{align}
\label{j1upper}
    J_{1,h}^{k}
    &= \frac{\sum_{i=1}^{n_h^k}\big(V_{h+1}^{k_{l_i}}+\hat{V}_{h+1}^{k_{l_i}}\big) \big(V_{h+1}^{k_{l_i}}-\hat{V}_{h+1}^{k_{l_i}}\big)(s_{h+1}^{k_{l_i}})}{n_h^k(s_h^{k}, a_h^{k})}\leq \frac{2H\Psi_h^k(s_h^{k}, a_h^{k})}{n_h^k(s_h^{k}, a_h^{k})}.
\end{align}
For the second term $J_{2,h}^{k}$, using the Cauchy--Schwarz inequality, we have:
\begin{align*}   
    J_{2,h}^{k} &= \frac{\sum_{i=1}^{n_h^k}\left(\hat{V}_{h+1}^{k_{l_i}}(s_{h+1}^{k_{l_i}})-\frac{\sum_{n=1}^{n_h^k}\hat{V}_{h+1}^{k_{l_n}}(s_{h+1}^{k_{l_n}})}{n_h^k(s_h^{k}, a_h^{k})}\right)^2}{n_h^k(s_h^{k}, a_h^{k})} \leq 2\left(J_{2,h,1}^{k} + J_{2,h,2}^{k}\right),
\end{align*}
where:
$$J_{2,h,1}^{k} = \frac{\sum_{i=1}^{n_h^k}\left(\left(\hat{V}_{h+1}^{k_{l_i}}-V_{h+1}^{\star}\right)(s_{h+1}^{k_{l_i}})+\frac{\sum_{n=1}^{n_h^k}\left(V_{h+1}^{\star}-\hat{V}_{h+1}^{k_{l_n}}\right)(s_{h+1}^{k_{l_n}})}{n_h^k(s_h^{k}, a_h^{k})}\right)^2}{n_h^k(s_h^{k}, a_h^{k})},$$
and
\begin{align*}
    J_{2,h,2}^{k}&= \frac{\sum_{i=1}^{n_h^k}\left(V_{h+1}^{\star}(s_{h+1}^{k_{l_i}})-\frac{\sum_{n=1}^{n_h^k}V_{h+1}^{\star}(s_{h+1}^{k_{l_n}})}{n_h^k(s_h^{k}, a_h^{k})}\right)^2}{n_h^k(s_h^{k}, a_h^{k})} \\
    &= \frac{\sum_{i=1}^{n_h^k}\left(V_{h+1}^{\star}(s_{h+1}^{k_{l_i}})\right)^2}{n_h^k(s_h^{k}, a_h^{k})}-\left(\frac{\sum_{i=1}^{n_h^k}V_{h+1}^{\star}(s_{h+1}^{k_{l_i}})}{n_h^k(s_h^{k}, a_h^{k})}\right)^2.
\end{align*}
Since $V_{h+1}^{\star}(s) \leq \hat{V}_{h+1}^{k_{l_i}}(s) \leq V_{h+1}^{\star}(s)+\beta$, it holds that:
\begin{align*}
    &\left|\left(\hat{V}_{h+1}^{k_{l_i}}-V_{h+1}^{\star}\right)(s_{h+1}^{k_{l_i}})+\frac{\sum_{n=1}^{n_h^k}\left(V_{h+1}^{\star}-\hat{V}_{h+1}^{k_{l_n}}\right)(s_{h+1}^{k_{l_n}})}{n_h^k(s_h^{k}, a_h^{k})}\right| \\
& \leq \left|\left(\hat{V}_{h+1}^{k_{l_i}}-V_{h+1}^{\star}\right)(s_{h+1}^{k_{l_i}})\right|+\left|\frac{\sum_{n=1}^{n_h^k}\left(V_{h+1}^{\star}-\hat{V}_{h+1}^{k_{l_n}}\right)(s_{h+1}^{k_{l_n}})}{n_h^k(s_h^{k}, a_h^{k})}\right| \\
&\leq 2\beta.
\end{align*}
Therefore, applying this inequality to $J_{2,h,1}^{k}$, we have $J_{2,h,1}^{k} \leq 4\beta^2$. 

Moreover, we claim that:
\begin{equation}
    \label{j2upper}
    J_{2,h,2}^{k} \leq \mathbb{V}_{s_h^{k}, a_h^{k}, h} (V_{h+1}^\star) + 3\sqrt{2}H^2\sqrt{\frac{\iota}{n_h^k(s_h^{k}, a_h^{k})}}.
\end{equation}
This is because
\begin{align}
    &J_{2,h,2}^{k} - \mathbb{V}_{s_h^{k}, a_h^{k}, h} (V_{h+1}^\star) \notag\\
    &= \frac{\sum_{i=1}^{n_h^k}\left(V_{h+1}^{\star}(s_{h+1}^{k_{l_i}})\right)^2}{n_h^k(s_h^{k}, a_h^{k})}-\left(\frac{\sum_{i=1}^{n_h^k}V_{h+1}^{\star}(s_{h+1}^{k_{l_i}})}{n_h^k(s_h^{k}, a_h^{k})}\right)^2 -\mathbb{P}_{s_h^{k}, a_h^{k},h}\left(V_{h+1}^\star\right)^2 + \left(\mathbb{P}_{s_h^{k}, a_h^{k},h}V_{h+1}^\star\right)^2 \notag\\
    &= \frac{\sum_{i=1}^{n_h^k}\left(\mathbbm{1}_{s_{h+1}^{k_{l_i}}}-\mathbb{P}_{s_h^{k}, a_h^{k},h}\right)\left(V_{h+1}^\star\right)^2}{n_h^k(s_h^{k}, a_h^{k})} + \left(\mathbb{P}_{s_h^{k}, a_h^{k},h}V_{h+1}^\star\right)^2 - \left(\frac{\sum_{i=1}^{n_h^k}V_{h+1}^\star(s_{h+1}^{k_{l_i}})}{n_h^k(s_h^{k}, a_h^{k})}\right)^2 \notag\\
    & \leq \frac{\sum_{i=1}^{n_h^k}\left(\mathbbm{1}_{s_{h+1}^{k_{l_i}}}-\mathbb{P}_{s_h^{k}, a_h^{k},h}\right)\left(V_{h+1}^\star\right)^2}{n_h^k(s_h^{k}, a_h^{k})} + 2H\left| \frac{\sum_{i=1}^{n_h^k}\left(\mathbbm{1}_{s_{h+1}^{k_{l_i}}}-\mathbb{P}_{s_h^{k}, a_h^{k},h}\right)V_{h+1}^\star}{n_h^k(s_h^{k}, a_h^{k})}\right|\notag\\
    &\leq 3\sqrt{2}H^2\sqrt{\frac{\iota}{n_h^k(s_h^{k}, a_h^{k})}}. \label{emdiff}
\end{align}
The last inequality is because of the events $\mathcal{E}_3$ and $\mathcal{E}_6$ in Lemma \ref{concentrations}.

Applying the results of \Cref{j1upper}, \Cref{j2upper} and $J_{2,h,1}^{k} \leq 4\beta^2$ to \Cref{refestimate}, we reach
\begin{align}
 &\sqrt{\frac{\tilde{W}_h^{k}(s_h^{k}, a_h^{k})}{n_h^k(s_h^{k}, a_h^{k})}} \leq \frac{\sqrt{2H\Psi_h^k(s_h^{k}, a_h^{k})}}{n_h^k(s_h^{k}, a_h^{k})} + \sqrt{\frac{2\mathbb{V}_{s_h^{k}, a_h^{k}, h} (V_{h+1}^\star)}{n_h^k(s_h^{k}, a_h^{k})}} + \sqrt{\frac{8\beta^2}{n_h^k(s_h^{k}, a_h^{k})}}+ \frac{\sqrt{6\sqrt{2}}H\iota^{\frac{1}{4}}} {(n_h^k(s_h^{k}, a_h^{k}))^{\frac{3}{4}}}. \label{singlerefupper}
\end{align}
Using the definition of $\mathbb{Q}_h^\star$, we further derive the following upper bound for $\tilde{W}_h^k(s_h^k,a_h^k)$:
\begin{align}
\label{recur2begin3}
    \sqrt{\frac{\tilde{W}_{h'}^{k}(s_{h'}^{k}, a_{h'}^{k})\iota}{n_{h'}^k(s_{h'}^{k}, a_{h'}^{k})}}
    &\leq  \frac{\sqrt{2H\Psi_{h'}^k(s_{h'}^{k}, a_{h'}^{k})\iota}}{n_{h'}^k(s_{h'}^{k}, a_{h'}^{k})}
    +\sqrt{\frac{2\mathbb{Q}_{h'}^\star\iota}{n_{h'}^k(s_{h'}^{k}, a_{h'}^{k})}} +\sqrt{\frac{8\beta^2\iota}{n_{h'}^k(s_{h'}^{k}, a_{h'}^{k})}}
    +\frac{\sqrt{6\sqrt{2}}H\iota^{\frac{3}{4}}}{n_{h'}^k(s_{h'}^{k}, a_{h'}^{k})^{\frac{3}{4}}}.
\end{align}
Here, we use the AM-GM inequality
$$\frac{H\iota^{\frac{3}{4}}}{n_{h'}^k(s_{h'}^{k}, a_{h'}^{k})^{\frac{3}{4}}} \leq \sqrt{\frac{\beta^2\iota}{n_{h'}^k(s_{h'}^{k}, a_{h'}^{k})}} + \frac{H^2\iota}{\beta n_{h'}^k(s_{h'}^{k}, a_{h'}^{k})}.$$
Combining \Cref{recur2begin}, \Cref{recur2begin2}, and \Cref{recur2begin3}, and using $\beta\leq H$, we then reach
\begin{align*}
        &(Q_{h'}^{k} - Q_{h'}^{\star})(s_{h'}^{k}, a_{h'}^{k}) \\
        &\leq \mathbb{I}\left[t_{h'}^{k}>1\right]\frac{\sum_{i=1}^{n_{h'}^k}(V_{{h'}+1}^{k_{l_i}}-V_{{h'}+1}^\star)(s_{{h'}+1}^{k_{l_i}})}{n_{h'}^k(s_{h'}^{k},a_{h'}^{k})}
        +63\mathbb{I}\left[t_{h'}^{k}>1\right]\sqrt{\frac{(\mathbb{Q}_{h'}^\star+\beta^2)\iota}{n_{h'}^k(s_{h'}^{k}, a_{h'}^{k})}}
        +152Z_{{h'}}^{k},
\end{align*}
where we use $\beta\leq H$, together with
\[
8\sqrt{2}+5+16\sqrt{2}+8\sqrt{6\sqrt{2}}<63,\ 128+8\sqrt{6\sqrt{2}}<152.
\]
Therefore, it holds that
    \begin{align}    
        &\sum_{k=1}^K \omega_{h}(k,{h'},s,a) (Q_{h'}^{k} - Q_{h'}^{\star})(s_{h'}^{k}, a_{h'}^{k}) \notag\\
        &\leq \sum_{k=1}^K \omega_{h}(k,{h'},s,a)\mathbb{I}\left[t_{h'}^{k}>1\right]\frac{\sum_{i=1}^{n_{h'}^k}\left(V_{{h'}+1}^{k_{l_i}}-V_{{h'}+1}^\star\right)(s_{{h'}+1}^{k_{l_i}})}{n_{h'}^k(s_{h'}^{k},a_{h'}^{k})} \notag\\
        &\quad +63\sum_{k=1}^K \omega_{h}(k,{h'},s,a)\mathbb{I}\left[t_{h'}^{k}>1\right]\sqrt{\frac{(\mathbb{Q}_{h'}^\star+\beta^2)\iota}{n_{h'}^k}}
        +152\sum_{k=1}^K \omega_{h}(k,{h'},s,a) Z_{{h'}}^{k}. \label{Rstarrecur}
    \end{align}
    For the first term of \Cref{Rstarrecur}, note that:
    \begin{align}
    &\sum_{k=1}^K\omega_{{h}}(k,h',s,a)\mathbb{I}[t_{h'}^{k}>1]\frac{\sum_{i=1}^{n_{h'}^k}\left(V_{{h'}+1}^{k_{l_i}}-V_{{h'}+1}^{\star}\right)(s_{{h'}+1}^{k_{l_i}})}{n_{h'}^k(s_{h'}^{k},a_{h'}^{k})} \notag\\
   &= \sum_{k=1}^K\sum_{i=1}^{n_{h'}^k}\omega_{h}(k,h',s,a)\frac{\left(V_{{h'}+1}^{k_{l_i}}-V_{{h'}+1}^{\star}\right)(s_{{h'}+1}^{k_{l_i}})}{n_{h'}^k(s_{h'}^{k},a_{h'}^{k})}\mathbb{I}[t_{h'}^{k}>1]\cdot \sum_{k^{\prime}=1}^K\mathbb{I}\left[k_{l_i}=k^{\prime}\right] \notag\\
   &= \sum_{k=1}^K\sum_{i=1}^{n_{h'}^k}\sum_{k^{\prime}=1}^K\omega_{h}(k,h',s,a)\frac{\left(V_{{h'}+1}^{k^\prime}-V_{{h'}+1}^{\star}\right)(s_{{h'}+1}^{k^{\prime}})}{n_{h'}^k(s_{h'}^{k},a_{h'}^{k})}\mathbb{I}[t_{h'}^{k}>1] \cdot \mathbb{I}[k_{l_i}=k^{\prime}] \notag\\  
   &=\sum_{k^{\prime}=1}^K\left(\sum_{k=1}^K\omega_{h}(k,h',s,a)\frac{\sum_{i=1}^{n_{h'}^k}\mathbb{I}\left[k_{l_i}=k^{\prime},t_{h'}^{k}>1\right]}{n_{h'}^k(s_{h'}^{k},a_{h'}^{k})}\right)\left(V_{{h'}+1}^{k^\prime}-V_{{h'}+1}^{\star}\right)(s_{{h'}+1}^{k^{\prime}}) \notag\\
   & \leq \sum_{k' = 1}^K \tilde{\omega}_{h}(k',h'+1, s, a) (Q_{{h'}+1}^{k'} - Q_{{h'}+1}^{\star})(s_{{h'}+1}^{k'},a_{{h'}+1}^{k'}). \label{recurvstar}
\end{align}
The inequality is because $$Q_{{h'}+1}^{k'}(s_{{h'}+1}^{k'},a_{{h'}+1}^{k'}) = V_{{h'}+1}^{k'}(s_{{h'}+1}^{k'}),\ Q_{{h'}+1}^{\star}(s_{{h'}+1}^{k'},a_{{h'}+1}^{k'}) \leq V_{{h'}+1}^{\star}(s_{{h'}+1}^{k'})$$ and
\begin{align*}
    &\sum_{k=1}^K\omega_{h}(k,h',s,a)\frac{\sum_{i=1}^{n_{h'}^k}\mathbb{I}\left[k_{l_i}=k^{\prime},t_{h'}^{k}>1\right]}{n_{h'}^k(s_{h'}^{k},a_{h'}^{k})} \\
    &= \mathbb{I}\left[(s_{h'}^{k'},a_{h'}^{k'}) = (s,a)\right]\sum_{k = 1}^K\omega_{h}(k,h')\frac{\sum_{i=1}^{n_{h'}^k}\mathbb{I}\left[k_{l_i}=k^{\prime},t_{h'}^{k}>1\right]}{n_{h'}^k(s_{h'}^{k},a_{h'}^{k})} =  \tilde{\omega}_{h}(k',h'+1, s, a)
\end{align*}
because $\sum_{i=1}^{n_{h'}^k}\mathbb{I}\left[k_{l_i}=k^{\prime},t_{h'}^{k}>1\right] = 1$
if and only if $(s_{h'}^{k},a_{h'}^{k}) = (s_{h'}^{k^{\prime}},a_{h'}^{k^{\prime}})$ and $1<t_{h'}^k(s_{h'}^{k},a_{h'}^{k}) = t_{h'}^{k^{\prime}}(s_{h'}^{k},a_{h'}^{k})+1$. 

For the second term of \Cref{Rstarrecur}, by Lemma \ref{wN} with $\alpha=1/2$, it holds that:
    \begin{align}
    \label{sum12}
        &\sum_{k=1}^K \omega_{h}(k,h',s,a)\sqrt{\frac{(\mathbb{Q}_{h'}^\star+\beta^2)\iota}{n_{h'}^k}}\mathbb{I}\left[t_{h'}^{k}>1\right] \leq 4\sqrt{(\mathbb{Q}_{h'}^\star+\beta^2)H\|\omega_h(\cdot,{h'})\|_{\infty}\|\omega_h(\cdot,{h'},s,a)\|_{1}\iota}.
    \end{align}
Applying \Cref{recurvstar} and \Cref{sum12} to \Cref{Rstarrecur}, for any $h' \in [h,H]$, we obtain
\begin{align*}
    &\sum_{k=1}^K \omega_{h}(k,{h'},s,a) (Q_{h'}^{k} - Q_{h'}^{\star})(s_{h'}^{k}, a_{h'}^{k}) \\
    &\leq \sum_{k' = 1}^K \tilde{\omega}_{h}(k',h'+1, s, a) (Q_{{h'}+1}^{k'} - Q_{{h'}+1}^{\star})(s_{{h'}+1}^{k'},a_{{h'}+1}^{k'})\\
    &\quad+252\sqrt{(\mathbb{Q}_{h'}^\star+\beta^2)H\|\omega_h(\cdot,{h'})\|_{\infty}\|\omega_h(\cdot,{h'},s,a)\|_{1}\iota}
    +152\sum_{k=1}^K \omega_{h}(k,{h'},s,a) Z_{{h'}}^{k},
\end{align*}
which proves the first conclusion \Cref{seprecurb}. The second conclusion follows by summing the inequality above over all state-action pairs and using property (a) of Lemma \ref{weight}.
\end{proof}

\begin{lemma}
\label{recurQbtotal}
For UCB-Bernstein+, under the event $\bigcap_{i=1}^{6}\mathcal{E}_i$ in Lemma \ref{concentrations}, for any given $h \in [H]$ and non-negative weight sequence $\{\omega_{h}^{k}\}_{k}$ at step $h$, it holds simultaneously for any subsequent step $h' \in [h,H]$ that:
\begin{align*}
         &\sum_{k = 1}^K \omega_{h}(k,h') (Q_{h'}^k - Q_{h'}^{\star})(s_{h'}^{k},a_{h'}^{k}) \\
         &\leq 252\sum_{h_1 = h'}^H\sum_{s,a}\sqrt{(\mathbb{Q}_{h_1}^\star+\beta^2)H\|\omega_h(\cdot,{h_1})\|_{\infty}\|\omega_h(\cdot,{h_1},s,a)\|_{1}\iota}  + 152\sum_{h_1 = h'}^H\sum_{k=1}^K\omega_{h}(k,h_1)Z_{h_1}^{k}.
\end{align*}
\end{lemma}
\begin{proof}[Proof of Lemma \ref{recurQbtotal}]
Repeatedly applying \Cref{sumrecurb} over steps $h',h'+1,\ldots,H$ and using the terminal condition at step $H+1$ completes the proof.
\end{proof}

\begin{lemma}
\label{weightvisit}
For UCB-Bernstein+, under the event $\bigcap_{i=1}^{6}\mathcal{E}_i$ in Lemma \ref{concentrations}, it holds simultaneously for any $h \in [H]$ that:
\begin{align*}
     \sum_{s,a} \Delta_h(s,a)N_h^{K+1}(s,a)&\lesssim  \sum_{h' = h}^H\sum_{\Delta_{h'}(s,a)>0}\frac{(\mathbb{Q}_{h'}^\star+\beta^2)H^2\iota}{\Delta_{h'}(s,a)} +\frac{H\big(\sum_{t = h + 1}^H\sqrt{(\mathbb{Q}_{t}^\star+\beta^2)|Z_{\opt,t}|}\big)^2\iota}{\Delta_{\min,h}} \nonumber\\
    &\quad  +  \sum_{h' = h + 1}^H\frac{\left(\sum_{t = h' + 1}^H\sqrt{(\mathbb{Q}_{t}^\star+\beta^2)|Z_{\textnormal{opt},t}|}\right)^2\iota}{\Delta_{\textnormal{min}, h'}} + \sum_{h' = h}^{H}\sum_{k=1}^K Z_{h'}^{k}.
\end{align*}
\end{lemma}
\begin{proof}[Proof of Lemma \ref{weightvisit}]
For any $h\in[H]$, define
\[
A_h:=\sum_{s,a}\Delta_h(s,a)N_h^{K+1}(s,a).
\]
Let $Z_{\non,h}:=\{(s,a)\mid\Delta_h(s,a)>0\}$ and define
\[
\omega_h^k=\omega_h(k,h):=\mathbb{I}\left[(s_h^k,a_h^k)\in Z_{\non,h}\right]\leq 1.
\]
For any $(h,k)\in[H]\times[K]$, if $(s_h^k,a_h^k)\in Z_{\non,h}$, then
\[
Q_h^k(s_h^k,a_h^k)-Q_h^\star(s_h^k,a_h^k)=V_h^k(s_h^k)-Q_h^\star(s_h^k,a_h^k)\geq V_h^\star(s_h^k)-Q_h^\star(s_h^k,a_h^k)=\Delta_h(s_h^k,a_h^k)>0,
\]
where the inequality follows from Lemma \ref{Q}. Hence, for any $(s,a)\in Z_{\non,h}$,
\[
\|\omega_h(\cdot,h,s,a)\|_1=\sum_{k=1}^K\mathbb{I}\left[(s_h^k,a_h^k)=(s,a)\right]=N_h^{K+1}(s,a),
\]
while $\|\omega_h(\cdot,h,s,a)\|_1=0$ for $(s,a)\in Z_{\opt,h}$.

By \Cref{seprecurb} in Lemma \ref{recurQb}, for any $(s,a)\in Z_{\non,h}$,
\begin{align}
&\sum_{k=1}^K\omega_h(k,h,s,a)\left(Q_h^k-Q_h^\star\right)(s_h^k,a_h^k) \notag\\
&\leq \sum_{k'=1}^K\tilde{\omega}_h(k',h+1,s,a)\left(Q_{h+1}^{k'}-Q_{h+1}^\star\right)(s_{h+1}^{k'},a_{h+1}^{k'}) \notag\\
&\quad+252\sqrt{(\mathbb{Q}_h^\star+\beta^2)HN_h^{K+1}(s,a)\iota}+152\sum_{k=1}^K\omega_h(k,h,s,a)Z_h^k. \label{upperdiff}
\end{align}
On the other hand,
\begin{align}
\sum_{k=1}^K\omega_h(k,h,s,a)\left(Q_h^k-Q_h^\star\right)(s_h^k,a_h^k)&\geq\Delta_h(s,a)\sum_{k=1}^K\mathbb{I}\left[(s_h^k,a_h^k)=(s,a)\right] \notag\\
&=\Delta_h(s,a)N_h^{K+1}(s,a). \label{lowerdiff}
\end{align}
Combining \Cref{upperdiff} and \Cref{lowerdiff}, and using the AM-GM inequality,
\[
252\sqrt{(\mathbb{Q}_h^\star+\beta^2)HN_h^{K+1}(s,a)\iota}\leq\frac{1}{2}\Delta_h(s,a)N_h^{K+1}(s,a)+31752\frac{(\mathbb{Q}_h^\star+\beta^2)H\iota}{\Delta_h(s,a)},
\]
we obtain
\begin{align*}
\Delta_h(s,a)N_h^{K+1}(s,a)
&\leq2\sum_{k'=1}^K\tilde{\omega}_h(k',h+1,s,a)\left(Q_{h+1}^{k'}-Q_{h+1}^\star\right)(s_{h+1}^{k'},a_{h+1}^{k'})\\
&\quad+63504\frac{(\mathbb{Q}_h^\star+\beta^2)H\iota}{\Delta_h(s,a)}+304\sum_{k=1}^K\omega_h(k,h,s,a)Z_h^k.
\end{align*}
Since $Q_{h+1}^k(s,a)\geq Q_{h+1}^\star(s,a)$ by Lemma \ref{Q}, summing the inequality above over $(s,a)\in Z_{\non,h}$ and using property (a) of Lemma \ref{weight} gives
\begin{align}
A_h
&\leq2\sum_{k'=1}^K\omega_h(k',h+1)\left(Q_{h+1}^{k'}-Q_{h+1}^\star\right)(s_{h+1}^{k'},a_{h+1}^{k'}) \notag\\
&\quad+63504\sum_{\Delta_h(s,a)>0}\frac{(\mathbb{Q}_h^\star+\beta^2)H\iota}{\Delta_h(s,a)}+304\sum_{k=1}^K Z_h^k. \label{target}
\end{align}

By Lemma \ref{recurQbtotal},
\begin{align}
&\sum_{k'=1}^K\omega_h(k',h+1)\left(Q_{h+1}^{k'}-Q_{h+1}^\star\right)(s_{h+1}^{k'},a_{h+1}^{k'}) \notag\\
&\leq252\sum_{h'=h+1}^H\sum_{s,a}\sqrt{(\mathbb{Q}_{h'}^\star+\beta^2)H\|\omega_h(\cdot,h')\|_\infty\|\omega_h(\cdot,h',s,a)\|_1\iota}+152\sum_{h'=h+1}^H\sum_{k=1}^K\omega_h(k,h')Z_{h'}^k. \label{totalrecur}
\end{align}
By part (e) of Lemma \ref{weight},
\begin{equation}
\|\omega_h(\cdot,h')\|_\infty\leq\left(1+\frac{1}{H}\right)^{h'-h}\|\omega_h(\cdot,h)\|_\infty\leq e<3, \label{total0}
\end{equation}
and
\begin{equation}
\|\omega_h(\cdot,h')\|_1\leq\|\omega_h(\cdot,h)\|_1=\sum_{\Delta_h(s,a)>0}N_h^{K+1}(s,a). \label{total1}
\end{equation}
Moreover, by part (c) of Lemma \ref{weight},
\begin{equation}
\|\omega_h(\cdot,h',s,a)\|_1\leq\|\omega_h(\cdot,h')\|_\infty N_{h'}^{K+1}(s,a)\leq3N_{h'}^{K+1}(s,a). \label{single1}
\end{equation}

We next divide the state-action pairs at each step $h'$ into $Z_{\non,h'}$ and $Z_{\opt,h'}$. For the suboptimal state-action pairs, by \Cref{total0}, \Cref{single1}, and the Cauchy--Schwarz inequality,
\begin{align}
&504\sum_{h'=h+1}^H\sum_{\Delta_{h'}(s,a)>0}\sqrt{(\mathbb{Q}_{h'}^\star+\beta^2)H\|\omega_h(\cdot,h')\|_\infty\|\omega_h(\cdot,h',s,a)\|_1\iota} \notag\\
&\leq1512\sqrt{\left(\sum_{h'=h+1}^H\sum_{\Delta_{h'}(s,a)>0}\frac{(\mathbb{Q}_{h'}^\star+\beta^2)H\iota}{\Delta_{h'}(s,a)}\right)\left(\sum_{h'=h+1}^H A_{h'}\right)} \notag\\
&\leq\frac{1}{2H}\sum_{h'=h+1}^H A_{h'}+1143072\sum_{h'=h+1}^H\sum_{\Delta_{h'}(s,a)>0}\frac{(\mathbb{Q}_{h'}^\star+\beta^2)H^2\iota}{\Delta_{h'}(s,a)}, \label{nonoptimal}
\end{align}
where the last inequality follows from the AM-GM inequality.

For the optimal state-action pairs, by parts (d) and (e) of Lemma \ref{weight},
\[
\sum_{(s,a)\in Z_{\opt,h'}}\|\omega_h(\cdot,h',s,a)\|_1\leq\|\omega_h(\cdot,h')\|_1\leq\sum_{\Delta_h(s,a)>0}N_h^{K+1}(s,a).
\]
Thus, again by \Cref{total0} and the Cauchy--Schwarz inequality,
\begin{align}
&504\sum_{h'=h+1}^H\sum_{(s,a)\in Z_{\opt,h'}}\sqrt{(\mathbb{Q}_{h'}^\star+\beta^2)H\|\omega_h(\cdot,h')\|_\infty\|\omega_h(\cdot,h',s,a)\|_1\iota} \notag\\
&\leq504\sqrt{3}\left(\sum_{h'=h+1}^H\sqrt{(\mathbb{Q}_{h'}^\star+\beta^2)H|Z_{\opt,h'}|\iota}\right)\sqrt{\sum_{\Delta_h(s,a)>0}N_h^{K+1}(s,a)}. \label{optimal}
\end{align}

If $|Z_{\non,h}|=0$, then $A_h=0$ and the desired conclusion holds trivially. Otherwise, $\Delta_{\min,h}\in(0,+\infty)$ and
\[
A_h=\sum_{\Delta_h(s,a)>0}\Delta_h(s,a)N_h^{K+1}(s,a)\geq\Delta_{\min,h}\sum_{\Delta_h(s,a)>0}N_h^{K+1}(s,a).
\]
Hence, applying the AM-GM inequality to \Cref{optimal}, we obtain
\begin{align}
&504\sqrt{3}\left(\sum_{h'=h+1}^H\sqrt{(\mathbb{Q}_{h'}^\star+\beta^2)H|Z_{\opt,h'}|\iota}\right)\sqrt{\sum_{\Delta_h(s,a)>0}N_h^{K+1}(s,a)} \notag\\
&\leq\frac{1}{2}A_h+381024\frac{H\left(\sum_{h'=h+1}^H\sqrt{(\mathbb{Q}_{h'}^\star+\beta^2)|Z_{\opt,h'}|}\right)^2\iota}{\Delta_{\min,h}}. \label{optimalbound}
\end{align}

Moreover, by \Cref{total0},
\[
304\sum_{h'=h+1}^H\sum_{k=1}^K\omega_h(k,h')Z_{h'}^k\leq912\sum_{h'=h+1}^H\sum_{k=1}^K Z_{h'}^k.
\]
Multiplying \Cref{totalrecur} by $2$, applying \Cref{nonoptimal} and \Cref{optimalbound}, and then substituting the resulting bound into \Cref{target}, we obtain
\begin{align*}
A_h
&\leq\frac{1}{2}A_h+\frac{1}{2H}\sum_{h'=h+1}^H A_{h'}\\
&\quad+1143072\sum_{h'=h+1}^H\sum_{\Delta_{h'}(s,a)>0}\frac{(\mathbb{Q}_{h'}^\star+\beta^2)H^2\iota}{\Delta_{h'}(s,a)}+63504\sum_{\Delta_h(s,a)>0}\frac{(\mathbb{Q}_h^\star+\beta^2)H\iota}{\Delta_h(s,a)}\\
&\quad+381024\frac{H\left(\sum_{h'=h+1}^H\sqrt{(\mathbb{Q}_{h'}^\star+\beta^2)|Z_{\opt,h'}|}\right)^2\iota}{\Delta_{\min,h}}+304\sum_{k=1}^K Z_h^k+912\sum_{h'=h+1}^H\sum_{k=1}^K Z_{h'}^k.
\end{align*}
Moving $A_h/2$ to the left-hand side, multiplying both sides by $2$, and enlarging the numerical constants if necessary, we obtain the explicit recursion
\begin{equation}
A_h\leq F_h+\frac{1}{H}\sum_{h'=h+1}^H A_{h'}, \label{targetin}
\end{equation}
where
\begin{align*}
F_h
&:=2286144\sum_{h'=h}^H\sum_{\Delta_{h'}(s,a)>0}\frac{(\mathbb{Q}_{h'}^\star+\beta^2)H^2\iota}{\Delta_{h'}(s,a)}\\
&\quad+762048\frac{H\left(\sum_{h'=h+1}^H\sqrt{(\mathbb{Q}_{h'}^\star+\beta^2)|Z_{\opt,h'}|}\right)^2\iota}{\Delta_{\min,h}}+1824\sum_{h'=h}^H\sum_{k=1}^K Z_{h'}^k.
\end{align*}
When $|Z_{\non,h}|=0$, $A_h=0$, so \Cref{targetin} also holds trivially.

It remains to solve the recursion \Cref{targetin}. We claim that
\begin{equation}
A_h\leq F_h+\frac{1}{H}\sum_{h'=h+1}^H\left(1+\frac{1}{H}\right)^{h'-h-1}F_{h'}. \label{explicitrecursion}
\end{equation}
This follows by backward induction. The claim is immediate for $h=H$. Assuming that it holds for all $t\in\{h+1,\ldots,H\}$, \Cref{targetin} gives
\begin{align*}
A_h
&\leq F_h+\frac{1}{H}\sum_{t=h+1}^H\left(F_t+\frac{1}{H}\sum_{h'=t+1}^H\left(1+\frac{1}{H}\right)^{h'-t-1}F_{h'}\right)\\
&=F_h+\frac{1}{H}\sum_{h'=h+1}^H\left(1+\frac{1}{H}\right)^{h'-h-1}F_{h'},
\end{align*}
which proves \Cref{explicitrecursion}.

Finally, since
\[
\left(1+\frac{1}{H}\right)^H\leq e,
\]
substituting the definition of $F_{h'}$ into \Cref{explicitrecursion} and exchanging the order of summation yields
\begin{align*}
A_h
&\lesssim\sum_{h'=h}^H\sum_{\Delta_{h'}(s,a)>0}\frac{(\mathbb{Q}_{h'}^\star+\beta^2)H^2\iota}{\Delta_{h'}(s,a)}+\frac{H\left(\sum_{t=h+1}^H\sqrt{(\mathbb{Q}_t^\star+\beta^2)|Z_{\opt,t}|}\right)^2\iota}{\Delta_{\min,h}}\\
&\quad+\sum_{h'=h+1}^H\frac{\left(\sum_{t=h'+1}^H\sqrt{(\mathbb{Q}_t^\star+\beta^2)|Z_{\opt,t}|}\right)^2\iota}{\Delta_{\min,h'}}+\sum_{h'=h}^H\sum_{k=1}^K Z_{h'}^k.
\end{align*}
The hidden constant in the final inequality is universal and independent of $h,H,S,A$, and $K$. This proves the lemma.
\end{proof}

Now we continue to bound the gap-dependent regret. We claim that
\begin{equation}
    \label{upperZ}
    \sum_{h' = 1}^{H}\sum_{k=1}^K Z_{h'}^{k} \lesssim \frac{H^5SA\iota^2}{\beta}.
\end{equation}
\textbf{Proof of \Cref{upperZ}:} 
By definition of $Z_{h'}^{k}$ in \Cref{defZ}, we have the following relationship
\begin{align}
\label{upperzbegin}
     &Z_{h'}^{k}= \mathbb{I}\left[t_{h'}^{k}=1\right]H +\mathbb{I}\left[t_{h'}^{k}>1\right]\frac{H^2\iota/\beta+\sqrt{H\Psi_{h'}^k(s_{{h'}}^{k},a_{{h'}}^{k})\iota}}{n_{h'}^k(s_{{h'}}^{k},a_{{h'}}^{k})}.
\end{align}
For the first term of \Cref{upperzbegin}, by \Cref{upperY}, we have
\begin{equation*}
    \sum_{h' = 1}^{H}\sum_{k=1}^K\mathbb{I}\left[t_{h'}^{k}=1\right]H  \lesssim H^3SA.
\end{equation*}
For the second term of \Cref{upperzbegin}, by Lemma \ref{wN} with $\alpha = 1$, we know
\begin{align}
\label{z21}
    \sum_{h' = 1}^{H} \sum_{k=1}^K\mathbb{I}\left[t_{h'}^{k}>1\right]\frac{H^2\iota}{\beta n_{h'}^k(s_{h'}^{k},a_{h'}^{k})} \lesssim\frac{H^4SA\iota^2}{\beta}.
\end{align}
We also claim that
\begin{align}
    \label{z22}
    \sum_{h'=1}^H\sum_{k=1}^K\mathbb{I}\left[t_{h'}^{k}>1\right]\frac{\sqrt{H\Psi_{h'}^k(s_{h'}^{k}, a_{h'}^{k})\iota}}{n_{h'}^k(s_{h'}^{k}, a_{h'}^{k})}  \lesssim \frac{H^5SA\iota^2}{\beta},
\end{align}
which will be proved later. Then combining the results of \Cref{z21} and \Cref{z22}, we can bound the second term of \Cref{upperzbegin} as
\begin{equation*}
    \sum_{h'=1}^H\sum_{k=1}^K\mathbb{I}\left[t_{h'}^{k}>1\right]\frac{H^2\iota/\beta+\sqrt{H\Psi_{h'}^k(s_{h'}^{k}, a_{h'}^{k})\iota}}{n_{h'}^k(s_{h'}^{k}, a_{h'}^{k})} \lesssim \frac{H^5SA\iota^2}{\beta}.
\end{equation*}
Combining the results above, we complete the proof of \cref{upperZ}.

Let  $\delta = \frac{1}{8T}$, then $\mathcal{E}: = \bigcap_{i=1}^8\mathcal{E}_i$ holds with probability at least $1-\frac{1}{T}$ and $\iota \leq O(\log(SAT)).$ Therefore, we have
\begin{align*}
    &\mathbb{E}(\textnormal{Regret}(T))
    \\
    &= \mathbb{E}\left(\sum_{h= 1}^H\sum_{s,a}\Delta_h(s,a)N_h^{K+1}(s,a)\right)\\
 & = \mathbb{E} \left(\sum_{h= 1}^H\sum_{s,a}\Delta_h(s,a)N_h^{K+1}(s,a) \bigg| \mathcal{E}\right)\mathbb{P}(\mathcal{E}) + \mathbb{E} \left(\sum_{h= 1}^H\sum_{s,a}\Delta_h(s,a)N_h^{K+1}(s,a) \bigg|  \mathcal{E}^c\right)\mathbb{P}(\mathcal{E}^c)\\
    & \leq O\Bigg( \sum_{h = 1}^H\sum_{\Delta_{h}(s,a)>0}\frac{(\mathbb{Q}_{h}^\star+\beta^2)H^3\iota_1}{\Delta_{h}(s,a)} +\sum_{h = 1}^H\frac{H(\sum_{t = h + 1}^H\sqrt{(\mathbb{Q}_{t}^\star+\beta^2)|Z_{\opt,t}|})^2\iota_1}{\Delta_{\min,h}} 
+ \frac{H^6SA\iota_1^2}{\beta}\Bigg) \notag\\
&\quad + \frac{1}{T} \cdot HT\\
&\leq O\Bigg( \sum_{h = 1}^H\sum_{\Delta_{h}(s,a)>0}\frac{(\mathbb{Q}_{h}^\star+\beta^2)H^3\iota_1}{\Delta_{h}(s,a)} +\sum_{h = 1}^H\frac{H(\sum_{t = h + 1}^H\sqrt{(\mathbb{Q}_{t}^\star+\beta^2)|Z_{\opt,t}|})^2\iota_1}{\Delta_{\min,h}} 
+ \frac{H^6SA\iota_1^2}{\beta}\Bigg).
\end{align*}
The first inequality holds due to Lemma \ref{weightvisit} under the event $\mathcal{E}$, and under the event $\mathcal{E}^c$, 
$$\sum_{h= 1}^H\sum_{s,a}\Delta_h(s,a)N_h^{K+1}(s,a) \leq HT.$$
\textbf{Proof of \cref{z22}:}

For each $(s,a,h)$ and stage $t\geq 2$, let $\Phi_h^t(s,a)$ denote the common value of
$\Psi_h^k(s,a)$ over episodes $k$ satisfying $t_h^k(s,a)=t$. This value is well-defined
because $\Psi_h^k(s,a)$ depends only on the samples collected in the preceding stage
$t-1$, and hence remains unchanged throughout stage $t$. 

Then we have
\begin{align}
    &\sum_{h=1}^H\sum_{k=1}^K\mathbb{I}\left[t_h^{k}>1\right]\frac{\sqrt{H\Psi_h^k(s_h^{k}, a_h^{k})}}{n_h^k(s_h^{k}, a_h^{k})} \notag\\
    &=\sum_{h=1}^H\sum_{s,a}\sum_{k=1}^K \frac{\sqrt{H\Psi_h^k(s,a)}}{n_h^k(s,a)}\mathbb{I}\left[ (s_h^{k}, a_h^{k}) = (s,a)\right]\left(\sum_{t=2}^{t_h^K(s,a)}\mathbb{I}\left[t_h^k(s,a) = t\right]\right) \nonumber \\
    &=\sum_{h=1}^H\sum_{s,a}\sum_{t=2}^{t_h^K(s,a)}\sum_{k=1}^K \frac{\sqrt{H\Psi_h^k(s,a)}}{e_{t-1}}\mathbb{I}\left[ (s_h^{k}, a_h^{k}) = (s,a)\right]\mathbb{I}\left[t_h^k(s,a) = t\right] \nonumber \\
    &\leq \sum_{h=1}^H\sum_{s,a}\sum_{t=2}^{t_h^K(s,a)} \sqrt{H\Phi_h^{t}(s,a)}\frac{e_t}{e_{t-1}}\nonumber \\
    &\lesssim \sum_{h=1}^H\sum_{s,a}\sum_{t=2}^{t_h^K(s,a)}\sqrt{H\Phi_h^{t}(s,a)} \nonumber\\
    &\lesssim \sqrt{SAH^3\iota \sum_{h=1}^H\sum_{s,a}\sum_{t=2}^{t_h^K(s,a)} \Phi_h^{t}(s,a)}. \label{sumpsibegin}
\end{align}
Here, the first inequality holds due to the definition of $\Phi_h^t(s,a)$. The second inequality is because $e_t \leq 2e_{t-1}$. The last inequality is by the Cauchy--Schwarz inequality with $\sum_{s,a,h}t_h^K(s,a) \lesssim SAH^2\iota$ by \Cref{sumnumstage}. Furthermore, note that
\begin{align}
    &\sum_{h=1}^H\sum_{s,a}\sum_{t=2}^{t_h^K(s,a)} \Phi_h^{t}(s,a) \notag\\
    &\leq \sum_{h=1}^H\sum_{k=1}^K \left(V_{h+1}^{k}-\hat{V}_{h+1}^{k}\right)(s_{h+1}^{k}) \nonumber\\
    &\leq \sum_{h=1}^H\sum_{k=1}^K \left(V_{h+1}^{k}-V_{h+1}^{\star}\right)(s_{h+1}^{k})\cdot\mathbb{I}\left[ \left(V_{h+1}^{k}-V_{h+1}^{\star}\right)(s_{h+1}^{k}) > \beta\right] \nonumber\\
    &\overset{\mathrm{(i)}}{\lesssim} \sum_{h=1}^H\sum_{k=1}^K \left(Q_{h+1}^{k}-Q_{h+1}^{\star}\right)(s_{h+1}^{k}, a_{h+1}^{k})\cdot\mathbb{I}\left[ \left(Q_{h+1}^{k}-Q_{h+1}^{\star}\right)(s_{h+1}^{k}, a_{h+1}^{k}) > \beta\right] \nonumber\\
    &\lesssim \frac{H^6 S A \iota}{\beta}. \label{sumpsiupper}
\end{align}
Here, $\mathrm{(i)}$ is because $Q_{h+1}^{k}(s_{h+1}^{k}, a_{h+1}^{k}) = V_{h+1}^k(s_{h+1}^{k})$ and $Q_{h+1}^{\star}(s_{h+1}^{k}, a_{h+1}^{k}) \leq V_{h+1}^{\star}(s_{h+1}^{k})$. The last inequality holds due to Lemma \ref{clipq}. Then by applying \Cref{sumpsiupper} to \Cref{sumpsibegin}, we complete the proof of \Cref{z22}.

\section{Proof of \texorpdfstring{\Cref{thm_gapcost}}{Theorem 3.2}}
\label{switchcost}
In this section, we provide the proof for \Cref{thm_gapcost}.
\begin{proof}[Proof of \Cref{thm_gapcost}]
For any target failure probability $\delta\in(0,1)$, run \Cref{alg1} with its failure-probability parameter set to $\delta/8$, so that $\iota=\log(16SAT/\delta)$. Applying Lemma \ref{concentrations} with failure parameter $\delta/8$ and taking a union bound, $\bigcap_{i=1}^{8}\mathcal E_i$ holds with probability at least $1-\delta$. 

Next, we will prove \Cref{thm_gapcost} under the event $\bigcap_{i=1}^{8} \mathcal{E}_i$.

From the proof of Theorem 2 in \citet{zhang2020almost}, we have:
 \begin{align*}
N_\textnormal{switch} \leq \sum_{s,a,h} 4H \log\bigg(\frac{N_h^{K+1}(s,a)}{2H}+1\bigg).
\end{align*}
Next for any $(s,a,h) \in \mathcal{S} \times \mathcal{A} \times [H]$, we will bound the term $N_h^{K+1}(s,a)$. Let $\mathcal{A}_h^\star(s) = \{a\mid a  \text{ maximizes }Q_h^\star(s,\cdot)\}$, which is the set of optimal actions for state-step pair $(s,h)$. For $a \notin \mathcal{A}_h^\star(s)$, we have $\Delta_h(s,a) >0 $ and then $\Delta_h(s,a) \geq \Delta_{\textnormal{min}}$. For any $h \in [H]$, let $Z_h$ be the set of all triples of $(s,a,h)$ such that $a \notin \mathcal{A}_h^\star(s)$, i.e., $Z_h = \{(s,a,h) | a \notin \mathcal{A}_h^\star(s)\}.$

    We also let the set $Z = \bigcup_{h=1}^{H}Z_h$ and the set $Z_{\textnormal{opt}} = \{(s,a,h) | a \in \mathcal{A}_h^\star(s)\}$. Then we have $|Z| + |Z_{\textnormal{opt}}| =SAH$. Since every state-step pair $(s,h)$ has at least one optimal action, we know $|Z_{\textnormal{opt}}| \geq SH$ and then $0 < |Z|\leq SH(A-1)$.
    
    If for given $(h,k) \in [H] \times [K]$, $(s_h^k,a_h^k,h) \in Z_h$, we have $\Delta_h(s_h^k,a_h^k) \geq \Delta_{\textnormal{min}}$. Then it holds that:
    $$Q_h^k(s_h^k,a_h^k) - Q_h^\star(s_h^k,a_h^k) = V_h^k(s_h^k) - Q_h^\star(s_h^k,a_h^k) \geq  \Delta_h(s_h^k,a_h^k) \geq \Delta_{\textnormal{min}}.$$
    The first inequality is because $V_h^k(s) \geq V_h^\star(s)$ by Lemma \ref{Q}. Therefore, we have 
    \begin{align*}
        \sum_{(s,a,h) \in Z_h}\mathbb{I}[(s_h^k,a_h^k) = (s,a)] &=  \mathbb{I}[(s_h^k,a_h^k,h) \in Z_h]\leq  \mathbb{I}\left[Q_h^k(s_h^k,a_h^k) - Q_h^\star(s_h^k,a_h^k) \geq \Delta_{\textnormal{min}}\right]
    \end{align*}
    and thus
    \begin{align} 
     \sum_{(s,a,h) \in Z}N_h^{K+1}(s,a) &= \sum_{h=1}^H\sum_{(s,a,h)\in Z_{h}}\sum_{k=1}^K \mathbb{I}[(s_h^k,a_h^k) = (s,a)] \notag\\
     &\leq \sum_{h=1}^H\sum_{k=1}^K\mathbb{I}\left[Q_h^k(s_h^k,a_h^k) - Q_h^\star(s_h^k,a_h^k) \geq \Delta_{\textnormal{min}}\right] \notag\\
     &\leq O\left( \frac{\sum_{h=1}^H\left(\mathbb{Q}_h^\star+\beta^2 \right) H^3SA\iota }{\Delta_{\textnormal{min}}^2}+\frac{H^6SA\iota^2}{\beta\Delta_{\textnormal{min}}}\right). \label{ndoptc}
    \end{align}
    The last inequality is by Lemma \ref{clipqb}. Therefore we have:
    \begin{align*}
    &N_\textnormal{switch} \leq \sum_{s,a,h} 4H \log\left(\frac{N_h^{K+1}(s,a)}{2H}+1\right) \nonumber\\
    & = \sum_{(s,a,h) \in Z}4H \log\left(\frac{N_h^{K+1}(s,a)}{2H}+1\right) + \sum_{(s,a,h) \in Z_{\textnormal{opt}}}4H \log\left(\frac{N_h^{K+1}(s,a)}{2H}+1\right) \nonumber\\
    &\leq O\Bigg(H(SAH-|Z_{\textnormal{opt}}|) \log \left(\frac{\sum_{h=1}^H\left(\mathbb{Q}_h^\star+\beta^2 \right) H^2SA\iota}{(SAH-|Z_{\textnormal{opt}}|)\Delta_{\textnormal{min}}^2} + \frac{H^5SA\iota^2}{(SAH-|Z_{\textnormal{opt}}|)\beta\Delta_{\textnormal{min}}}\right)  \notag\\
    &\qquad +H|Z_{\textnormal{opt}}|\log \left(\frac{K}{|Z_{\textnormal{opt}}|}+1\right) \Bigg) \\
    &\leq O\left(H|Z_{\textnormal{opt}}^c| \log \left(\frac{H^6SA\iota^2}{|Z_{\textnormal{opt}}^c|\beta\Delta_{\textnormal{min}}^2}\right)  +H|Z_{\textnormal{opt}}|\log \left(\frac{K}{|Z_{\textnormal{opt}}|}+1\right) \right)\\
    & = O\Bigg(H|Z_{\textnormal{opt}}|\log \left(\frac{K}{|Z_{\textnormal{opt}}|}+1\right) + H|Z_{\textnormal{opt}}^c| \log \Bigg(\frac{H^{\frac{5}{2}}S^{\frac{1}{2}}A^{\frac{1}{2}}\log(\frac{SAT}{\delta})
   }{|Z_{\textnormal{opt}}^c|^{\frac{1}{2}}\Delta_{\textnormal{min}}}\Bigg)\Bigg).
    \end{align*}
    The second inequality is because of Jensen's inequality and \Cref{ndoptc}. The last inequality holds due to 
    $$\frac{\sum_{h=1}^H\left(\mathbb{Q}_h^\star+\beta^2 \right)H^2SA\iota}{\left(SAH-|Z_{\textnormal{opt}}|\right)\Delta_{\textnormal{min}}^2} \leq \frac{H^6SA\iota^2}{\left(SAH-|Z_{\textnormal{opt}}|\right)\beta\Delta_{\textnormal{min}}^2},$$
    and
    $$\frac{H^5SA\iota^2}{(SAH-|Z_{\textnormal{opt}}|)\beta\Delta_{\textnormal{min}}} \leq \frac{H^6SA\iota^2}{(SAH-|Z_{\textnormal{opt}}|)\beta\Delta_{\textnormal{min}}^2}$$
    by $\mathbb{Q}_h^\star \leq H^2$ for all $h$ and $0 <\beta,\Delta_{\textnormal{min}} \leq H$. The last equation holds by setting $\beta = H$. Then we finish the proof.
    
    In particular, if each state-step pair has a unique optimal action, which means $|Z_{\textnormal{opt}}| = SH$, then the policy switching cost is upper bounded by:
    $$ O\left(H^2SA \log \left(\frac{H^2\log(\frac{SAT}{\delta})}{\Delta_{\textnormal{min}}}\right) + H^2S\log \left(\frac{K}{HS}+1\right) \right).$$
\end{proof}
Now we prove Lemma \ref{clipqb}.
\begin{lemma}
\label{clipqb}
Under the event $\bigcap_{i=1}^{6}\mathcal{E}_i$ in Lemma \ref{concentrations}, for all $\epsilon \in (0, H]$, we have the following conclusion:
\begin{align*} 
    \sum_{h=1}^H \sum_{k=1}^K \mathbb{I} \left [Q_h^k(s_h^{k}, a_h^{k}) - Q_h^{\star}(s_h^{k}, a_h^{k}) \geq \epsilon \right]
    &\lesssim \frac{\sum_{h = 1}^H\left(\mathbb{Q}_{h}^\star+\beta^2\right)H^3 S A \iota}{\epsilon^2}  + \frac{H^6SA\iota^2}{\beta\epsilon}.
\end{align*}
\end{lemma}
\begin{proof}[Proof of Lemma \ref{clipqb}]
By Lemma \ref{recurQbtotal}, for any given $h \in [H]$ and non-negative weight sequence $\{\omega_{h}^{k}\}_{k}$ at step $h$, it holds simultaneously for any $(s,a) \in \mathcal{S} \times \mathcal{A} $ and subsequent step $h' \in [h,H]$ that:
\begin{align}
     &\sum_{k = 1}^K \omega_{h}(k,h') (Q_{h'}^k - Q_{h'}^{\star})(s_{h'}^{k},a_{h'}^{k}) \notag\\
     &\lesssim \sum_{h_1 = h'}^H\sum_{s,a}\sqrt{(\mathbb{Q}_{h_1}^\star+\beta^2)H\|\omega_h(\cdot,{h_1})\|_{\infty}\|\omega_h(\cdot,{h_1},s,a)\|_{1}\iota} + \sum_{h_1 = h'}^H\sum_{k=1}^K\omega_{h}(k,h_1)Z_{h_1}^{k} \notag\\
     &\lesssim \sqrt{\sum_{h_1 = h'}^H(\mathbb{Q}_{h_1}^\star+\beta^2)H^2SA\|\omega_h(\cdot,{h_1})\|_{\infty}\|\omega_h(\cdot,{h_1})\|_{1}\iota} +  \sum_{h_1 = h'}^H\sum_{k=1}^K\omega_{h}(k,h_1)Z_{h_1}^{k}\notag\\
     &\lesssim \sqrt{\sum_{h_1 = h'}^H\left(\mathbb{Q}_{h_1}^\star+\beta^2\right)H^2SA\|\omega_h(\cdot,{h})\|_{\infty}\|\omega_h(\cdot,{h})\|_{1}\iota} +  \sum_{h_1 = h'}^H\sum_{k=1}^K\omega_{h}(k,h_1)Z_{h_1}^{k}. \label{recurQbsum}
\end{align}
Here, the second inequality follows from the Cauchy--Schwarz inequality. The last inequality holds due to 
\begin{equation*}
    \|\omega_h(\cdot,h_1)\|_{\infty} \leq \left(1+\frac{1}{H}\right)^{h_1-h}\|\omega_h(\cdot,h)\|_{\infty} \lesssim \|\omega_h(\cdot,h)\|_{\infty},\ \|\omega_h(\cdot,h_1)\|_{1} \leq \|\omega_h(\cdot,h)\|_{1}
\end{equation*}
for any $h_1 \in [h,H]$ by part (e) of Lemma \ref{weight}.

Then, following the same argument as in the proof of Lemma \ref{clipq}, let $N=\max\{1,\lceil\log_2(H/\epsilon)\rceil\}$. For any fixed $h \in [H]$, $k \in [K]$, and $i \in [N-1]$, define
$$\omega_h^i(k,h) = \mathbb{I}\left[(Q_h^k - Q_h^{\star})(s_h^{k}, a_h^{k}) \in \big[2^{i-1}\epsilon,2^i\epsilon\big)\right] \leq 1$$
and
$$\omega_h^N(k,h) = \mathbb{I}\left[(Q_h^k - Q_h^{\star})(s_h^{k}, a_h^{k}) \in \big[2^{N-1}\epsilon,H\big]\right] \leq 1.$$
Moreover, for any $k' \in [K]$ and $h \leq h' \leq H$, we define $\omega_h^i(k',h')$ analogously as in \Cref{irecurweight}, and denote the corresponding norms $\|\omega_h^i(\cdot,h')\|_{1}$ and $\|\omega_h^i(\cdot,h')\|_{\infty}$ in the same manner as in \Cref{iweightnorm}. Then we have:
\begin{equation}
        \label{difflowerb}
        \sum_{k=1}^K \omega_{h}^i(k,h) \left(Q_h^k - Q_h^{\star}\right)(s_h^{k}, a_h^{k}) \geq 2^{i-1}\epsilon\|\omega_h^i(\cdot,h)\|_{1}.
    \end{equation}
By \Cref{recurQbsum}, since $\omega_{h}^i(k,h) \leq 1$, we have
\begin{align*}
     &\sum_{k = 1}^K \omega_{h}^i(k,h) \left(Q_{h}^k - Q_{h}^{\star}\right)(s_{h}^{k},a_{h}^{k})\\ 
     &\lesssim  \sum_{h_1 = h}^H\sum_{k=1}^K\omega_{h}^i(k,h_1)Z_{h_1}^{k} + \sqrt{\sum_{h_1 = h}^H\left(\mathbb{Q}_{h_1}^\star+\beta^2\right)H^2SA\|\omega_h^i(\cdot,{h})\|_{1}\iota}.
\end{align*}
Combining the result above and \Cref{difflowerb}, we have the following inequality:
    \begin{align}
    \label{inib}
        2^{i-1}\epsilon\|\omega_{h}^i(\cdot,{h})\|_{1} \lesssim \sqrt{\sum_{h_1 = h}^H\left(\mathbb{Q}_{h_1}^\star+\beta^2\right)H^2SA\|\omega_{h}^i(\cdot,{h})\|_{1}\iota} + \sum_{h_1 = h}^H\sum_{k=1}^K\omega_{h}^i(k,h_1)Z_{h_1}^{k}.
    \end{align}
Solving the inequality in \Cref{inib}, we can derive the following relationship:
    \begin{equation}
    \label{answerb}
       \|\omega_{h}^i(\cdot,{h})\|_{1} \lesssim \frac{\sum_{h_1 = h}^H\left(\mathbb{Q}_{h_1}^\star+\beta^2\right)H^2SA\iota}{4^i\epsilon^2} + \frac{\sum_{h_1 = h}^H\sum_{k=1}^K\omega_{h}^i(k,h_1)Z_{h_1}^{k}}{2^i\epsilon}.
    \end{equation}
We now prove the conclusion. By \Cref{clipkmj1} and \Cref{answerb}, we have
    \begin{align*}
        &\sum_{h=1}^H\sum_{k=1}^K \mathbb{I} \left [Q_h^k(s_h^{k}, a_h^{k}) - Q_h^{\star}(s_h^{k}, a_h^{k}) \geq \epsilon \right] \notag\\
        &= \sum_{h=1}^H\sum_{i=1}^N\|\omega_{h}^i(\cdot,{h})\|_{1} \notag\\
       &\lesssim \sum_{h=1}^H\sum_{i=1}^N\frac{\sum_{h_1 = h}^H\left(\mathbb{Q}_{h_1}^\star+\beta^2\right)H^2SA\iota}{4^i\epsilon^2} + \sum_{h=1}^H\sum_{i=1}^N\frac{\sum_{h_1 = h}^H\sum_{k=1}^K\omega_{h}^i(k,h_1)Z_{h_1}^{k}}{2^i\epsilon} \nonumber\\
        &\lesssim \frac{\sum_{h = 1}^H\left(\mathbb{Q}_{h}^\star+\beta^2\right)H^3 S A \iota}{\epsilon^2}  + \sum_{h=1}^H\sum_{i=1}^N\frac{\sum_{h_1 = h}^H\sum_{k=1}^KZ_{h_1}^{k}}{2^i\epsilon} \nonumber\\
        &\lesssim  \frac{\sum_{h = 1}^H\left(\mathbb{Q}_{h}^\star+\beta^2\right)H^3 S A \iota}{\epsilon^2}  + \frac{H^6SA\iota^2}{\beta\epsilon}.
    \end{align*}
    Here, the second inequality is because $0 \leq \omega_{h}^i(k,h_1)< 3$ by \Cref{omegaprime}. The last inequality holds due to \Cref{upperZ}. Then we finish the proof.
\end{proof}

\section{Proof of \texorpdfstring{\Cref{thm_regret}}{Theorem 3.3}}
\label{worstregretproof}
In this section, we bound the worst-case regret under the event $\bigcap_{i=1}^{8}\mathcal{E}_i$ in Lemma \ref{concentrations}.

As $V_h^k(s)\geq V_h^{\star}(s)$ for all $(s,h,k)\in \mathcal{S}\times [H]\times [K]$ by Lemma \ref{Q}, we have:
\begin{equation*}
    \mbox{Regret}(T) = \sum_{k=1}^K\left(V_1^\star(s_1^{k}) - V_1^{\pi^{k}}(s_1^{k})\right) \leq \sum_{k=1}^K\left(V_1^k(s_1^{k}) - V_1^{\pi^{k}} (s_1^{k})\right) = \sum_{k=1}^K \zeta_1^k.
\end{equation*}
Thus, we only need to bound $\sum_{k=1}^K \zeta_1^k$. By \Cref{singlerecurinitial} and the definition of $\tilde{b}_h^{k}$ in \Cref{betabernnew}, we have
\begin{align}
\label{recur1beginb}
        &\sum_{k=1}^K(Q_h^{k} - Q_h^{\star})(s_h^{k}, a_h^{k}) \nonumber\\
        &\leq \sum_{k=1}^K\mathbb{I}\left[t_h^{k}=1\right]H
        +8\sum_{k=1}^K\mathbb{I}\left[t_h^{k}>1\right]\sqrt{\frac{\tilde{W}_h^{k}(s_h^{k}, a_h^{k})\iota}{n_h^k}}
        +119\sum_{k=1}^K\mathbb{I}\left[t_h^{k}>1\right]\frac{H\iota}{n_h^k} \nonumber\\
        &+ \sum_{k=1}^K\mathbb{I}\left[t_h^{k}>1\right]\frac{\sum_{i=1}^{n_h^k}(V_{h+1}^{k_{l_i}}-V_{h+1}^\star)(s_{h+1}^{k_{l_i}})}{n_h^k(s_h^{k},a_h^{k})}
        + \sum_{k=1}^K\mathbb{I}\left[t_h^{k}>1\right]\frac{\sum_{i=1}^{n_h^k}\left(\mathbbm{1}_{s_{h+1}^{k_{l_i}}} - \mathbb{P}_{s_h^{k}, a_h^{k},h}\right)V_{h+1}^\star}{n_h^k(s_h^{k},a_h^{k})}.
\end{align}
For the lower-order bonus term in \Cref{recur1beginb}, by Lemma \ref{wN}, it holds that
\begin{equation}
    \label{recur12-2}
    \sum_{k=1}^K\mathbb{I}\left[t_h^{k}>1\right]\frac{H\iota}{n_h^k} \leq 10H^2SA\iota^2.
\end{equation}
Moreover, by \Cref{recur11} and \Cref{recur13},
\begin{align*}
    \sum_{k=1}^K\mathbb{I}\left[t_h^{k}=1\right]H &\leq H^2SA,\\
    \sum_{k=1}^K\mathbb{I}\left[t_h^{k}>1\right]\frac{\sum_{i=1}^{n_h^k}\left(V_{h+1}^{k_{l_i}}-V_{h+1}^{\star}\right)(s_{h+1}^{k_{l_i}})}{n_h^k(s_h^{k},a_h^{k})}
    &\leq \left(1+\frac{1}{H}\right)\sum_{k=1}^K \delta_{h+1}^k.
\end{align*}
Since $\iota=\log(2SAT/\delta)\geq \log 2$, we have $1\leq 3\iota^2$. Therefore, combining the above inequalities with \Cref{recur1beginb} yields
\begin{align*}
        \sum_{k=1}^K(Q_h^{k} - Q_h^{\star})(s_h^{k}, a_h^{k})
        &\leq \left(1+\frac{1}{H}\right)\sum_{k=1}^K \delta_{h+1}^k
        +1193H^2SA\iota^2
        +8\sum_{k=1}^K\mathbb{I}\left[t_h^{k}>1\right]\sqrt{\frac{\tilde{W}_h^{k}(s_h^{k}, a_h^{k})\iota}{n_h^k}}\\     &\quad+\sum_{k=1}^K\mathbb{I}\left[t_h^{k}>1\right]\frac{\sum_{i=1}^{n_h^k}\left(\mathbbm{1}_{s_{h+1}^{k_{l_i}}} - \mathbb{P}_{s_h^{k}, a_h^{k},h}\right)V_{h+1}^\star}{n_h^k(s_h^{k},a_h^{k})}.
\end{align*}

Furthermore,
\begin{align*}
    &\sum_{k=1}^K\mathbb{P}_{s_h^{k}, a_h^{k},h}(V_{h+1}^\star - V_{h+1}^{\pi^k})=\sum_{k=1}^K\left(\mathbb{P}_{s_h^{k}, a_h^{k},h}-\mathbbm{1}_{s_{h+1}^{k}}\right)\left(V_{h+1}^\star - V_{h+1}^{\pi^k}\right)
    +\sum_{k=1}^K\left(\zeta_{h+1}^k-\delta_{h+1}^k\right).
\end{align*}
Hence, together with \Cref{recurbegin}, we obtain
\begin{align*}
    &\sum_{k=1}^K \zeta_h^k
    \leq \sum_{k=1}^K\zeta_{h+1}^k+\frac{1}{H}\sum_{k=1}^K\delta_{h+1}^k
    +1193H^2SA\iota^2+8\sum_{k=1}^K\mathbb{I}\left[t_h^{k}>1\right]\sqrt{\frac{\tilde{W}_h^{k}(s_h^{k}, a_h^{k})\iota}{n_h^k}}\\    &+\sum_{k=1}^K\mathbb{I}\left[t_h^{k}>1\right]\frac{\sum_{i=1}^{n_h^k}\left(\mathbbm{1}_{s_{h+1}^{k_{l_i}}} - \mathbb{P}_{s_h^{k}, a_h^{k},h}\right)V_{h+1}^\star}{n_h^k(s_h^{k},a_h^{k})}+\sum_{k=1}^K\left(\mathbb{P}_{s_h^{k}, a_h^{k},h}-\mathbbm{1}_{s_{h+1}^{k}}\right)\left(V_{h+1}^\star - V_{h+1}^{\pi^k}\right)\\
    &\leq \left(1+\frac{1}{H}\right)\sum_{k=1}^K\zeta_{h+1}^k
    +1193H^2SA\iota^2+8\sum_{k=1}^K\mathbb{I}\left[t_h^{k}>1\right]\sqrt{\frac{\tilde{W}_h^{k}(s_h^{k}, a_h^{k})\iota}{n_h^k}}\\    &\quad+\sum_{k=1}^K\mathbb{I}\left[t_h^{k}>1\right]\frac{\sum_{i=1}^{n_h^k}\left(\mathbbm{1}_{s_{h+1}^{k_{l_i}}} - \mathbb{P}_{s_h^{k}, a_h^{k},h}\right)V_{h+1}^\star}{n_h^k(s_h^{k},a_h^{k})}+\sum_{k=1}^K\left(\mathbb{P}_{s_h^{k}, a_h^{k},h}-\mathbbm{1}_{s_{h+1}^{k}}\right)\left(V_{h+1}^\star - V_{h+1}^{\pi^k}\right),
\end{align*}
where the last inequality uses $\delta_{h+1}^k\leq\zeta_{h+1}^k$.

By recursively applying the inequality above over $h=1,2,\ldots,H$, since $\zeta_{H+1}^k = 0$ and $\left(1+1/H\right)^H\leq e$, we obtain
$$\mbox{Regret}(T) \leq \sum_{k=1}^K \zeta_1^k \leq 1193eH^3SA\iota^2 +8R_1+R_2+R_3,$$
where 
\begin{equation*}
\label{R1}
    R_1 = \sum_{h=1}^H \left(1+\frac{1}{H}\right)^{h-1} \sum_{k=1}^K\mathbb{I}\left[t_h^{k}>1\right]\sqrt{\frac{\tilde{W}_h^{k}(s_h^{k}, a_h^{k})\iota}{n_h^k}},
\end{equation*}
\begin{equation*}
\label{R2}
    R_2 = \sum_{h=1}^H \left(1+\frac{1}{H}\right)^{h-1}\sum_{k=1}^K\mathbb{I}\left[t_h^{k}>1\right]\frac{\sum_{i=1}^{n_h^k}\left(\mathbbm{1}_{s_{h+1}^{k_{l_i}}} - \mathbb{P}_{s_h^{k}, a_h^{k},h}\right)V_{h+1}^\star}{n_h^k(s_h^{k},a_h^{k})},
\end{equation*}
and
\begin{equation*}
\label{R4}
    R_3 = \sum_{h=1}^H \left(1+\frac{1}{H}\right)^{h-1}\sum_{k=1}^K\left(\mathbb{P}_{s_h^{k}, a_h^{k},h}-\mathbbm{1}_{s_{h+1}^{k}}\right)\left(V_{h+1}^\star - V_{h+1}^{\pi^k}\right).
\end{equation*}
Under the event $\mathcal{E}_4$ in Lemma \ref{concentrations}, we can bound $R_3$ by $3\sqrt{2H^2T\iota}$. Then, we reach
$$\mbox{Regret}(T) \lesssim \sqrt{H^2T\iota} +H^3SA\iota^2 + R_1+R_2 \lesssim (1+\beta) \sqrt{H^3SAT\iota} + \left(1+\frac{1}{\beta}\right)H^{\frac{9}{2}}SA\iota^2.$$
The last inequality holds due to Lemma \ref{upperboundr1} and \ref{upperboundr2}. 

For any target failure probability $\delta\in(0,1)$, run Algorithm 1 with its failure-probability parameter set to $\delta/8$, so that $\iota=\log(16SAT/\delta)$. Applying Lemma \ref{concentrations} with failure parameter $\delta/8$ and taking a union bound, the event $\bigcap_{i=1}^8 \mathcal{E}_i$ holds with probability at least $1-\delta$. On this event, setting $\beta=1$, we have
$$
\mbox{Regret}(T) \lesssim \sqrt{H^3SAT\iota} + H^{\frac{9}{2}}SA\iota^2.
$$

\subsection{\texorpdfstring{Upper Bounds of $R_1$ and $R_2$}{Upper Bounds of R1 and R2}}
\begin{lemma}
    \label{upperboundr1}
    Under the event $\bigcap_{i=1}^{8}\mathcal{E}_i$ in Lemma \ref{concentrations}, we have
    $$R_1 \lesssim (1+\beta) \sqrt{H^3SAT\iota} + \left(1+\frac{1}{\beta}\right)H^{\frac{9}{2}}SA\iota^{\frac{3}{2}}.$$
\end{lemma}
\begin{proof}[Proof of Lemma \ref{upperboundr1}]
Since $\tilde{W}_h^{k}(s_h^{k}, a_h^{k}) \geq 0$, we have
\begin{align*}
    R_1 
    &\lesssim \sum_{h=1}^H\sum_{k=1}^K\mathbb{I}\left[t_h^{k}>1\right]\sqrt{\frac{\tilde{W}_h^{k}(s_h^{k}, a_h^{k})\iota}{n_h^k(s_h^{k}, a_h^{k})}}.
\end{align*}
By \Cref{singlerefupper}, we have
\begin{align}
 &\sqrt{\frac{\tilde{W}_h^{k}(s_h^{k}, a_h^{k})}{n_h^k(s_h^{k}, a_h^{k})}} \lesssim \frac{\sqrt{H\Psi_h^k(s_h^{k}, a_h^{k})}}{n_h^k(s_h^{k}, a_h^{k})} + \sqrt{\frac{\mathbb{V}_{s_h^{k}, a_h^{k}, h} (V_{h+1}^\star)}{n_h^k(s_h^{k}, a_h^{k})}}+ \sqrt{\frac{\beta^2}{n_h^k(s_h^{k}, a_h^{k})}}+ \frac{H\iota^{\frac{1}{4}}} {(n_h^k(s_h^{k}, a_h^{k}))^{\frac{3}{4}}}. \label{singlerefupperworst}
\end{align}
Note that by Lemma \ref{wN} with $\alpha = \frac{1}{2}$ and $\frac{3}{4}$, we have
$$\sum_{h=1}^H\sum_{k=1}^K\mathbb{I}\left[t_h^{k}>1\right]\sqrt{\frac{\beta^2}{n_h^k(s_h^{k}, a_h^{k})}} \lesssim \beta \sqrt{H^2SAT},$$
and
$$\sum_{h=1}^H\sum_{k=1}^K\mathbb{I}\left[t_h^{k}>1\right]\frac{H\iota^{\frac{1}{4}}}{(n_h^k(s_h^{k}, a_h^{k}))^{\frac{3}{4}}} \lesssim  H^{\frac{5}{2}}(SA)^{\frac{3}{4}}(T\iota)^{\frac{1}{4}}.$$
Applying these two inequalities to \Cref{singlerefupperworst}, it holds that
\begin{align}
 \sum_{h=1}^H\sum_{k=1}^K&\mathbb{I}\left[t_h^{k}>1\right]\sqrt{\frac{\tilde{W}_h^{k}(s_h^{k}, a_h^{k})}{n_h^k(s_h^{k}, a_h^{k})}}\lesssim \sum_{h=1}^H\sum_{k=1}^K\mathbb{I}[t_h^{k}>1]\frac{\sqrt{H\Psi_h^k(s_h^{k}, a_h^{k})}}{n_h^k(s_h^{k}, a_h^{k})} \nonumber\\
 & \quad + \sum_{h=1}^H\sum_{k=1}^K\mathbb{I}\left[t_h^{k}>1\right]\sqrt{\frac{\mathbb{V}_{s_h^{k}, a_h^{k}, h} (V_{h+1}^\star)}{n_h^k(s_h^{k}, a_h^{k})}}+  H^{\frac{5}{2}}(SA)^{\frac{3}{4}}(T\iota)^{\frac{1}{4}} + \beta \sqrt{H^2SAT}. \label{refupper}
\end{align}
By \Cref{sumpsibegin} and \Cref{sumpsiupper}, we have
\begin{equation}
    \label{sumpsiworst}
    \sum_{h=1}^H\sum_{k=1}^K\mathbb{I}\left[t_h^{k}>1\right]\frac{\sqrt{H\Psi_h^k(s_h^{k}, a_h^{k})}}{n_h^k(s_h^{k}, a_h^{k})} \lesssim \sqrt{\frac{1}{\beta}}H^{\frac{9}{2}}SA\iota.
\end{equation}
We further claim the following conclusion
\begin{equation}
    \label{vstar}
    \sum_{h=1}^H\sum_{k=1}^K\mathbb{I}\left[t_h^{k}>1\right]\sqrt{\frac{\mathbb{V}_{s_h^{k}, a_h^{k}, h} (V_{h+1}^\star)}{n_h^k(s_h^{k}, a_h^{k})}} \lesssim \sqrt{H^3SAT}+H^{\frac{9}{2}}SA\sqrt{\iota},
\end{equation}
which will be proved later. Then combining the results of \Cref{refupper}, \Cref{vstar} and \Cref{sumpsiworst}, we have
\begin{align*}
    R_1 &\lesssim (1+\beta) \sqrt{H^3SAT\iota} + H^{\frac{5}{2}}(SA\iota)^{\frac{3}{4}}T^{\frac{1}{4}} +\left(1+\sqrt{\frac{1}{\beta}}\right)H^{\frac{9}{2}}SA\iota^{\frac{3}{2}} \\
    &\lesssim  (1+\beta) \sqrt{H^3SAT\iota} + \left(1+\frac{1}{\beta}\right)H^{\frac{9}{2}}SA\iota^{\frac{3}{2}}.
\end{align*}
The last inequality follows from the AM--GM inequality:
\begin{align*}
H^{\frac{5}{2}}(SA\iota)^{\frac{3}{4}}T^{\frac{1}{4}}
&=\sqrt{\sqrt{H^3SAT\iota}\,H^{\frac{7}{2}}SA\iota}\\
&\leq \frac{1}{2}\sqrt{H^3SAT\iota}+\frac{1}{2}H^{\frac{7}{2}}SA\iota\\
&\lesssim \sqrt{H^3SAT\iota}+H^{\frac{9}{2}}SA\iota^{\frac{3}{2}}.
\end{align*}

Next, we will prove \Cref{vstar}.

\textbf{Proof of \cref{vstar}:}

Because of Lemma \ref{wN} with $\omega_h^{k}=1$, $\alpha = \frac{1}{2}$, we know
\begin{align}
    \label{nsqrtN}
    \sum_{k=1}^K\mathbb{I}\left[t_h^{k}>1\right]\frac{\mathbb{I}\left[(s_h^{k}, a_h^{k})=(s,a)\right]}{\sqrt{n_h^k(s,a)}}\lesssim \sqrt{HN_h^{K+1}(s,a)}.
\end{align}
Then we can derive the following relationship
    \begin{align}
    \label{vstarbegin}
        &\sum_{h=1}^H\sum_{k=1}^K\mathbb{I}\left[t_h^{k}>1\right]\sqrt{\frac{\mathbb{V}_{s_h^{k}, a_h^{k}, h} (V_{h+1}^\star)}{n_h^k(s_h^{k}, a_h^{k})}}\notag\\
        &= \sum_{h=1}^H\sum_{s,a}\sum_{k=1}^K\mathbb{I}\left[t_h^{k}>1\right]\sqrt{\frac{\mathbb{V}_{s, a, h} (V_{h+1}^\star)}{n_h^k(s,a)}} \mathbb{I}\left[(s_h^{k},a_h^{k}) = (s,a)\right] \nonumber\\
        & \lesssim \sum_{h=1}^H\sum_{s,a}\sqrt{HN_h^{K+1}(s,a)\mathbb{V}_{s, a, h}(V_{h+1}^\star)}\nonumber\\
        &\leq \sqrt{H^2SA}\sqrt{\sum_{h=1}^H\sum_{k=1}^K\mathbb{V}_{s_h^{k}, a_h^{k}, h} (V_{h+1}^\star)}.
    \end{align}
    The first inequality holds due to \Cref{nsqrtN}. The last inequality is by the Cauchy--Schwarz inequality. Next, we will bound the right-hand side of \Cref{vstarbegin}. We first note that
    \begin{align}
    \label{vstarmiddle1}
        \sum_{h=1}^H\sum_{k=1}^K\mathbb{V}_{s_h^{k}, a_h^{k}, h} (V_{h+1}^\star) 
        &\leq \sum_{h=1}^H\sum_{k=1}^K\mathbb{V}_{s_h^{k}, a_h^{k}, h} (V_{h+1}^{\pi^k}) + \sum_{h=1}^H\sum_{k=1}^K\left|\mathbb{V}_{s_h^{k}, a_h^{k}, h} (V_{h+1}^{\star})-\mathbb{V}_{s_h^{k}, a_h^{k}, h} (V_{h+1}^{\pi^k})\right|\nonumber\\
        &\lesssim HT + H^3\iota + \sum_{h=1}^H\sum_{k=1}^K\left|\mathbb{V}_{s_h^{k}, a_h^{k}, h} (V_{h+1}^{\star})-\mathbb{V}_{s_h^{k}, a_h^{k}, h} (V_{h+1}^{\pi^k})\right|.
    \end{align}
    The last inequality follows directly from $\mathcal{E}_7$ in Lemma \ref{concentrations}. The last term on the right-hand side of \cref{vstarmiddle1} can be bounded as follows
    \begin{align}
        &\sum_{h=1}^H\sum_{k=1}^K\left|\mathbb{V}_{s_h^{k}, a_h^{k}, h} (V_{h+1}^{\star})-\mathbb{V}_{s_h^{k}, a_h^{k}, h} (V_{h+1}^{\pi^k})\right| \nonumber \\
    &= \sum_{h=1}^{H} \sum_{k=1}^K \left| \mathbb{P}_{s_h^{k}, a_h^{k},h} (V_{h+1}^{\star})^2 - \mathbb{P}_{s_h^{k}, a_h^{k},h} (V_{h+1}^{\pi^k})^2 - (\mathbb{P}_{s_h^{k}, a_h^{k},h} V_{h+1}^{\star})^2 + (\mathbb{P}_{s_h^{k}, a_h^{k},h} V_{h+1}^{\pi^k})^2 \right| \notag \\
    &\leq \sum_{h=1}^{H} \sum_{k=1}^K \left\{ \mathbb{P}_{s_h^{k}, a_h^{k},h} \left( (V_{h+1}^{\star} - V_{h+1}^{\pi^k}) (V_{h+1}^{\star} + V_{h+1}^{\pi^k}) \right) + \left| (\mathbb{P}_{s_h^{k}, a_h^{k},h} V_{h+1}^{\star})^2 - (\mathbb{P}_{s_h^{k}, a_h^{k},h}V_{h+1}^{\pi^k})^2 \right| \right\} \nonumber \\
    &\leq 4H \sum_{h=1}^{H}\sum_{k=1}^K \mathbb{P}_{s_h^{k}, a_h^{k},h}\left(V_{h+1}^{\star} - V_{h+1}^{\pi^k}\right) \nonumber \\
    &= 4H \sum_{h=1}^{H}\sum_{k=1}^K  \left\{ V_{h+1}^{\star} (s_{h+1}^{k}) - V_{h+1}^{\pi^k} (s_{h+1}^{k}) + \left(\mathbb{P}_{s_h^{k}, a_h^{k},h} - \mathbbm{1}_{s_{h+1}^{k}}\right) (V_{h+1}^{\star} - V_{h+1}^{\pi^k}) \right\} \nonumber \\
    &\overset{\mathrm{(i)}}{\lesssim} H \sum_{h=1}^{H}\sum_{k=1}^K \left(V_{h+1}^{
    k} - V_{h+1}^{\pi^k}\right)(s_{h+1}^{k}) + H^2\sqrt{T\iota} \nonumber \\
    &\overset{\mathrm{(ii)}}{\lesssim} H^4\sqrt{SAT\iota} + H^5SA. \label{vstarmiddle2}
    \end{align}
    Here, $\mathrm{(i)}$ holds due to the event $\mathcal{E}_8$ in Lemma \ref{concentrations}.  (ii) follows from the bound on $\sum_{k=1}^K \zeta_{h+1}^k$ established by the proof of Lemma \ref{regreth}. Therefore, applying \Cref{vstarmiddle2} to \Cref{vstarmiddle1}, we reach
    \begin{align}
        \label{sumvarvstar}
        \sum_{h=1}^H\sum_{k=1}^K\mathbb{V}_{s_h^{k}, a_h^{k}, h} (V_{h+1}^\star) 
        \lesssim HT + H^7SA\iota.
    \end{align}
    Here we use $H^4\sqrt{SAT\iota} \leq HT + H^7SA\iota$ by the AM-GM inequality. Applying \Cref{sumvarvstar} to \Cref{vstarbegin}, we finish the proof of \Cref{vstar}:
    \begin{align*}
        \sum_{h=1}^H\sum_{k=1}^K\mathbb{I}\left[t_h^{k}>1\right]\sqrt{\frac{\mathbb{V}_{s_h^{k}, a_h^{k}, h} (V_{h+1}^\star)}{n_h^k(s_h^{k}, a_h^{k})}} 
        &\lesssim\sqrt{H^2SA}\sqrt{HT + H^7SA\iota} \leq \sqrt{H^3SAT}+H^{\frac{9}{2}}SA\sqrt{\iota}.
    \end{align*}
\end{proof}
\begin{lemma}
    \label{upperboundr2}
    Under the event $\bigcap_{i=1}^{8}\mathcal{E}_i$ in Lemma \ref{concentrations}, we have
    $$R_2 \lesssim \sqrt{H^3SAT\iota}+H^{\frac{9}{2}}SA\iota^2.$$
\end{lemma}
\begin{proof}[Proof of Lemma \ref{upperboundr2}]
By the event $\mathcal{E}_5$ in Lemma \ref{concentrations}, we have
\begin{align*}
    R_2 &= \sum_{h=1}^H \left(1+\frac{1}{H}\right)^{h-1}\sum_{k=1}^K\mathbb{I}\left[t_h^{k}>1\right]\frac{\sum_{i=1}^{n_h^k}\left(\mathbbm{1}_{s_{h+1}^{k_{l_i}}} - \mathbb{P}_{s_h^{k}, a_h^{k},h}\right)V_{h+1}^\star}{n_h^k(s_h^{k},a_h^{k})}\\
    &\lesssim \sum_{h=1}^H \sum_{k=1}^K\mathbb{I}\left[t_h^{k}>1\right]\left(\sqrt{\frac{\iota}{n_h^k}\mathbb{V}_{s_h^k,a_h^k,h}(V_{h+1}^{\star})} + \frac{H\iota}{n_h^k}\right) \\
    &\lesssim \sqrt{H^3SAT\iota}+H^{\frac{9}{2}}SA\iota^2.
\end{align*}
The last inequality is by \Cref{vstar} and Lemma \ref{wN} with $\alpha = 1$.
\end{proof}

\section{Proof of \texorpdfstring{\Cref{thm_cost}}{Theorem 3.4}}
\label{worstswitchproof}
Because $\pi_h^k(s)=\arg\max_a Q_h^k(s,a)$, we have
\[
\mathbb{I}\left[\pi_h^k(s)\neq\pi_h^{k+1}(s)\right]\leq\mathbb{I}\left[\exists a,\ Q_h^{k+1}(s,a)\neq Q_h^k(s,a)\right].
\]
Therefore,
\begin{align*}
\sum_{k=1}^{K-1}\mathbb{I}\left[\pi_h^k(s)\neq\pi_h^{k+1}(s)\right]
&\leq\sum_{k=1}^{K-1}\mathbb{I}\left[\exists a,\ Q_h^{k+1}(s,a)\neq Q_h^k(s,a)\right]\\
&\leq\sum_{a\in\mathcal A}\left(t_h^K(s,a)-1\right).
\end{align*}
Hence, the total local switching cost of UCB-Bernstein+ is bounded by
\begin{align*}
N_{\textnormal{switch}}
&=\sum_{s\in\mathcal S}\sum_{h=1}^H\sum_{k=1}^{K-1}\mathbb{I}\left[\pi_h^k(s)\neq\pi_h^{k+1}(s)\right]\\
&\leq\sum_{h=1}^H\sum_{s,a}\left(t_h^K(s,a)-1\right)\\
&\leq 4H^2SA\log\left(\frac{K}{HSA}+1\right),
\end{align*}
where the last inequality follows from \Cref{sumnumstage}.

\end{document}

%% file: math_commands.tex
\usepackage{amsmath,amsfonts,bm}

\def\eqref#1{equation~\ref{#1}}
\def\1{\bm{1}}

\DeclareMathAlphabet{\mathsfit}{\encodingdefault}{\sfdefault}{m}{sl}
\SetMathAlphabet{\mathsfit}{bold}{\encodingdefault}{\sfdefault}{bx}{n}

%% file: main.bib
@article{jin2018q,
  title={Is Q-learning provably efficient?},
  author={Jin, Chi and Allen-Zhu, Zeyuan and Bubeck, Sebastien and Jordan, Michael I},
  journal={Advances in Neural Information Processing Systems},
  volume={31},
  year={2018}
}

@inproceedings{qiao2022sample,
  title={Sample-efficient reinforcement learning with loglog (t) switching cost},
  author={Qiao, Dan and Yin, Ming and Min, Ming and Wang, Yu-Xiang},
  booktitle={International Conference on Machine Learning},
  pages={18031--18061},
  year={2022},
  organization={PMLR}
}

@article{bai2019provably,
  title={Provably efficient q-learning with low switching cost},
  author={Bai, Yu and Xie, Tengyang and Jiang, Nan and Wang, Yu-Xiang},
  journal={Advances in Neural Information Processing Systems},
  volume={32},
  year={2019}
}

@article{zhang2020almost,
  title={Almost optimal model-free reinforcement learning via reference-advantage decomposition},
  author={Zhang, Zihan and Zhou, Yuan and Ji, Xiangyang},
  journal={Advances in Neural Information Processing Systems},
  volume={33},
  pages={15198--15207},
  year={2020}
}

@article{auer2008near,
  title={Near-optimal regret bounds for reinforcement learning},
  author={Auer, Peter and Jaksch, Thomas and Ortner, Ronald},
  journal={Advances in Neural Information Processing Systems},
  volume={21},
  year={2008}
}

@article{agrawal2017optimistic,
  title={Optimistic posterior sampling for reinforcement learning: worst-case regret bounds},
  author={Agrawal, Shipra and Jia, Randy},
  journal={Advances in Neural Information Processing Systems},
  volume={30},
  year={2017}
}

@inproceedings{zhang2021reinforcement,
  title={Is reinforcement learning more difficult than bandits? a near-optimal algorithm escaping the curse of horizon},
  author={Zhang, Zihan and Ji, Xiangyang and Du, Simon},
  booktitle={Conference on Learning Theory},
  pages={4528--4531},
  year={2021},
  organization={PMLR}
}

@inproceedings{agarwal2020model,
  title={Model-based reinforcement learning with a generative model is minimax optimal},
  author={Agarwal, Alekh and Kakade, Sham and Yang, Lin F},
  booktitle={Conference on Learning Theory},
  pages={67--83},
  year={2020},
  organization={PMLR}
}

@inproceedings{menard2021ucb,
  title={UCB Momentum Q-learning: Correcting the bias without forgetting},
  author={M{\'e}nard, Pierre and Domingues, Omar Darwiche and Shang, Xuedong and Valko, Michal},
  booktitle={International Conference on Machine Learning},
  pages={7609--7618},
  year={2021},
  organization={PMLR}
}

@inproceedings{zanette2019tighter,
  title={Tighter problem-dependent regret bounds in reinforcement learning without domain knowledge using value function bounds},
  author={Zanette, Andrea and Brunskill, Emma},
  booktitle={International Conference on Machine Learning},
  pages={7304--7312},
  year={2019},
  organization={PMLR}
}

@book{1998Reinforcement,
  title={Reinforcement Learning: An Introduction},
  author={ Sutton, Richard S.  and  Barto, Andrew G. },
  year={2018},
  publisher={MIT Press}
}

@article{gao2021provably,
  title={A provably efficient algorithm for linear markov decision process with low switching cost},
  author={Gao, Minbo and Xie, Tianle and Du, Simon S and Yang, Lin F},
  journal={arXiv preprint arXiv:2101.00494},
  year={2021}
}

@inproceedings{zheng2023federated,
title={Federated Q-Learning: Linear Regret Speedup with Low Communication Cost},
author={Zhong Zheng and Fengyu Gao and Lingzhou Xue and Jing Yang},
booktitle={The Twelfth International Conference on Learning Representations},
year={2024}
}

@inproceedings{zheng2024federated,
  title={Federated Q-Learning with Reference-Advantage Decomposition: Almost Optimal Regret and Logarithmic Communication Cost},
  author={Zheng, Zhong and Zhang, Haochen and Xue, Lingzhou},
  booktitle={The Thirteenth International Conference on Learning Representations},
  year={2025}
}

@inproceedings{zhang2024settling,
  title={Settling the sample complexity of online reinforcement learning},
  author={Zhang, Zihan and Chen, Yuxin and Lee, Jason D and Du, Simon S},
  booktitle={Conference on Learning Theory},
  pages={5213--5219},
  year={2024},
  organization={PMLR}
}

@inproceedings{simchowitz2019non,
  title={Non-asymptotic gap-dependent regret bounds for tabular MDPs},
  author={Simchowitz, Max and Jamieson, Kevin G},
  booktitle={Advances in Neural Information Processing Systems},
  year={2019}
}

@inproceedings{xu2021fine,
  title={Fine-grained gap-dependent bounds for tabular mdps via adaptive multi-step bootstrap},
  author={Xu, Haike and Ma, Tengyu and Du, Simon},
  booktitle={Conference on Learning Theory},
  pages={4438--4472},
  year={2021},
  organization={PMLR}
}

@inproceedings{yang2021q,
  title={Q-learning with logarithmic regret},
  author={Yang, Kunhe and Yang, Lin and Du, Simon},
  booktitle={International Conference on Artificial Intelligence and Statistics},
  pages={1576--1584},
  year={2021},
  organization={PMLR}
}

@inproceedings{azar2017minimax,
  title={Minimax regret bounds for reinforcement learning},
  author={Azar, Mohammad Gheshlaghi and Osband, Ian and Munos, R{\'e}mi},
  booktitle={International Conference on Machine Learning},
  pages={263--272},
  year={2017},
  organization={PMLR}
}

@inproceedings{mirhoseini2017device,
  title={Device placement optimization with reinforcement learning},
  author={Mirhoseini, Azalia and Pham, Hieu and Le, Quoc V and Steiner, Benoit and Larsen, Rasmus and Zhou, Yuefeng and Kumar, Naveen and Norouzi, Mohammad and Bengio, Samy and Dean, Jeff},
  booktitle={International Conference on Machine Learning},
  pages={2430--2439},
  year={2017},
  organization={PMLR}
}

@inproceedings{nguyen2019incomplete,
  title={Incomplete conditional density estimation for fast materials discovery},
  author={Nguyen, Phuoc and Tran, Truyen and Gupta, Sunil and Rana, Santu and Barnett, Matthew and Venkatesh, Svetha},
  booktitle={Proceedings of the 2019 SIAM International Conference on Data Mining},
  pages={549--557},
  year={2019},
  organization={SIAM}
}

@inproceedings{he2021logarithmic,
  title={Logarithmic regret for reinforcement learning with linear function approximation},
  author={He, Jiafan and Zhou, Dongruo and Gu, Quanquan},
  booktitle={International Conference on Machine Learning},
  pages={4171--4180},
  year={2021},
  organization={PMLR}
}

@incollection{auer2007logarithmic,
  title={Logarithmic online regret bounds for undiscounted reinforcement learning},
  author={Auer, Peter and Ortner, Ronald},
  booktitle={Advances in Neural Information Processing Systems},
  pages={49--56},
  year={2007},
  publisher={MIT Press}
}

@inproceedings{tewari2008optimistic,
  title={Optimistic linear programming gives logarithmic regret for irreducible MDPs},
  author={Tewari, Ambuj and Bartlett, Peter},
  booktitle={Advances in Neural Information Processing Systems},
  pages={1505--1512},
  year={2008}
}

@article{jaksch2010near,
  title={Near-optimal Regret Bounds for Reinforcement Learning},
  author={Jaksch, Thomas and Ortner, Ronald and Auer, Peter},
  journal={Journal of Machine Learning Research},
  volume={11},
  pages={1563--1600},
  year={2010}
}

@article{ok2018exploration,
  title={Exploration in structured reinforcement learning},
  author={Ok, Jungseul and Proutiere, Alexandre and Tranos, Damianos},
  journal={Advances in Neural Information Processing Systems},
  volume={31},
  year={2018}
}

@inproceedings{nguyen2023instance,
  title={On instance-dependent bounds for offline reinforcement learning with linear function approximation},
  author={Nguyen-Tang, Thanh and Yin, Ming and Gupta, Sunil and Venkatesh, Svetha and Arora, Raman},
  booktitle={Proceedings of the AAAI Conference on Artificial Intelligence},
  pages={9310--9318},
  year={2023}
}

@article{li2024q,
  title={Is Q-learning minimax optimal? a tight sample complexity analysis},
  author={Li, Gen and Cai, Changxiao and Chen, Yuxin and Wei, Yuting and Chi, Yuejie},
  journal={Operations Research},
  volume={72},
  number={1},
  pages={222--236},
  year={2024},
  publisher={INFORMS}
}

@inproceedings{zheng2024gap,
  title={Gap-Dependent Bounds for Q-Learning using Reference-Advantage Decomposition},
  author={Zheng, Zhong and Zhang, Haochen and Xue, Lingzhou},
    booktitle={The Thirteenth International Conference on Learning Representations},
  year={2025}
}

@inproceedings{jin2020provably,
  title={Provably efficient reinforcement learning with linear function approximation},
  author={Jin, Chi and Yang, Zhuoran and Wang, Zhaoran and Jordan, Michael I},
  booktitle={Conference on Learning Theory},
  pages={2137--2143},
  year={2020},
  organization={PMLR}
}

@InProceedings{labbi2024federated,
  title = 	 {Federated UCBVI: Communication-Efficient Federated Regret Minimization with Heterogeneous Agents},
  author =       {Labbi, Safwan and Tiapkin, Daniil and Mancini, Lorenzo and Mangold, Paul and Moulines, Eric},
  booktitle = 	 {Proceedings of The 28th International Conference on Artificial Intelligence and Statistics},
  pages = 	 {1315--1323},
  year = 	 {2025},
  editor = 	 {Li, Yingzhen and Mandt, Stephan and Agrawal, Shipra and Khan, Emtiyaz},
  volume = 	 {258},
  series = 	 {Proceedings of Machine Learning Research},
  month = 	 {03--05 May},
  publisher =    {PMLR},
  url = 	 {https://proceedings.mlr.press/v258/labbi25a.html}
}

@inproceedings{
zhang2025gapdependent,
title={Gap-Dependent Bounds for Federated $Q$-Learning},
author={Haochen Zhang and Zhong Zheng and Lingzhou Xue},
booktitle={Forty-second International Conference on Machine Learning},
year={2025}
}

@article{li2023breaking,
  title={Breaking the sample complexity barrier to regret-optimal model-free reinforcement learning},
  author={Li, Gen and Shi, Laixi and Chen, Yuxin and Chi, Yuejie},
  journal={Information and Inference: A Journal of the IMA},
  volume={12},
  number={2},
  pages={969--1043},
  year={2023},
  publisher={Oxford University Press}
}

@inproceedings{zhang2025regret,
 author = {Zhang, Haochen and Zheng, Zhong and Xue, Lingzhou},
 booktitle = {Advances in Neural Information Processing Systems},
 doi = {10.52202/085713-4288},
 editor = {D. Belgrave and C. Zhang and H. Lin and R. Pascanu and P. Koniusz and M. Ghassemi and N. Chen},
 pages = {128711--128799},
 publisher = {Curran Associates, Inc.},
 title = {Regret-Optimal Q-Learning with Low Cost for Single-Agent and Federated Reinforcement Learning},
 url = {https://proceedings.neurips.cc/paper_files/paper/2025/file/bb203e938836544655996d1bb94a0fd7-Paper-Conference.pdf},
 volume = {38, Main Conference},
 year = {2025}
}

@inproceedings{chen2025sharp,
 author = {Chen, Shulun and Zhou, Runlong and Zhang, Zihan and Fazel, Maryam and Du, Simon},
 booktitle = {Advances in Neural Information Processing Systems},
 doi = {10.52202/085713-3747},
 editor = {D. Belgrave and C. Zhang and H. Lin and R. Pascanu and P. Koniusz and M. Ghassemi and N. Chen},
 pages = {112169--112205},
 publisher = {Curran Associates, Inc.},
 title = {Sharp Gap-Dependent Variance-Aware Regret Bounds for Tabular MDPs},
 url = {https://proceedings.neurips.cc/paper_files/paper/2025/file/a2fb8d64706276cc6c49a52f143e0e5c-Paper-Conference.pdf},
 volume = {38, Main Conference},
 year = {2025}
}

@inproceedings{zhang2025q,
 author = {Zhang, Haochen and Zheng, Zhong and Xue, Lingzhou},
 booktitle = {International Conference on Learning Representations},
 editor = {C. Vondrick and B. Hariharan and C. Raffel and L. Pinto and D. Yang and A. Faust},
 pages = {48118--48162},
 title = {Q-Learning with Fine-Grained Gap-Dependent Regret},
 url = {https://proceedings.iclr.cc/paper_files/paper/2026/file/4f8d1e02a057beb7920618ac73788078-Paper-Conference.pdf},
 volume = {2026},
 year = {2026}
}

@article{zhang2024achieving,
  title={Achieving constant regret in linear Markov decision processes},
  author={Zhang, Weitong and Fan, Zhiyuan and He, Jiafan and Gu, Quanquan},
  journal={Advances in Neural Information Processing Systems},
  volume={37},
  pages={130694--130738},
  year={2024}
}

@inproceedings{ayoub2020model,
  title={Model-based reinforcement learning with value-targeted regression},
  author={Ayoub, Alex and Jia, Zeyu and Szepesvari, Csaba and Wang, Mengdi and Yang, Lin},
  booktitle={International Conference on Machine Learning},
  pages={463--474},
  year={2020},
  organization={PMLR}
}

@article{papini2021reinforcement,
  title={Reinforcement learning in linear mdps: Constant regret and representation selection},
  author={Papini, Matteo and Tirinzoni, Andrea and Pacchiano, Aldo and Restelli, Marcello and Lazaric, Alessandro and Pirotta, Matteo},
  journal={Advances in Neural Information Processing Systems},
  volume={34},
  pages={16371--16383},
  year={2021}
}

@misc{Dominguesrlberry2021,
author = {Domingues, Omar Darwiche and Flet-Berliac, Yannis and Leurent, Edouard and Ménard, Pierre and Shang, Xuedong and Valko, Michal},
doi = {10.5281/zenodo.5223307},
month = oct,
title = {{rlberry - A Reinforcement Learning Library for Research and Education}},
url = {https://github.com/rlberry-py/rlberry},
version = {0.2.2-dev},
year = {2021}
}

@article{liang2026provably,
  title={Provably Efficient Federated Reinforcement Learning with Linear Function Approximation and Logarithmic Communication Cost},
  author={Liang, Zihang and Zhang, Haochen and Xue, Lingzhou},
  journal={arXiv preprint arXiv:2609.00193},
  year={2026}
}

@article{zhang2026gap,
  title={Gap-Dependent Bounds for Nearly Minimax Optimal Reinforcement Learning with Linear Function Approximation},
  author={Zhang, Haochen and Zheng, Zhong and Xue, Lingzhou},
  journal={arXiv preprint arXiv:2602.20297},
  year={2026}
}
